\documentclass[11pt, a4paper, onecolumn, copyright, goog]{google}

\usepackage[authoryear, sort&compress, round]{natbib}
\usepackage{microtype}
\usepackage{hyperref}
\usepackage{url}
\usepackage{booktabs}
\usepackage{amsmath, amssymb, amsthm}
\usepackage{wrapfig}
\usepackage{graphicx} 
\usepackage{algorithm}
\usepackage{algpseudocode}
\usepackage{svg}
\usepackage{pgfplots}
\usepackage{amsthm}

\graphicspath{{./}{assets/}{arxiv/assets/}}

\newtheorem{definition}{Definition}
\newtheorem{proposition}{Proposition}

\usepackage{tikz}
\usepackage{subcaption}
\usetikzlibrary{calc, shapes.geometric, arrows.meta, positioning, decorations.pathreplacing, backgrounds, fit}
\usetikzlibrary{tikzmark}

\usepackage{multirow}
\newcommand{\midsize}{\fontsize{8.25pt}{8.75pt}\selectfont}

\newcommand{\hlbase}[1]{\tikz[baseline=(x.base)]\node[fill=blue!15, rounded corners=2pt, inner sep=2pt] (x) {#1};}
\newcommand{\hlprob}[1]{\tikz[baseline=(x.base)]\node[fill=green!20, rounded corners=2pt, inner sep=2pt] (x) {#1};}

\pgfdeclarelayer{background}
\pgfsetlayers{background,main}

\definecolor{etablue}{RGB}{31, 119, 180}
\definecolor{nsaorange}{RGB}{255, 127, 14}
\definecolor{darkgray}{RGB}{60, 60, 60}

\makeatletter
\newenvironment{breakablealgorithm}
  {
   \begin{center}
     \refstepcounter{algorithm}%
     \hrule height.8pt depth0pt \kern2pt%
     \renewcommand{\caption}[2][\relax]{%
       {\raggedright\textbf{\fname@algorithm~\thealgorithm} ##2\par}%
       \ifx\relax##1\relax
         \addcontentsline{loa}{algorithm}{\protect\numberline{\thealgorithm}##2}%
       \else
         \addcontentsline{loa}{algorithm}{\protect\numberline{\thealgorithm}##1}%
       \fi
       \kern2pt\hrule\kern2pt
     }%
  }
  {
     \kern2pt\hrule
   \end{center}
  }
\makeatother

\definecolor{darkblue}{rgb}{0, 0, 0.5}
\hypersetup{colorlinks=true, citecolor=darkblue, linkcolor=darkblue, urlcolor=darkblue}

\usepackage{xcolor}

\keywords{sparse attention, inference optimization, adaptive sparsity, threshold attention}

\uselogo{} 

\title{Elastic Threshold Attention: Learned Contextual Sparsity for Long-Context Decoding}

\correspondingauthor{themisharis@google.com. This work was concluded while the author was interning at Google.}

\author[1,2]{Themistoklis Haris}
\author[1]{Henry Li}
\author[1]{Maryam Karimzadehgan}

\affil[1]{\thepa{}{}}
\affil[2]{Boston University}

\begin{abstract}
Massive KV caches can cause severe memory-bandwidth bottlenecks during long-context decoding. Sparse attention methods mitigate this problem, but often drop necessary context, leading to quality degradation. We introduce \textbf{Elastic Threshold Attention (ETA)}, an end-to-end trainable architecture that rivals dense model quality under hardware-aligned block-sparse decoding. ETA predicts dynamic, contextual thresholds directly from query representations, adjusting context retention depending on the task at hand. To learn this policy from scratch while enabling near lossless KV cache pruning at inference time, we filter attention logits through a shifted SiLU gate during training. We show theoretically and empirically that this creates a smooth, near-uniform attention floor that neutralizes sub-threshold value contributions while simultaneously causing localized attention sinks on initial tokens to disappear. To materialize these advantages, we implement a fused inference-time kernel in Triton that screens KV blocks in $O(1)$ time using cached geometric-probabilistic bounds. Across language modeling, reasoning, and RULER benchmarks, our 1.45B ETA model matches or exceeds dense quality, outperforming alternative fast decoding methods (Quest, H$_2$O, NSA) while achieving higher sparsity levels. Our kernel also achieves up to $2.15\times$ end-to-end speedups over FlashAttention-2 at context lengths of up to $512$K tokens. Finally,
we introduce an offline calibration algorithm for domain-specific deployments
that freezes per-head constant thresholds, cutting
attention compute by an additional 27\% at no quality cost.
\end{abstract}

\renewcommand{\topfraction}{0.95}
\renewcommand{\bottomfraction}{0.95}
\renewcommand{\textfraction}{0.05}
\renewcommand{\floatpagefraction}{0.85}
\begin{document}

\maketitle

\begin{figure}[h]
  \centering
  \begin{subfigure}[b]{0.485\textwidth}
    \centering
    \resizebox{\linewidth}{!}{%
    \begin{tikzpicture}[
        font=\sffamily, >=Stealth, thick,
        box/.style={rectangle, draw=black!70, rounded corners=3pt, align=center, minimum height=1.05cm, minimum width=2.4cm, inner sep=4pt, fill=white, font=\small},
        lossbox/.style={box, draw=red!60, fill=red!5, dashed},
        input/.style={circle, draw=black!70, fill=gray!10, minimum size=0.9cm, inner sep=2pt, font=\bfseries\small},
        label/.style={font=\scriptsize\itshape, text=black!70}
    ]
    \node[input] (x) {$x_t^\ell$};
    \node[box, above right=0.35cm and 0.65cm of x] (qk) {Q, K Projections \\ + RoPE};
    \node[box, below right=0.35cm and 0.65cm of x] (thresh) {Linear Threshold \\ Predictor};
    \node[box, right=0.65cm of qk] (scores) {Attention \\ Logits ($S$)};
    \node[box, right=0.65cm of thresh] (mask) {SiLU Mask \\ $\sigma\big(\beta(S - \tau_t)\big)$};
    \node[box, right=1.1cm of scores] (final) {Final Attention \\ $S \odot \sigma$};
    \node[lossbox, right=1.1cm of mask] (loss) {Sparsity \\ Regularization};

    \draw[->] (x) |- (qk);
    \draw[->] (qk.south) -- node[right=1pt, font=\scriptsize, text=black!60] {$\bar{q}_t$} (thresh.north);
    \draw[->] (qk) -- (scores);
    \draw[->] (thresh) -- (mask);
    \draw[->] (scores) -- (mask) node[midway, right, label] {Scores};
    
    \draw[->] (scores) -- (final);
    
    \draw[->, rounded corners=3pt] (mask.north east) -- ++(0.35, 0.45) |- (final.west);

    \draw[->, red!70, dashed] (mask) -- (loss) 
        node[midway, above, font=\scriptsize, text=red!70, align=center] {Gradient\\Flow};

    \begin{pgfonlayer}{background}
        \node[fit=(thresh)(mask)(loss), fill=blue!5, rounded corners, draw=blue!40!black, dotted, inner sep=7pt] (bg) {};
    \end{pgfonlayer}
    \end{tikzpicture}%
    }
    \label{fig:training-flow-vis}
  \end{subfigure}%
  \hspace{0.02\textwidth}%
  \begin{subfigure}[b]{0.485\textwidth}
    \centering
    \resizebox{\linewidth}{!}{%
    \begin{tikzpicture}[
        >=Stealth, font=\sffamily,
        embed/.style={circle, draw=blue!65!black, minimum size=0.95cm, fill=blue!12, font=\bfseries\small, thick},
        proc_box/.style={draw=blue!65!black, rectangle, rounded corners=3pt, minimum width=2.3cm, minimum height=1.05cm, align=center, fill=blue!12, font=\small\bfseries, thick},
        thresh_box/.style={draw=purple!70!black, rectangle, rounded corners=3pt, minimum width=2.3cm, minimum height=1.05cm, align=center, fill=purple!15, font=\small\bfseries, thick},
        filter_node/.style={circle, draw=orange!75!black, fill=orange!15, minimum size=0.95cm, thick, font=\small\bfseries},
        key_node/.style={draw, rectangle, sharp corners, minimum width=1.2cm, minimum height=0.52cm, align=center, font=\small, anchor=center, outer sep=0pt},
        selected_key/.style={fill=teal!18, draw=teal!75!black, thick, font=\small\bfseries, text=black},
        unselected_key/.style={fill=gray!10, draw=gray!35, thin, text=gray!70},
        arrow_plain/.style={->, draw=gray!70!black, thick, rounded corners=2pt},
        arrow_active/.style={->, draw=teal!80!black, line width=1.3pt, shorten >=2pt},
        arrow_pruned/.style={->, draw=gray!45, dotted, thick, shorten >=2pt}
    ]
    \node[embed] (xt) at (0, 0) {$x_t^\ell$};
    \node[proc_box] (query) at (2.4, 1.05) {Query ($\bar{q}_t$)};
    \node[thresh_box] (predictor) at (2.4, -1.05) {Predictor ($\tau_t$)};
    \node[filter_node] (join) at (4.8, 0) {$>\!\tau_t$};

    \matrix [matrix anchor=west, nodes={key_node}, row sep=-\pgflinewidth] (key_matrix) at (6.6, 0) {
        \node[unselected_key] (k1) {$B_1$}; \\
        \node[selected_key]   (k2) {$B_2$}; \\
        \node[unselected_key] (k3) {$B_3$}; \\
        \node[unselected_key] (k4) {$B_4$}; \\
        \node[selected_key]   (k5) {$B_5$}; \\
        \node[unselected_key] (k6) {$B_6$}; \\
    };
    \node[above=0.08cm of key_matrix, font=\scriptsize\bfseries, color=gray!75!black] {KV Blocks};

    \draw[arrow_plain] (xt.north) |- (query.west);
    \draw[arrow_plain] (xt.south) |- (predictor.west);
    \draw[arrow_plain] (query.east) -| (join.north);
    \draw[arrow_plain] (predictor.east) -| (join.south);

    \draw[arrow_pruned] (join.east) -- (k1.west);
    \draw[arrow_active] (join.east) -- (k2.west);
    \draw[arrow_pruned] (join.east) -- (k3.west);
    \draw[arrow_pruned] (join.east) -- (k4.west);
    \draw[arrow_active] (join.east) -- (k5.west);
    \draw[arrow_pruned] (join.east) -- (k6.west);
    \end{tikzpicture}%
    }
    \label{fig:inference-flow-vis}
  \end{subfigure}
  \vspace{1.5mm}
  \caption{\textbf{Overview of Elastic Threshold Attention (ETA).} \emph{Left (Training):} Queries predict a dynamic threshold $\tau_t$ that gates attention logits via a differentiable SiLU filter, compressing unselected logits toward zero rather than $-\infty$. \emph{Right (Inference):} A custom Triton kernel evaluates geometric-probabilistic block bounds against $\tau_t$ in $\mathcal{O}(1)$ time, loading only active KV blocks from HBM.}
  \label{fig:eta-overview}
\end{figure}
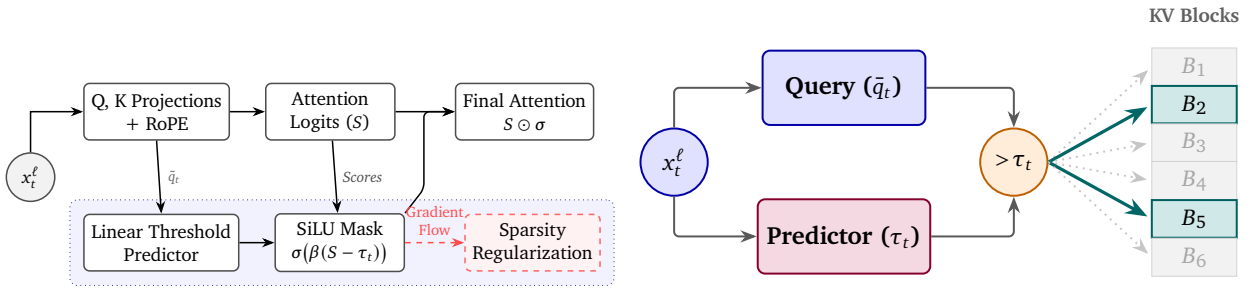

\section{Introduction}
The attention mechanism is the core computational and memory bottleneck of long-context Transformers \citep{Vaswani+2017}. While prefilling is compute-bound, autoregressive decoding is strictly memory-bound because generating each token requires streaming the entire Key-Value (KV) cache from High-Bandwidth Memory (HBM) to on-chip SRAM. As demand for efficient, long-context inference is greater than ever, numerous approaches have been proposed to ameliorate this issue. Most methods implement a version of selective loading, in which unimportant tokens in the KV cache are skipped and not loaded to SRAM, thus improving latency. A common problem that appears, however, is that such selectivity often trims necessary information or distorts the embedding distribution that the model was trained on, thus leading to quality degradation. 

In this work, we ask whether inference-time selectivity can be made naturally compatible with training-time quality in a simple, unified architecture. We introduce \textbf{Elastic Threshold Attention (ETA)}, an end-to-end trainable design that achieves hardware-accelerated decoding while matching dense model quality (\autoref{fig:eta-overview}). Rather than enforcing a fixed token budget, ETA predicts a dynamic, query-conditioned score threshold $\tau_t$ via a lightweight linear projection. When a query encounters a routine token, $\tau_t$ rises to prune the cache; when a query requires retrieval, $\tau_t$ drops automatically to expand the active context window.

To learn this contextual policy from scratch without losing dense representational fidelity, ETA redefines how sub-threshold tokens are treated. Whereas prior sparse mechanisms delete unselected tokens by pushing logits to $-\infty$, we find that hard deletion is counterproductive for end-to-end training. We instead apply a context-based tempering to the distribution by passing the attention logits through a centered SiLU filter. As a result, ETA can adaptively move between two regimes: dense attention, and a sparse operating mode in which a smooth, uniform attention floor replaces the tail of sub-threshold keys. 

This uniform background floor offers two important advantages: First, as we formalize in \autoref{sec:method} and \autoref{appx:log_odds_invariance}, flattening sub-threshold weights causes background value averages to cancel at a $1/\sqrt{m}$ rate up to a positive scalar factor absorbed by $\operatorname{RMSNorm}$, rendering the representation invariant to both dropping pruned KV blocks and co-admitting sub-threshold keys inside loaded blocks. In contrast, an identical model trained with hard masking suffers a $16.6\%$ perplexity penalty while remaining $1.84\times$ denser. Second, ETA exhibits no attention sinks as they are absorbed into the uniform probability floor. This makes ETA mechanistically more interpretable and renders it compatible with modern KV cache compression algorithms.


\paragraph{Our Contributions}
\begin{itemize}[labelindent=0.15in, leftmargin=*]
    \item \textbf{Temper, not Prune \& Elimination of Attention Sinks:} We show theoretically and empirically (\autoref{appx:train_inference_consistency}) that SiLU filtering in attention logit space during training creates a uniform attention floor whose sub-threshold value contributions cancel under $\operatorname{RMSNorm}$, making the model resilient to inference-time block pruning and over-inclusion while eliminating localized attention sinks on initial anchor tokens.
    \item \textbf{Contextual Sparsity \& Domain-Specific Calibration:} We show that runtime query-conditioned thresholds preserve dense model quality, while also supporting optional offline calibration into static per-head constants for domain-specific serving.
    \item \textbf{Hardware-Co-Designed Decoding via Dual Probabilistic-Geometric Screening:} We develop a fused Triton decode kernel that evaluates geometric and probabilistic block bounds in $\mathcal{O}(1)$ time, eliminating KV cache memory traffic for pruned blocks.
    \item \textbf{Quality Retention \& Decoding Acceleration:} Pre-trained at $1.45$B parameters, ETA rivals dense attention and surpasses alternative efficient decoding algorithms (like Quest \citep{tang2024quest}, H$_2$O \citep{zhang2023h2o}, NSA \citep{yuan2025native}) across language modeling, reasoning, and multi-needle/2-hop retrieval benchmarks, showing up to $2.15\times$ decode speedups over FlashAttention-2.
\end{itemize}

\subsection{Related Work}
The literature of efficient attention methods is vast with techniques spanning both theoretical and practical considerations. 
\paragraph{Heuristic Eviction and Windowing.}
Early subquadratic attention methods relied on structural heuristics like sliding windows \citep{child2019generating}, random routing \citep{zaheer2020big}, or low-rank kernel approximations \citep{choromanski2020rethinking, shen2021efficient}. For long-context generation, modern KV cache management methods such as $\text{H}_2\text{O}$ \citep{zhang2023h2o} and StreamingLLM \citep{xiao2024efficient} evict past tokens dynamically by tracking cumulative attention scores or preserving initial sink tokens. While effective at reducing cache footprints, these heuristics decouple pruning from model training, frequently evicting non-recent information and degrading long-context reasoning.

\paragraph{Dynamic Routing and Block-Sparse Selectors.}
To bridge theoretical FLOP reductions and practical GPU acceleration \citep{jiang2024minference, tang2024quest}, dynamic selection methods identify salient context on the fly via $k$-nearest neighbor search \citep{roy2021efficient, haris2025knn}, Locality Sensitive Hashing \citep{chen2025magicpig}, or learned projections \citep{desai2024hashattention}. Hardware-aligned variants further structure selection into memory tiles: Quest \citep{tang2024quest} screens KV pages using min/max bounds, SpargeAttention \citep{zhang2025spargeattention} filters blocks via compressed inner products, and SeerAttention \citep{gao2024seerattention} distills a 2D block gate from a dense teacher. Furthermore, \citet{hu2022adaptive} explore differentiable score thresholding for Named Entity Recognition (NER), though their formulation relies on a tuned top-$k$ hyperparameter, hard-masks unselected tokens to $-\infty$, and lacks a custom kernel to avoid quadratic materialization.

\paragraph{Trainable Sparse Attention.}
Recent architectures seek to incorporate sparsity natively into pretraining. HiRE \citep{yerram2024hire} explores learnable compression for low-rank attention, and DMA \citep{shi2025trainable} mixes fixed masking patterns using learned gates. Native Sparse Attention (NSA) \citep{yuan2025native} achieves notable quality by combining sliding-window, compressed-coarse, and selected-fine branches. However, NSA's multi-branch design requires substantial architectural reconfiguration and tuning. In contrast, ETA introduces a minimal, drop-in modification ($<0.1\%$ parameter overhead) that preserves standard Transformer execution while accelerating decoding.

\section{Elastic Threshold Attention}
\label{sec:method}
\label{sec:background}
We consider a standard decoder-only Transformer with $L$ layers and $H$ attention heads (full preliminaries in \autoref{app:transformer_math}), where query, key, and value vectors $q_t, k_s, v_s \in \mathbb{R}^{d_h}$ at positions $s \le t \le T$ compute dense attention output
$$
o_t = \sum_{s=1}^{t} A_{ts} v_s,\quad\text{with }A_{ts} = \frac{\exp(\langle q_t, k_s \rangle / \sqrt{d_h})}{\sum_{j=1}^{t} \exp(\langle q_t, k_j \rangle / \sqrt{d_h})}
$$ 
While prefill requires $\Theta(T^2 d_h)$ FLOPs, autoregressive decoding is strictly memory-bandwidth bound because generating each token $t$ requires streaming all cached past keys and values from GPU High-Bandwidth Memory (HBM) to SRAM at every step. Sparse attention addresses this bottleneck by restricting attention to an active index subset $\mathcal{I}_t \subseteq \{1, \dots, t\}$ with $|\mathcal{I}_t| \ll t$. Converting theoretical index sparsity into wall-clock speedups requires token selection to align with coalesced memory blocks on hardware, and the model to remain accurate when unselected context is omitted.

\subsection{Threshold Prediction}
In an effort to satisfy both desiderata, we propose Elastic Threshold Attention (ETA). ETA learns sparsity patterns by quantifying key relevance through query-key inner products $\langle q_t, k_s \rangle/\!\sqrt{d_h}$. Tokens falling below an adaptive threshold $\tau_t$ are suppressed. We predict $\tau_t$ using a per-layer linear projection on post-RoPE query vectors across all $H$ attention heads:
\begin{equation}
\bar{q}_t^{\ell} = \big[\, q_{t,1}^{\ell} \,\big\|\, \dots \,\big\|\, q_{t,H}^{\ell} \,\big] \in \mathbb{R}^{H d_h}, \qquad \tau_t^{\ell} = \bar{q}_t^{\ell} W_\tau^{\ell} + b_\tau^{\ell} \in \mathbb{R}^H
\label{eq:tau_pred}
\end{equation}
where $W_\tau^\ell \in \mathbb{R}^{H d_h \times H}$ and $b_\tau^\ell \in \mathbb{R}^H$ are learnable parameters. Conditioning on concatenated post-RoPE queries provides two essential benefits: (1) rotary positional information is preserved, and (2) all heads within a layer coordinate jointly, allowing dynamic allocation of attention budgets across heads.

To ensure training stability and avoid premature pruning, we employ a \emph{start-dense} initialization: projection weights are zero-initialized ($W_\tau^\ell = 0$) and biases are set to $b_\tau^\ell = -8.0$. Initially, $100\%$ of keys pass the attention mask, enabling a smooth transition from dense to sparse attention.

\subsection{Differentiable Masking and Tail Tempering}
To allow gradient propagation to $W_\tau^\ell$, we formulate a continuous relaxation using an annealed sigmoid gate. For query $t$ and key $s$ at head $h$, let $S_{ts}^h = \langle q_t, k_s \rangle / \sqrt{d_h}$ denote the pre-softmax score. The soft inclusion probability $m_{ts}^h \in (0, 1)$ is:
\begin{equation}
m_{ts}^h = \sigma\Big(\beta \cdot (S_{ts}^h - \tau_{t,h}^\ell)\Big)
\label{eq:soft_mask}
\end{equation}
where $\sigma(x) = \frac{1}{1+e^{-x}}$ and temperature $\beta \ge 1$ is annealed linearly from $1.0$ to $5.0$ over pre-training. Annealing provides high gradient flow early in training to escape initialization.

We then apply multiplicative score gating. Let $M_{ts}^{\text{causal}} = \begin{cases}
    0 &\text{if } s \leq t\\
    -\infty &\text{otherwise}
\end{cases}$ and define:
\begin{equation}
A_{ts}^h = \text{softmax}\left( (S_{ts}^h \odot M_{ts}) + M_{ts}^{\text{causal}} \right), \quad M_{ts} = \begin{cases} 1 & \text{if } s = t \\ m_{ts}^h & \text{otherwise} \end{cases}
\label{eq:mult_attn}
\end{equation}

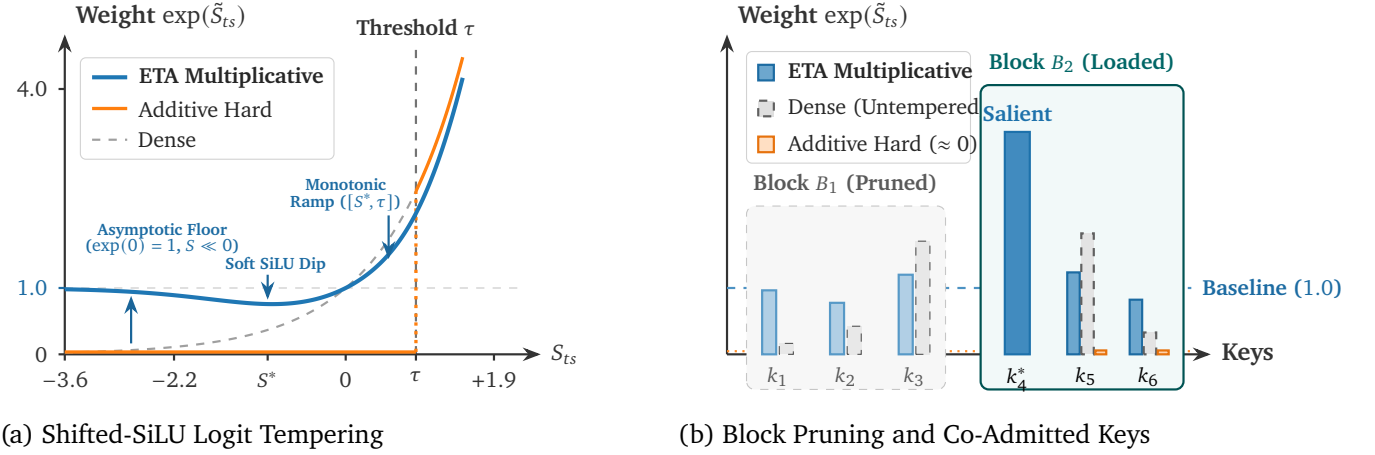
\begin{figure}[t]
  \centering
  \begin{subfigure}[b]{0.46\textwidth}
    \centering
    \resizebox{\linewidth}{!}{%
    \begin{tikzpicture}[font=\sffamily, >=Stealth]
      \useasboundingbox (-0.8, -0.65) rectangle (6.6, 4.35);

      \draw[dashed, gray!45] (0, 0.85) -- (5.9, 0.85);
      \draw[dashed, thick, black!55] (4.5, 0) -- (4.5, 3.95) node[above, font=\scriptsize\bfseries, text=black!80] {Threshold $\tau$};

      \draw[->, thick, black!80] (0,0) -- (6.1,0) node[right, font=\footnotesize\bfseries] {$S_{ts}$};
      \draw[->, thick, black!80] (0,0) -- (0,4.1) node[above right, font=\footnotesize\bfseries, yshift=-2pt] {Weight $\exp(\tilde{S}_{ts})$};

      \foreach \x/\lbl in {0.0/$-3.6$, 1.4/$-2.2$, 2.6/$S^*$, 3.6/$0$, 4.5/$\tau$, 5.5/$+1.9$} {
        \draw[thick, black!80] (\x, 0) -- (\x, -0.08) node[below, font=\scriptsize] {\lbl};
      }

      \draw[thick, black!80] (0, 0) -- (-0.08, 0) node[left, font=\scriptsize] {$0$};
      \draw[thick, etablue!90!black] (0, 0.85) -- (-0.08, 0.85) node[left, font=\scriptsize\bfseries, text=etablue!90!black] {$1.0$};
      \draw[thick, black!80] (0, 3.40) -- (-0.08, 3.40) node[left, font=\scriptsize] {$4.0$};

      \draw[domain=0.0:5.10, samples=60, smooth, dashed, thick, gray!75] plot (\x, {0.85*exp(\x - 3.6)});

      \draw[very thick, nsaorange] (0.0, 0.03) -- (4.5, 0.03);
      \draw[very thick, nsaorange, dotted] (4.5, 0.03) -- (4.5, {0.85*exp(0.9)});
      \draw[domain=4.5:5.10, samples=30, smooth, very thick, nsaorange] plot (\x, {0.85*exp(\x - 3.6)});

      \draw[domain=0.0:5.10, samples=90, smooth, line width=1.5pt, etablue] plot (\x, {0.85*exp((\x - 3.6) * (0.78 / (1 + exp(min(7.0, -1.7*(\x - 2.95)))) + 0.22 / (1 + exp(min(7.0, -3.2*(\x - 4.5))))))});

      \draw[->, thick, etablue!85!black] (0.85, 0.14) -- (0.85, 0.75);
      \node[font=\tiny\bfseries, text=etablue!90!black, align=center] at (1.25, 1.48) {Asymptotic Floor\\[-1pt]{($\exp(0)=1$, $S\ll 0$)}};
      \draw[->, thick, etablue!85!black] (2.60, 1.02) -- (2.60, 0.70);
      \node[font=\tiny\bfseries, text=etablue!90!black, align=center] at (2.70, 1.15) {Soft SiLU Dip};
      \draw[->, thick, etablue!85!black] (4.15, 1.78) -- (4.15, 1.22);
      \node[font=\tiny\bfseries, text=etablue!90!black, align=center] at (3.60, 2.05) {Monotonic\\[-1pt]Ramp ($[S^*\!,\tau]$)};

      \draw[fill=white, draw=gray!35, rounded corners=2pt] (0.2, 2.50) rectangle (3.45, 3.80);
      \draw[line width=1.5pt, etablue] (0.35, 3.55) -- (0.80, 3.55) node[right, font=\scriptsize\bfseries, text=black!85] {ETA Multiplicative};
      \draw[very thick, nsaorange] (0.35, 3.15) -- (0.80, 3.15) node[right, font=\scriptsize, text=black!85] {Additive Hard};
      \draw[dashed, thick, gray!75] (0.35, 2.75) -- (0.80, 2.75) node[right, font=\scriptsize, text=black!85] {Dense};
    \end{tikzpicture}%
    }
    \caption{Shifted-SiLU Logit Tempering}
    \label{fig:tail_contraction_curve}
  \end{subfigure}\hfill
  \begin{subfigure}[b]{0.46\textwidth}
    \centering
    \resizebox{\linewidth}{!}{%
    \begin{tikzpicture}[font=\sffamily, >=Stealth]
      \useasboundingbox (-0.6, -0.65) rectangle (6.8, 4.35);

      \draw[->, thick, black!80] (0,0) -- (6.2,0) node[right, font=\footnotesize\bfseries] {Keys};
      \draw[->, thick, black!80] (0,0) -- (0,4.1) node[above right, font=\footnotesize\bfseries, yshift=-2pt] {Weight $\exp(\tilde{S}_{ts})$};

      \draw[dashed, thick, etablue!80] (0, 0.85) -- (5.95, 0.85) node[right, font=\scriptsize\bfseries, text=etablue!90!black] {Baseline ($1.0$)};
      \draw[dotted, thick, nsaorange!90] (0, 0.04) -- (5.95, 0.04) node[right, font=\scriptsize\bfseries, text=nsaorange!90!black] {};

      \draw[fill=gray!6, draw=gray!45, dashed, rounded corners=3pt] (0.25, -0.45) rectangle (2.80, 1.90);
      \node[above, font=\scriptsize\bfseries, text=gray!70!black] at (1.52, 1.93) {Block $B_1$ (Pruned)};
      \foreach \x/\eta/\raw/\lbl in {0.65/0.82/0.14/k_1, 1.52/0.66/0.36/k_2, 2.40/1.02/1.45/k_3} {
        \fill[etablue!35] (\x-0.20, 0) rectangle (\x-0.02, \eta);
        \draw[etablue!75, thick] (\x-0.20, 0) rectangle (\x-0.02, \eta);
        \fill[gray!18] (\x+0.02, 0) rectangle (\x+0.20, \raw);
        \draw[gray!65!black, thin, dashed] (\x+0.02, 0) rectangle (\x+0.20, \raw);
        \node[below, font=\scriptsize, text=gray!70!black] at (\x, -0.03) {$\lbl$};
      }

      \draw[fill=white, draw=gray!35, rounded corners=2pt] (0.25, 2.45) rectangle (3.12, 3.88);
      \fill[etablue!65] (0.40, 3.51) rectangle (0.60, 3.71);
      \draw[etablue!90!black, thick] (0.40, 3.51) rectangle (0.60, 3.71);
      \node[right, font=\scriptsize\bfseries, text=black!85] at (0.65, 3.61) {ETA Multiplicative};
      \fill[gray!25] (0.40, 3.05) rectangle (0.60, 3.25);
      \draw[gray!75!black, thick, dashed] (0.40, 3.05) rectangle (0.60, 3.25);
      \node[right, font=\scriptsize, text=black!85] at (0.65, 3.15) {Dense (Untempered)};
      \fill[nsaorange!25] (0.40, 2.59) rectangle (0.60, 2.79);
      \draw[nsaorange!90!black, thick] (0.40, 2.59) rectangle (0.60, 2.79);
      \node[right, font=\scriptsize, text=black!85] at (0.65, 2.69) {Additive Hard ($\approx 0$)};

      \draw[fill=teal!5, draw=teal!65!black, thick, rounded corners=3pt] (3.25, -0.45) rectangle (5.85, 3.45);
      \node[above, font=\scriptsize\bfseries, text=teal!80!black] at (4.55, 3.48) {Block $B_2$ (Loaded)};

      \fill[etablue!85] (3.72-0.16, 0) rectangle (3.72+0.16, 2.85);
      \draw[etablue!95!black, thick] (3.72-0.16, 0) rectangle (3.72+0.16, 2.85);
      \node[above, font=\scriptsize\bfseries, text=etablue!95!black] at (3.72, 2.85) {Salient};
      \node[below, font=\scriptsize\bfseries] at (3.72, -0.03) {$k_4^*$};

      \foreach \x/\eta/\raw/\lbl in {4.60/1.05/1.55/k_5, 5.40/0.70/0.28/k_6} {
        \fill[etablue!65] (\x-0.24, 0) rectangle (\x-0.08, \eta);
        \draw[etablue!90!black, thick] (\x-0.24, 0) rectangle (\x-0.08, \eta);
        \fill[gray!22] (\x-0.06, 0) rectangle (\x+0.10, \raw);
        \draw[gray!75!black, thick, dashed] (\x-0.06, 0) rectangle (\x+0.10, \raw);
        \fill[nsaorange!35] (\x+0.12, 0) rectangle (\x+0.26, 0.05);
        \draw[nsaorange!90!black, thick] (\x+0.12, 0) rectangle (\x+0.26, 0.05);
        \node[below, font=\scriptsize] at (\x, -0.03) {$\lbl$};
      }
    \end{tikzpicture}%
    }
    \caption{Block Pruning and Co-Admitted Keys}
    \label{fig:block_overinclusion_buffer}
  \end{subfigure}
  \caption{\textbf{SiLU-gated tail tempering vs.\ support truncation.} \textbf{(a)} A shifted $\operatorname{SiLU}$ in logit space tempers the attention distribution without dropping tokens from its support. \textbf{(b)} Across pruned blocks ($B_1$) and co-admitted keys ($k_5, k_6$) in loaded blocks ($B_2$), tempering compresses wide raw weight spreads (gray) into a tightly bounded band around $1.0$ while preserving near-threshold order ($k_2 < k_3$ and $k_6 < k_5$).}
  \label{fig:tail_flattening_intuition}
\end{figure}
Multiplicative score gating is equivalent to passing the attention logits through a shifted $\operatorname{SiLU}$ gate. Let $\tilde{S}(S) = S\cdot \sigma(\beta(S-\tau))$ denote the gated scores. Taking the derivative at $\tau$,
$$
\left.\frac{d\tilde{S}}{dS}\right|_{S=\tau} = \left.M_{ts}[1+\beta S(1-M_{ts})]\right|_{S=\tau} = \frac{1}{2}\left(1+\frac{\beta\tau}{2}\right),
$$
reveals two operating regimes across our trained models:
\begin{enumerate}[leftmargin=*]
    \item \textbf{Regime 1: $\tau \lesssim -\frac{2}{\beta}$}. This is the regime ETA occupies at initialization ($b_\tau^\ell = -8$) and retains in Layer~$0$ after training ($\tau \in [-9.3, -4.8]$). Here the gate acts essentially as dense attention since $\tilde{S}(S) \approx S$ throughout the support of $S$.
    \item \textbf{Regime 2: $\tau > 0 > -\frac{2}{\beta}$}. This is the regime $99.9\%$ of heads in Layers~$1\text{--}21$ converge to after training (median $\tau = +2.81$). Here the mechanism acts as a \textit{two-sided contraction toward zero} (\autoref{fig:tail_contraction_curve}): negative scores ($S < 0$) are pulled upward toward $0$, while positive sub-threshold scores ($0 < S < \tau$) are damped downward below $S$. Solving $\frac{d\tilde{S}}{dS} = 0$ bounds the unique global minimizer $S^*$ of $\tilde{S}$ analytically\footnote{$W_0$ denotes the principal branch of the Lambert-$W$ function, with $0 < W_0(e^{-1-\beta\tau}) < W_0(e^{-1}) \approx 0.278$ for $\tau > 0$.}:
    $$
    -\frac{1.28}{\beta} < S^*(\beta,\tau) = -\frac{1+W_0(e^{-1-\beta\tau)}}{\beta} < -\frac{1}{\beta} < 0, \qquad \tilde{S}(S^*) = -\frac{W_0(e^{-1-\beta\tau})}{\beta} \approx 0.
    $$
    Consequently, $\tilde{S}$ remains strictly monotonically increasing on $[S^*,\infty)$ while the entire sub-threshold tail $(-\infty, \tau)$ is \textit{tempered} into a low-variance, high-entropy band around $\exp(0) = 1$: normalized tail entropy rises from $0.790$ to $0.916$ and the max-to-mean tail weight ratio $\kappa = w_{\max}/\bar{w}$ drops from $15.6$ to $2.03$ (\autoref{tab:masking_formulation_comparison}).
\end{enumerate}

\paragraph{Invariance to block pruning and over-inclusion.}
At block-sparse inference (\autoref{fig:block_overinclusion_buffer}), the causal context $\mathcal{S}_t = \{1, \dots, t\}$ partitions into three disjoint sets: salient keys $\mathcal{K}_t = \{s : S_{ts}^h \ge \tilde{\tau}_t\}$ ($k_4^*$), co-admitted sub-threshold keys $\mathcal{C}_t$ inside loaded blocks ($k_5, k_6$ in $B_2$), and dropped sub-threshold keys $\mathcal{D}_t$ inside pruned blocks ($k_1\text{--}k_3$ in $B_1$). Both training and inference evaluate the same gated weights $w_s := \exp(\tilde{S}_{ts}^h)$ on all retained tokens. For any token subset $\mathcal{X} \subseteq \mathcal{S}_t$, let $Z_{\mathcal{X}} := \sum_{s \in \mathcal{X}} w_s$ denote its partition sum and $\tilde{v}_{\mathcal{X}} := \frac{1}{Z_{\mathcal{X}}} \sum_{s \in \mathcal{X}} w_s v_s$ its locally normalized attention output.

Then the token-pruned output $o_t^{\mathrm{token}} := \tilde{v}_{\mathcal{K}_t}$ (retaining $\mathcal{K}_t$), the block-sparse inference output $o_t^{\mathrm{block}} := \tilde{v}_{\mathcal{K}_t \cup \mathcal{C}_t}$ (loading blocks $\mathcal{K}_t \sqcup \mathcal{C}_t$ while skipping $\mathcal{D}_t$), and the training output $o_t^{\mathrm{train}} := \tilde{v}_{\mathcal{K}_t \cup \mathcal{C}_t \cup \mathcal{D}_t}$ (attending over all $\mathcal{S}_t$) satisfy the exact decomposition:
\begin{equation}
o_t^{\mathrm{train}} = \frac{Z_{\mathcal{K}} o_t^{\mathrm{token}} + Z_{\mathcal{C}} \tilde{v}_{\mathcal{C}} + Z_{\mathcal{D}} \tilde{v}_{\mathcal{D}}}{Z_{\mathcal{K}} + Z_{\mathcal{C}} + Z_{\mathcal{D}}}, \qquad
o_t^{\mathrm{block}} = \frac{Z_{\mathcal{K}} o_t^{\mathrm{token}} + Z_{\mathcal{C}} \tilde{v}_{\mathcal{C}}}{Z_{\mathcal{K}} + Z_{\mathcal{C}}}.
\label{eq:convex_output_decomp}
\end{equation}
Because SiLU tempering compresses the max-to-mean tail weight ratio ($\kappa = 2.03$ vs.\ $15.6$), we show under simplifying statistical assumptions in \autoref{appx:log_odds_invariance} that both tail averages $\tilde{v}_{\mathcal{C}}$ and $\tilde{v}_{\mathcal{D}}$ cancel towards zero at rate $\mathcal{O}(\sqrt{\kappa/|\mathcal{X}|})$ for $\mathcal{X} \in \{\mathcal{C}_t,\mathcal{D}_t\}$ (\autoref{prop:uniform_floor_cancellation}). Dropping these near-zero vectors leaves all three outputs pointing along the exact same active signal vector $o_t^{\mathrm{token}}$:
\begin{equation}
o_t^{\mathrm{train}} \approx \frac{Z_{\mathcal{K}}}{Z_{\mathcal{K}} + Z_{\mathcal{C}} + Z_{\mathcal{D}}}\, o_t^{\mathrm{token}}, \qquad
o_t^{\mathrm{block}} \approx \frac{Z_{\mathcal{K}}}{Z_{\mathcal{K}} + Z_{\mathcal{C}}}\, o_t^{\mathrm{token}}.
\end{equation}
Since $\operatorname{RMSNorm}$ is invariant to positive scalar multiplication, neither dropping pruned blocks $\mathcal{D}_t$ nor co-admitting sub-threshold keys $\mathcal{C}_t$ alters the post-norm representation significantly; on the 22-layer 1.45B model, tail value norms follow the theoretical $1/\sqrt{|\mathcal{K}_t|+|\mathcal{C}_t|}$ decay and block-sparse decoding maintains a $0.9917$ residual cosine similarity with training-time gating (\autoref{appx:empirical_alignment}, \autoref{tab:prop1_verification_145b}, \autoref{tab:representation_alignment}). 

Co-admitting $\mathcal{C}_t$ actually moves $o_t^{\mathrm{block}}$ strictly closer to $o_t^{\mathrm{train}}$ than exact token pruning $o_t^{\mathrm{token}}$ (\autoref{prop:overinclusion_interpolation}). By contrast, additive hard masking sets $Z_{\mathcal{C}} = Z_{\mathcal{D}} = 0$ during training ($o_t^{\mathrm{train,add}} = o_t^{\mathrm{token}}$) and puts zero gradient on tail values, so co-admitted keys $\mathcal{C}_t$ at inference push $o_t^{\mathrm{block}}$ away from $o_t^{\mathrm{train,add}}$ and degrade PPL by $+16.6\%$ (\autoref{appx:additive_masking}, \autoref{tab:masking_formulation_comparison}).

\paragraph{Sparsity regularization.}
To incentivize sparsity during pre-training, we penalize the mean inclusion probability across layers $L$, heads $H$, and positions $T$:
\begin{equation}
\mathcal{L}_{\text{sparsity}} = \frac{\lambda_{\text{sparsity}}}{LHT} \sum_{\ell=1}^L \sum_{h=1}^H \sum_{t=1}^T \frac{1}{t} \sum_{s \le t} m_{ts}^h
\label{eq:l_sparsity}
\end{equation}
Because $m_{ts}^h$ is smooth, $\mathcal{L}_{\text{sparsity}}$ supplies direct gradients $\partial \mathcal{L}_{\text{sparsity}} / \partial \tau$, balancing predictive loss against computational budget.

\subsection{Efficient Fused Training Kernel}
To prevent materializing the dense $T \times T$ attention matrix in HBM, we implement a custom fused Triton kernel \citep{tillet2019triton}. As a result, ETA pre-training incurs the exact same $\mathcal{O}(T^2 d)$ FLOPs, $\mathcal{O}(Td)$ HBM footprint, and $\mathcal{O}(T^2 d^2 M^{-1})$ I/O complexity as dense FlashAttention-2 \citep{dao2022flashattention}. Forward/backward pseudocode (\autoref{alg:fwd_kernel}, \autoref{alg:bwd_kernel}), exact gradient derivations (\autoref{app:attention_gradients}), and complexity analyses (\autoref{appx:complexities}) are detailed in \autoref{app:triton_kernel}.

\section{Hardware-Aligned Inference Flow}
\label{sec:inference}
In line with our formalization, during autoregressive decoding, ETA avoids loading non-salient KV blocks $(\mathcal{D}_t)$ into SRAM. Realizing wall-clock speedups requires identifying non-significant blocks in $\mathcal{O}(1)$ time without computing individual dot products.

\subsection{Block-Based Scanning via Dual Metadata Caching}
To align with GPU memory layouts, we partition the KV cache into contiguous blocks of size $b$. A block is termed \emph{significant} if it contains at least one token with attention score $S_{ts}^h \ge \tau_t$.

To filter non-significant blocks in $\mathcal{O}(1)$ time without loading full KV tiles from HBM, we cache a \textbf{Dual Probabilistic-Geometric Block Index} storing three summary statistics per block $B$: the key centroid $\boldsymbol{\mu}_B = \frac{1}{b}\sum_{k \in B} k \in \mathbb{R}^{d_h}$, the coordinate-wise variance $\boldsymbol{\sigma}_B^2 = \frac{1}{b}\sum_{k \in B} (k - \boldsymbol{\mu}_B)^{\odot 2} \in \mathbb{R}^{d_h}$, and the maximum Euclidean norm $M_B = \max_{k \in B} \|k\|_2 \in \mathbb{R}$. Given a query $q \in \mathbb{R}^{d_h}$, we estimate the maximum in-block dot product by intersecting a directional probabilistic moment bound with the exact Cauchy--Schwarz spherical norm ceiling:
\begin{equation}
\operatorname{Score}_{\text{dual}}(q, B) = \frac{1}{\sqrt{d_h}} \min\!\left( \langle q, \boldsymbol{\mu}_B \rangle + z \sqrt{\sum_{i=1}^{d_h} q_i^2 \sigma_{B,i}^2}, \;\; \|q\|_2 M_B \right).
\label{eq:dual_bound_eq}
\end{equation}
Any block where $\operatorname{Score}_{\text{dual}}(q, B) < \tau_t$ is bypassed during HBM loading\footnote{In practice we often loosen the threshold by adding a slack constant $\alpha$ to obtain $\tilde{\tau}_t = \tau_t - \alpha$}, reducing decode complexity from $\mathcal{O}(n \cdot d)$ to $\mathcal{O}\big((\lceil n/b \rceil + m \cdot b) d\big)$ for $m \ll \lceil n/b \rceil$ retrieved blocks. At coarse hardware tile sizes ($b=64$), we can also optionally replace $\boldsymbol{\sigma}_B$ with a sub-block envelope standard deviation and pin the two most recent blocks to prevent variance dilution without increasing the metadata footprint (see \autoref{appx:screening_index_quality} and \autoref{appx:dual_index_math}).

\subsection{Fused Triton Kernel Implementation with Top-1 Rescue}
We implement block-sparse decoding as a custom fused Triton kernel \citep{tillet2019triton}. As illustrated in \autoref{fig:inference_block_selective_load} (with full algorithmic details and pseudocode in \autoref{appx:fused-kernel}, \autoref{alg:inference_kernel}), the kernel streams block metadata into SRAM to evaluate $\operatorname{Score}_{\text{dual}}(q, B)$ on the fly, skipping HBM transfers for pruned blocks. To maximize Streaming Multiprocessor (SM) occupancy, the sequence is partitioned into independent block chunks processed in parallel and merged via associative online-softmax reduction. If an aggressive threshold prunes all candidate blocks in the first chunk ($k=1$), a localized \textit{Top-1 Rescue Mechanism} unconditionally admits the single highest-scoring block, guaranteeing numerical stability without cross-SM synchronization (\autoref{appx:fused-kernel}). Because metadata is indexed per KV head ($H_{KV}$), the footprint scales as $\mathcal{O}(1/b)$, adding merely $1.6\%$ memory overhead at $b=64$ in a held-out table compatible with paged caching. Context extrapolation up to $8192$ tokens ($4\times L_{\text{train}}$) is preserved via NTK-aware RoPE scaling \citep{bloc972023ntk}.

\renewcommand{\topfraction}{0.95}
\renewcommand{\bottomfraction}{0.95}
\renewcommand{\textfraction}{0.04}
\renewcommand{\floatpagefraction}{0.90}
\section{Experiments}
\label{sec:experiments}

\subsection{Training Dynamics and Regularization Scheduling}
\label{subsec:training_dynamics}
We train a 1.45B parameter ETA language model from scratch on FineWeb \citep{penedo2024fineweb} using a LLaMA-style backbone \citep{llama3modelcard}. See \autoref{appx:training} for more details on training. To isolate the effect of trainable thresholding, we train an identical dense baseline under matching hyperparameters and optimizer settings.

As illustrated in \autoref{fig:training}, ETA converges to pretraining loss parity with the dense baseline while achieving $85\%$ effective soft sparsity ($1 - \overline{m_{ts}^h}$) across all heads and layers. Although start-dense initialization ($b_\tau^\ell = -8$) begins in a low-gradient regime ($\sigma'(8) = 3.35 \times 10^{-4}$), Adam's scale-invariant updates combined with the $\beta$-annealing strategy drive rapid threshold migration into the score distribution within $\sim 10^2$ optimizer steps (\autoref{appx:training_escape}).

\begin{figure}[!htbp]
\vspace{-2pt}
    \centering
    \begin{subfigure}{0.48\textwidth}
        \centering
        \includegraphics[width=0.84\linewidth, keepaspectratio]{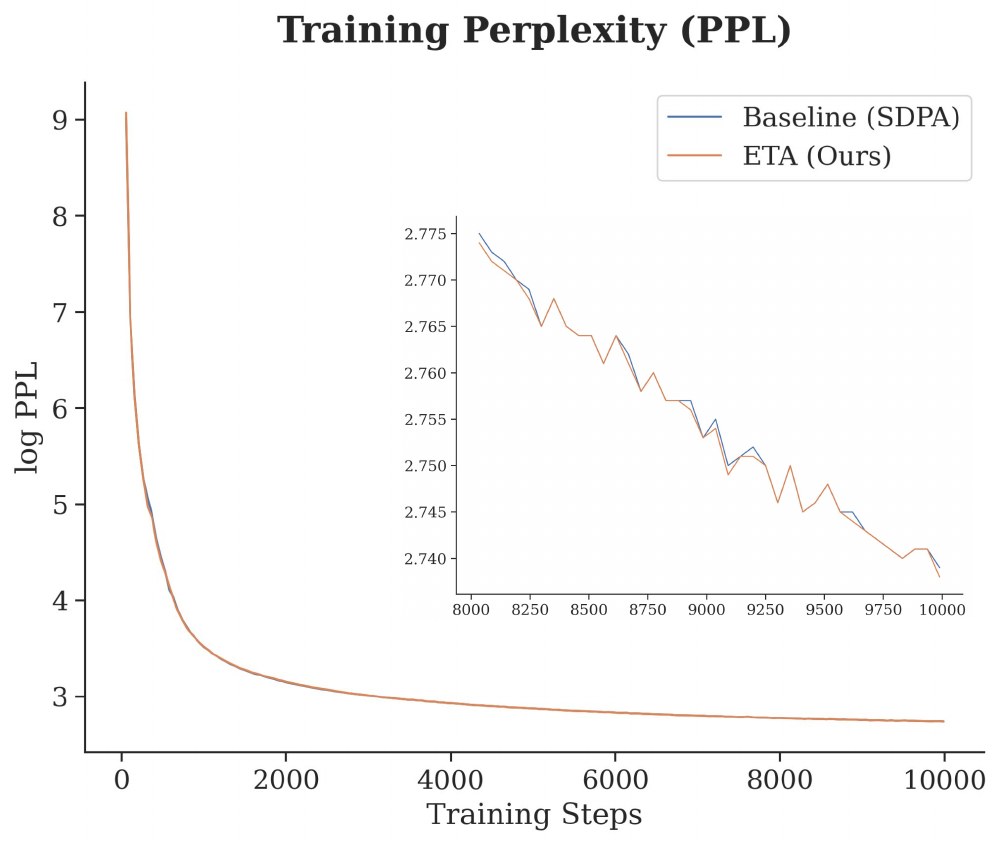}
        \label{fig:losses}
    \end{subfigure}
    \hfill
    \begin{subfigure}{0.48\textwidth}
        \centering
        \includegraphics[width=0.84\linewidth, keepaspectratio]{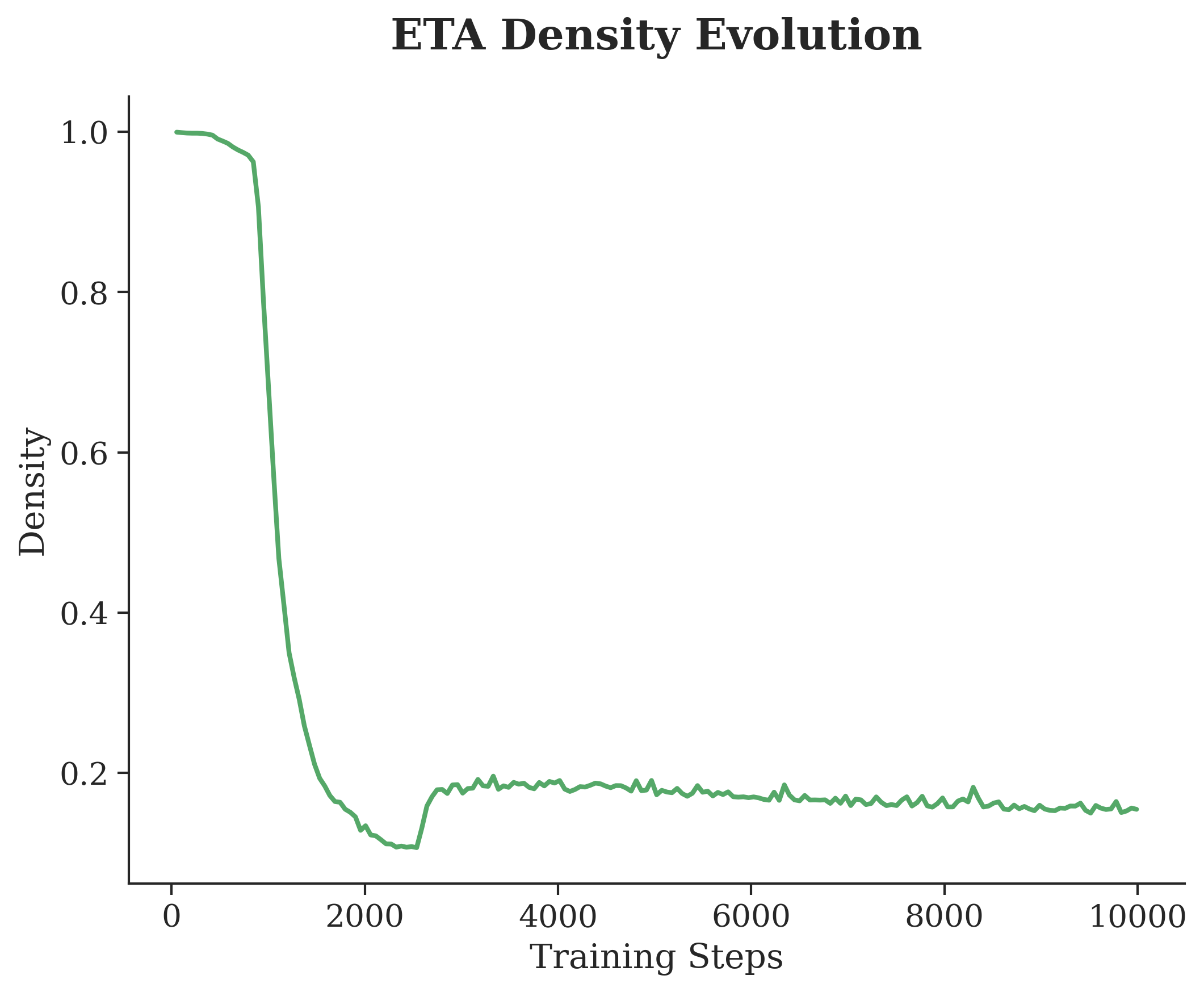}
        \label{fig:sparsity-training}
    \end{subfigure}
    \vspace{-2pt}
    \caption{\textbf{Training dynamics of ETA against standard dense attention.} ETA achieves loss parity with dense attention while maintaining $\approx 85\%$ average soft sparsity across all heads and layers.}
    \label{fig:training}
\vspace{-2pt}
\end{figure}

\subsection{Inference-Time Quality and Knowledge Benchmarks}
\label{subsec:inference_benchmarks}
We evaluate the 1.45B model at native context length ($L=2048$) on five standard benchmarks: WikiText-2 \citep{merity2016pointer}, FineWeb, and C4 \citep{JMLR:v21:20-074} for language modeling perplexity, and ARC-Easy \citep{clark2018think} and HellaSwag \citep{zellers2019hellaswag} for zero-shot common-sense reasoning. We compare both fine-grained ETA ($b=4$) and our hardware-aligned Tensor Core configuration ($b=64$) against dense attention (FlashAttention-2 / SDPA), Sliding Window Attention (SWA-256), StreamingLLM \citep{xiao2024efficient} ($W,S=256,4$), BigBird \citep{zaheer2020big} ($W,S,\text{Str}=256,4,16$), and $\text{H}_2\text{O}$ \citep{zhang2023h2o} ($W,S,\text{HH}=256,4,10\%$).

\begin{table}[h]
\centering
\midsize
\renewcommand{\arraystretch}{0.92} 
\setlength{\tabcolsep}{4.5pt} 
\caption{\textbf{Language modeling perplexity and zero-shot reasoning accuracy at $L=2048$ on the 1.45B model.} Lower sub-text indicates the average fraction of past tokens attended per head. Because ARC-Easy and HellaSwag prompts ($25\text{--}110$ tokens) are shorter than the $W=256$ local window of the heuristic baselines, those methods degenerate to $100\%$ dense attention on reasoning tasks, whereas fine-grained ETA actively prunes redundant tokens.}
\label{tab:inference_experiments_stats}
\resizebox{\textwidth}{!}{%
\begin{tabular}{l c c c c c}
\toprule
\textbf{Method Variant} & \textbf{\shortstack{WikiText-2\\PPL ($\downarrow$)}} & \textbf{\shortstack{FineWeb\\PPL ($\downarrow$)}} & \textbf{\shortstack{C4\\PPL ($\downarrow$)}} & \textbf{\shortstack{ARC-Easy\\Acc (\%) ($\uparrow$)}} & \textbf{\shortstack{HellaSwag\\Acc (\%) ($\uparrow$)}} \\
\midrule
\textbf{Dense Baseline} & 
\shortstack{$19.35 \pm 0.80$ \\ \tiny 100\%} & 
\shortstack{$16.19 \pm 1.25$ \\ \tiny 100\%} & 
\shortstack{$17.49 \pm 0.33$ \\ \tiny 100\%} & 
\shortstack{$44.0\% \pm 3.51\%$ \\ \tiny 100\%} & 
\shortstack{$36.0\% \pm 3.39\%$ \\ \tiny 100\%} \\
\cmidrule(lr){1-6} \addlinespace[1pt] 
\textbf{$\text{H}_2\text{O}$} & 
\shortstack{$20.44 \pm 0.80$ \\ \tiny 33.7\%} & 
\shortstack{$16.39 \pm 1.29$ \\ \tiny 33.4\%} & 
\shortstack{$17.72 \pm 0.69$ \\ \tiny 33.1\%} & 
\shortstack{$44.0\% \pm 3.51\%$ \\ \tiny 100\%} & 
\shortstack{$36.0\% \pm 3.39\%$ \\ \tiny 100\%} \\
\textbf{BigBird} & 
\shortstack{$22.54 \pm 0.90$ \\ \tiny 28.5\%} & 
\shortstack{$16.81 \pm 1.33$ \\ \tiny 28.5\%} & 
\shortstack{$18.04 \pm 0.74$ \\ \tiny 28.5\%} & 
\shortstack{$44.0\% \pm 3.51\%$ \\ \tiny 100\%} & 
\shortstack{$36.0\% \pm 3.39\%$ \\ \tiny 100\%} \\
\textbf{StreamingLLM} & 
\shortstack{$23.06 \pm 0.91$ \\ \tiny 23.8\%} & 
\shortstack{$16.90 \pm 1.34$ \\ \tiny 23.8\%} & 
\shortstack{$18.12 \pm 0.74$ \\ \tiny 23.8\%} & 
\shortstack{$44.0\% \pm 3.51\%$ \\ \tiny 100\%} & 
\shortstack{$36.0\% \pm 3.39\%$ \\ \tiny 100\%} \\
\textbf{SWA-256} & 
\shortstack{$25.63 \pm 0.99$ \\ \tiny 23.4\%} & 
\shortstack{$18.32 \pm 1.48$ \\ \tiny 23.4\%} & 
\shortstack{$19.83 \pm 0.82$ \\ \tiny 23.4\%} & 
\shortstack{$44.0\% \pm 3.51\%$ \\ \tiny 100\%} & 
\shortstack{$36.0\% \pm 3.39\%$ \\ \tiny 100\%} \\
\cmidrule(lr){1-6} \addlinespace[1pt]
\textbf{ETA ($b=4, z=2.0$)} & 
\hlprob{\shortstack{\textbf{\boldmath $19.39 \pm 0.78$} \\ \tiny 42.7\%}} & 
\hlprob{\shortstack{\textbf{\boldmath $16.31 \pm 1.26$} \\ \tiny 38.2\%}} & 
\hlprob{\shortstack{\textbf{\boldmath $17.56 \pm 0.33$} \\ \tiny 36.3\%}} & 
\hlprob{\shortstack{\textbf{\boldmath $44.5\% \pm 3.51\%$} \\ \tiny 91.3\%}} & 
\hlprob{\shortstack{\textbf{\boldmath $40.0\% \pm 3.46\%$} \\ \tiny 89.7\%}} \\
\textbf{ETA ($b=64, a=-0.4$)} & 
\hlprob{\shortstack{$20.68 \pm 0.81$ \\ \tiny 55.0\%}} & 
\hlprob{\shortstack{$17.11 \pm 1.39$ \\ \tiny 50.2\%}} & 
\hlprob{\shortstack{$18.60 \pm 0.29$ \\ \tiny 48.3\%}} & 
\hlprob{\shortstack{$44.0\% \pm 3.51\%$ \\ \tiny 100\%}} & 
\hlprob{\shortstack{$36.0\% \pm 3.39\%$ \\ \tiny 100\%}} \\
\bottomrule
\end{tabular}%
}
\vspace{-4pt}
\end{table}

As shown in \autoref{tab:inference_experiments_stats}, static-pattern baselines (SWA-256, StreamingLLM, and BigBird) incur substantial perplexity degradation across all three language modeling corpora because they blindly discard distant context once tokens exit the fixed local window. Fine-grained ETA ($b=4$) achieves dense-level perplexity across all three corpora and attains the best performance on every language modeling and zero-shot reasoning benchmark. When scaling block size by $16\times$ to our hardware-aligned Tensor Core deployment ($b=64$), ETA continues to remain competitive across all corpora while preserving full reasoning accuracy.

\subsection{Block Granularity and GQA Union Loading}
\label{sec:blocksize}
\label{sec:gqa_union}

\begin{wraptable}{r}{0.48\linewidth}
\vspace{-6pt}
\centering
\caption{\textbf{Block size at matched ${\sim}51\%$ GQA-union bandwidth} and $126$M C4 PPL ($L=2048$).}
\label{tab:blocksize_pareto_1}
\label{tab:blocksize_pareto}
\small
\setlength{\tabcolsep}{3pt}
\begin{tabular}{rrrrr}
\toprule
$b$ & $z$ & $D_{\text{head}}$ & $D_{\text{union}}$ & PPL ($\downarrow$) \\
\midrule
$4$  & $2.0$ & $34.77\%$ & $51.57\%$ & $102.32$ \\
$8$  & $2.0$ & $37.23\%$ & $53.95\%$ & $98.39$  \\
$16$ & $1.5$ & $32.70\%$ & $47.51\%$ & $99.07$  \\
$32$ & $1.0$ & $39.00\%$ & $50.61\%$ & $95.92$  \\
$64$ & $1.0$ & $\mathbf{44.26\%}$ & $51.94\%$ & $\mathbf{93.48}$ \\
\bottomrule
\end{tabular}
\vspace{1mm}
\end{wraptable}

\paragraph{GQA Union Sharing Favors Larger Block Sizes}
On GPU hardware, Grouped-Query Attention ($G = H_Q / H_{\text{KV}}$) loads the union of active KV blocks across all query heads sharing a KV head. We find that while fine-grained blocks ($b=4$) minimize per-head density (\autoref{tab:index_quality}), coarser blocks ($b=64$) lead to higher cross-head overlapping and are thus significantly more efficient at a fixed \emph{physical HBM bandwidth budget} (\autoref{tab:blocksize_pareto_1}). Thus, for the exact same HBM memory transfer, $b=64$ allows each query head to attend to $+9.5\%$ more context and improves perplexity by $8.8$ points. This is very beneficial during inference because it enables coalesced Tensor Core execution (\autoref{appx:block_mechanics}).

\paragraph{Comparison with Quest \citep{tang2024quest}}
At block size $b=64$, ETA's Dual Bound (\autoref{eq:dual_bound_eq}) captures a substantially higher fraction of true above-threshold keys than Quest's coordinate-wise bounding box (\autoref{tab:pareto_c4_needle}; see \autoref{appx:dual_index_math} for the geometric analysis). This advantage stems from ETA's robustness to coordinate outliers. Quest bounds each block by summing per-coordinate extrema ($\sum_{i=1}^{d_h} \max(q_i k_{\max, i}, q_i k_{\min, i})$), which loosens in high dimensions whenever isolated coordinates spike. In contrast, ETA's dual bounds are much sharper, reducing block false negatives and closing most of the perplexity gap between $b=64$ and $b=4$ (\autoref{tab:pareto_c4_needle}, \autoref{tab:full_c4_pareto_appendix}).

\begin{table}[!htbp]
\centering
\caption{\textbf{Pareto Frontier on C4 Perplexity ($L=2048$), Token Recall, and Needle Retrieval ($L=4096$)} on the 1.45B model across block sizes $b \in \{4, 64\}$ vs.\ Quest \citep{tang2024quest} ($b=64$) and $\text{H}_2\text{O}$ \citep{zhang2023h2o}. Extended offset sweep in \autoref{tab:full_c4_pareto_appendix}.}
\label{tab:pareto_c4_needle}
\renewcommand{\arraystretch}{0.85}
\setlength{\tabcolsep}{5.5pt}
\resizebox{0.84\linewidth}{!}{%
\begin{tabular}{l c c c c c c}
\toprule
\textbf{Method} & \textbf{Config} & \textbf{$D_{\text{head}}$} & \textbf{$D_{\text{union}}$} & \textbf{Recall ($\uparrow$)} & \textbf{C4 PPL ($\downarrow$)} & \textbf{Needle ($\uparrow$)} \\
\midrule
\textbf{Dense Baseline} & \texttt{FA2 / SDPA} & $100\%$ & $100\%$ & $100\%$ & $17.49 \pm 0.33$ & $66.7\%$ \\
\cmidrule(lr){1-7}
\multirow{3}{*}{\textbf{ETA ($b=4$)}} & $a=+0.4$ & $36.3\%$ & $66.6\%$ & $98.3\%$ & $\mathbf{17.56 \pm 0.33}$ & $\mathbf{66.7\%}$ \\
 & $a=-0.8$ & $24.4\%$ & $50.6\%$ & $97.1\%$ & $17.82 \pm 0.33$ & $\mathbf{66.7\%}$ \\
 & $a=-2.5$ & $\mathbf{13.1\%}$ & $\mathbf{30.1\%}$ & $94.8\%$ & $19.81 \pm 0.34$ & $\mathbf{66.7\%}$ \\
\cmidrule(lr){1-7}
\multirow{3}{*}{\textbf{ETA ($b=64$)}} & $a=+0.2$ & $55.3\%$ & $79.0\%$ & $\mathbf{98.4\%}$ & $18.30 \pm 0.29$ & $\mathbf{66.7\%}$ \\
 & $a=-0.4$ & $48.3\%$ & $69.7\%$ & $97.3\%$ & $18.60 \pm 0.29$ & $\mathbf{66.7\%}$ \\
 & $a=-1.8$ & $37.6\%$ & $49.6\%$ & $91.6\%$ & $19.21 \pm 0.30$ & $\mathbf{66.7\%}$ \\
\cmidrule(lr){1-7}
\multirow{3}{*}{\textbf{Quest ($b=64$)}} & $\text{top-}38\%$ & $41.1\%$ & $66.7\%$ & $54.0\%$ & $18.17 \pm 0.34$ & $66.7\%$ \\
 & $\text{top-}25\%$ & $26.5\%$ & $47.3\%$ & $36.6\%$ & $19.26 \pm 0.35$ & $66.7\%$ \\
 & $\text{top-}10\%$ & $12.8\%$ & $24.2\%$ & $19.4\%$ & $22.37 \pm 0.39$ & $66.7\%$ \\
\cmidrule(lr){1-7}
\multirow{2}{*}{\textbf{$\text{H}_2\text{O}$ Eviction}} & $W=256$ & $48.0\%$ & $68.0\%$ & --- & $17.59 \pm 0.33$ & $22.2\%$ \\
 & $W=64$ & $17.5\%$ & $28.8\%$ & --- & $17.91 \pm 0.33$ & $0.0\%$ \\
\bottomrule
\end{tabular}%
}
\end{table}

\subsection{Long-Context RULER Multi-Needle and Multi-Hop Retrieval}
\label{subsec:needle_retrieval}
To evaluate whether sparse selection preserves recall under distractors and multi-hop reasoning, we benchmark four RULER-like tasks \citep{hsieh2024ruler}—Single-Needle ($K=1$), Multi-Key distractors ($K=4$), Multi-Value recall ($K=2$), and 2-Hop Variable Tracking—across $L \in \{2048, 4096\}$ (\autoref{tab:needle-experiments}), alongside single-needle extrapolation up to $L=8192$ (\autoref{tab:single_needle_extrap_appx}, \autoref{appx:ruler_suite}).

While $\text{H}_2\text{O}$ and Quest degrade severely under distractors and 2-hop chains, ETA dynamically lowers $\tau_t$ on retrieval heads ($\Delta\tau_{\text{ret}} = -0.48$ to $-2.19$) to expand active density by $+2.4\%\text{--}6.5\%$ (\autoref{tab:needle-experiments}). Across $N=432$ trials at $L=2048$, ETA ($b=4$) achieves $\mathbf{88.0\% \pm 2.9\%}$ RULER accuracy, \emph{surpassing dense attention} and matching it at $L=4096$. At hardware-aligned tile size $b=64$, ETA also outperforms dense attention at $L=2048$, beating Quest and $\text{H}_2\text{O}$ by $\mathbf{3.3\times\text{--}8.8\times}$.

\begin{table}[!htbp]
\centering
\renewcommand{\arraystretch}{0.78} 
\setlength{\tabcolsep}{3.5pt} 
\caption{\textbf{RULER multi-needle and 2-hop retrieval.} Full 16-arm suite with 95\% CIs in \autoref{appx:ruler_suite}.}
\label{tab:needle-experiments}
\resizebox{0.82\linewidth}{!}{%
\begin{tabular}{l c c c c c c c}
\toprule
\textbf{Method} & \textbf{\shortstack{GQA Union\\($2\text{K}$)}} & \textbf{\shortstack{Single\\($2\text{K}$)}} & \textbf{\shortstack{MultiKey\\($2\text{K}$)}} & \textbf{\shortstack{MultiVal\\($2\text{K}$)}} & \textbf{\shortstack{2-Hop VT\\($2\text{K}$)}} & \textbf{\shortstack{RULER Avg\\$L=2048$}} & \textbf{\shortstack{RULER Avg\\$L=4096$}} \\
\midrule
\textbf{Dense Baseline} & $100.0\%$ & $87.0$ & $76.9$ & \hlbase{\textbf{88.9}} & $87.0$ & $85.0 \pm 3.2$ & \hlbase{\textbf{55.6 $\pm$ 4.1}} \\
\cmidrule(lr){1-8}
\textbf{H2O ($W=256, \sim 38\%$)} & $72.3\%$ & $25.0$ & $12.0$ & \textit{0.9} & \textit{0.9} & $9.7 \pm 2.8$ & \textit{1.6 $\pm$ 1.2} \\
\textbf{Quest ($b=64$, top-$38\%$)} & $51.2\%$ & $55.6$ & $23.1$ & $12.0$ & $13.0$ & $25.9 \pm 4.1$ & $10.0 \pm 2.8$ \\
\textbf{ETA ($b=64, a=0.0$)} & $\mathbf{37.4\rightarrow39.8\%}$ & $\mathbf{87.0}$ & $\mathbf{76.9}$ & $\mathbf{78.7}$ & $\mathbf{67.6}$ & $\mathbf{77.6 \pm 3.8}$ & $\mathbf{46.5 \pm 4.2}$ \\
\textbf{ETA ($b=64, a=0.4$)} & $54.6\rightarrow57.9\%$ & $\mathbf{87.0}$ & \hlprob{\textbf{88.0}} & \hlprob{\textbf{88.9}} & $\mathbf{77.8}$ & \hlprob{\textbf{85.4 $\pm$ 3.1}} & $\mathbf{53.0 \pm 4.2}$ \\
\textbf{ETA ($b=4, a=0.4$)} & $66.3\rightarrow69.8\%$ & \hlprob{\textbf{88.0}} & $\mathbf{87.0}$ & \hlprob{\textbf{88.9}} & \hlprob{\textbf{88.0}} & \hlprob{\textbf{88.0 $\pm$ 2.9}} & \hlprob{\textbf{55.6 $\pm$ 4.1}} \\
\bottomrule
\end{tabular}%
}
\end{table}

\subsection{Comparison with Native Sparse Attention (NSA)}
\label{sec:nsa}
We benchmark ETA against Native Sparse Attention (NSA) \citep{yuan2025native} on an identical 126M backbone ($163.8$M tokens across $10{,}000$ steps on single-node H100 GPUs; see \autoref{appx:nsa_eta}). ETA optimizes faster and converges to a lower pretraining loss ($4.370\text{--}4.399$ vs.\ $4.508$, $\Delta\text{PPL} = -9.4$) at matched H100 training step throughput ($\mathbf{0.5315\text{ s/step}}$ vs.\ $\mathbf{0.5311\text{ s/step}}$; \autoref{fig:loss_convergence}). In downstream generation ($L=2048$), ETA achieves lower perplexity on FineWeb ($99.80$ vs.\ $103.55$) and WikiText-2 ($158.64$ vs.\ $161.65$) at $53\%$ sparsity (\autoref{tab:ppl_results}). Furthermore, ETA's fused single-pass decode kernel eliminates multi-kernel launch overhead, achieving up to $2.72\times$ lower step latency and $2.72\times$ higher throughput with $1.63\times$ speedup at $32\text{K}$ context (\autoref{tab:decoding_efficiency}).
\subsection{Elimination of Attention Sinks and Autonomous Layer Specialization}
\label{subsec:mechanistic_interpretation}
In standard transformers, excess probability mass is deposited onto initial anchor tokens, creating persistent \emph{attention sinks} \citep{xiao2024efficient} that post-hoc eviction methods (e.g., StreamingLLM, $\text{H}_2\text{O}$) must explicitly pin to avoid collapse. Under ETA, \textbf{localized attention sinks disappear entirely} (\autoref{fig:mechanisms}, Left): initial-token attention mass drops from $>50\%$ in dense models to $<2\%$, as unallocated probability mass routes into the diffuse background floor (\autoref{appx:sink_layer_analysis}). Simultaneously, ETA discovers distinct layer-wise sparsity allocations autonomously (\autoref{fig:mechanisms}, Right): Layer~0 acts as a broad aggregator ($\approx 96\%$ density), intermediate layers compress context aggressively (reaching $<8\%$ density at Layer~7), and deeper layers selectively restore capacity for final token synthesis.

\begin{figure}[!htbp]
    \centering
    \begin{subfigure}{0.45\linewidth}
        \centering
        \includegraphics[width=\linewidth]{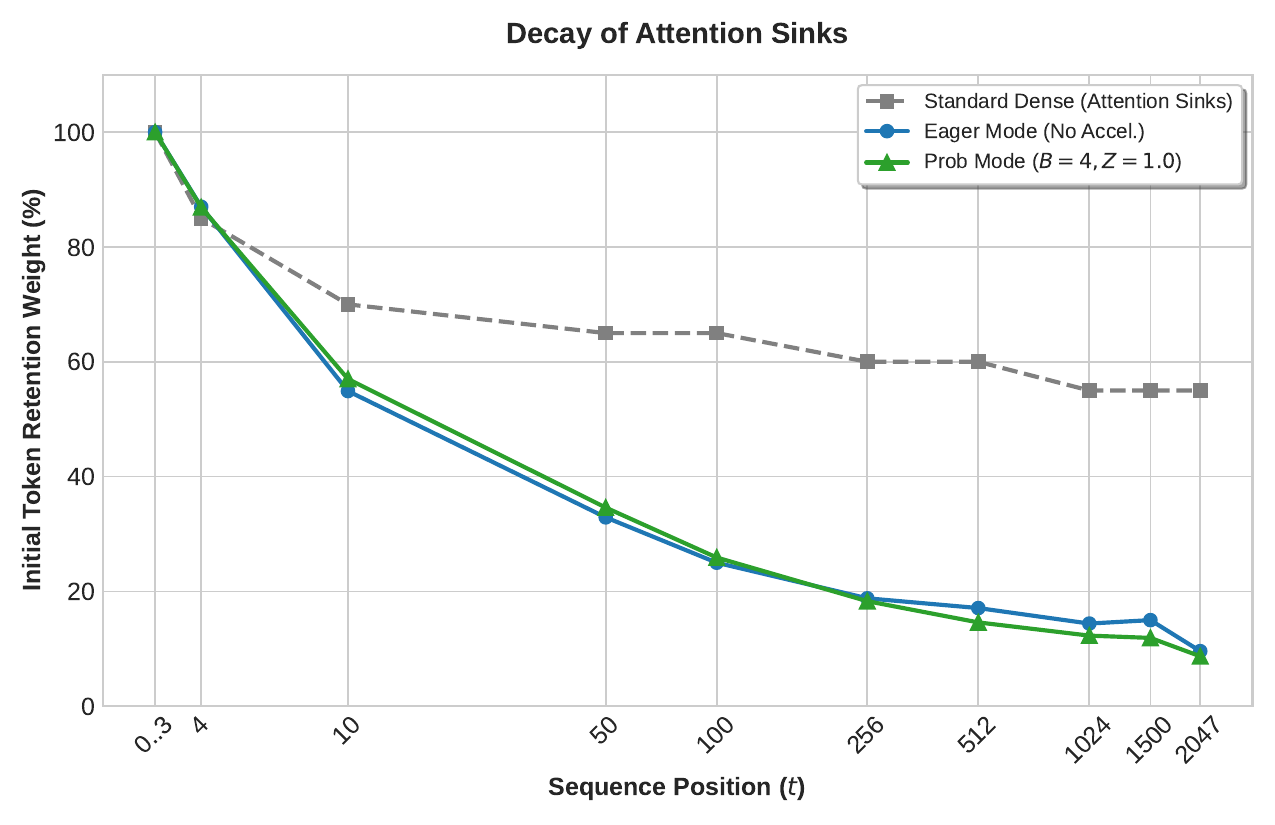}
        \label{fig:sinks}
    \end{subfigure}
    \hfill
    \begin{subfigure}{0.53\linewidth}
        \centering
        \includegraphics[width=\linewidth]{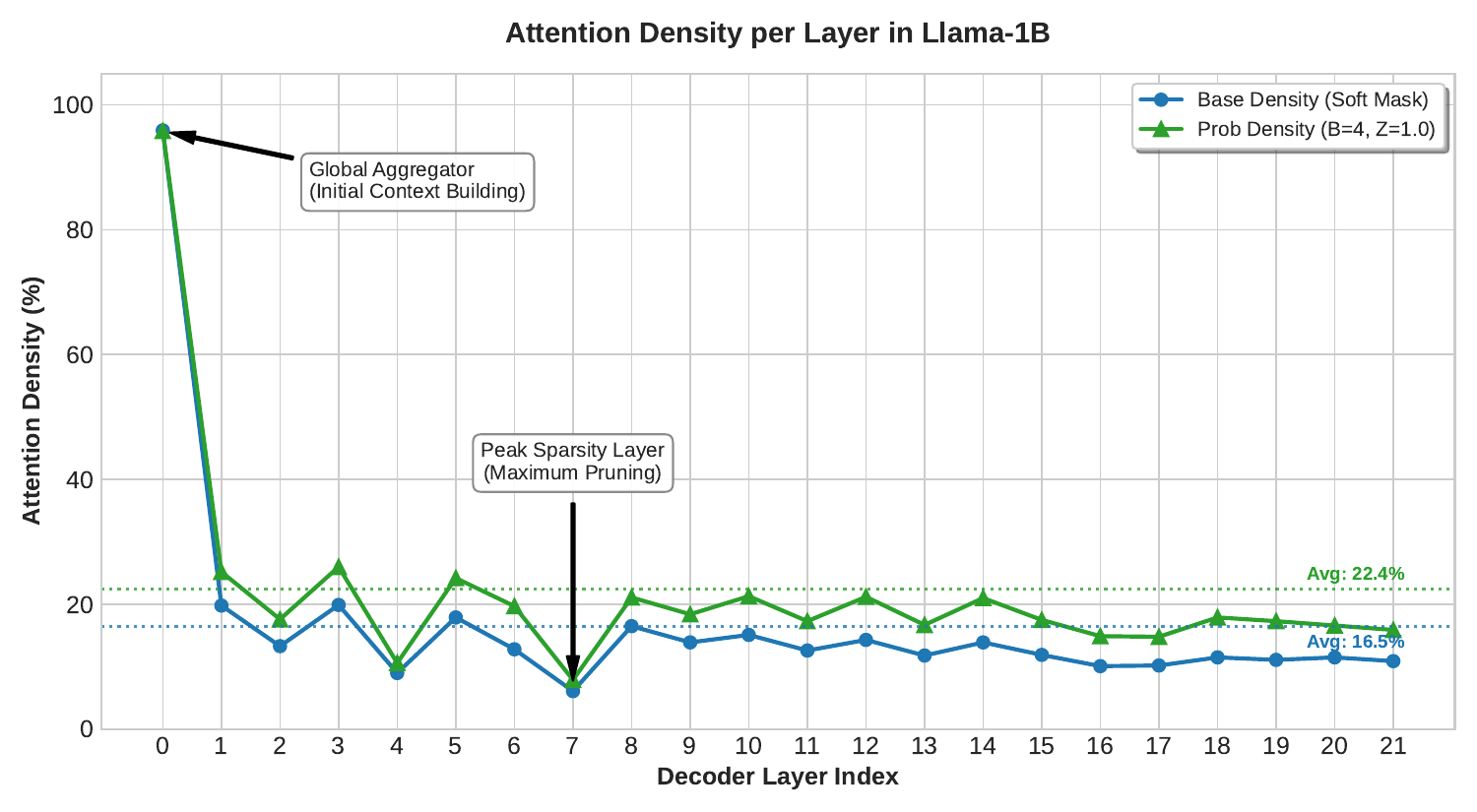}
        \label{fig:density}
    \end{subfigure}
    \caption{\textbf{Mechanistic Dynamics of Attention in ETA.} \textbf{(Left)} Localized attention sinks disappear under ETA's uniform background floor. \textbf{(Right)} Autonomous layer-wise active density across Transformer depth.}
    \label{fig:mechanisms}
\end{figure}

\subsection{Runtime Contextuality and Domain-Specific Offline Calibration}
\label{subsec:contextuality}
\label{subsec:calibration}

\paragraph{Why Contextuality Matters and What It Encodes.}
Pretraining an input-independent baseline with static per-head thresholds $\tau_h^\ell$ (without $W_\tau^\ell q_t + b_\tau^\ell$) converges to $2.13\times$ higher active density than ETA. Regressing realized thresholds $\tau_{t,h}^\ell$ against eight query and score-distribution statistics on held-out C4 explains at most $32.7\%$ of variance (\autoref{tab:tau_r2}), with $\approx 70\%$ occurring \emph{within} individual documents (\autoref{appx:contextuality_variance}). These results confirm that contextuality is a necessary part of our architecture.

\paragraph{Distilling Static Thresholds for Stationary Deployments.}
For stationary deployments, the predictor can be distilled post-hoc into frozen constants $c^{\ell h}$ via offline calibration (\autoref{alg:calibration} in \autoref{appx:calibration_derivation}), which accumulates a $1/t$-weighted score histogram over $16$ sequences and solves via binary search for each head's training density. On held-out C4 (\autoref{tab:frozen_tau}, \autoref{subsection:frozen_benchmarks}), calibrated static thresholds match dynamic perplexity in-distribution and curb RoPE density inflation at $2\times L_{\text{train}}$, cutting compute at no perplexity cost.

\subsection{Decode Speed and End-to-End Efficiency}
\label{sec:speed}

We evaluate ETA's hardware speedups by benchmarking single-step decode latency ($q_{\text{len}} = 1$) against FlashAttention-2 \citep{dao2023flashattention} on an NVIDIA H100 GPU (batch sizes $B \in \{32, 64, 128\}$, context lengths up to $512\text{K}$; \autoref{tab:standalone_attention}, with block/chunk size ablations in \autoref{tab:standalone_arms}). 

Although short sequences ($L=2048$) face a $\sim 50\%$ GQA union floor due to fixed boundary blocks ($\mathtt{lwb}=2$ plus the tail block), this overhead becomes negligible ($<0.2\%$) at $L=64\text{K}\text{--}512\text{K}$, enabling effective sparsity levels of $16\%\text{--}35.5\%$ union density (\autoref{subsec:calibration}). At those configurations, ETA achieves a $1.31\times\text{--}1.48\times$ kernel speedup while preserving full dense retrieval accuracy on needle-in-a-haystack tasks. At higher sparsity ($16.0\%$ union), kernel speedup reaches $2.36\times$. Across a full 22-layer model, these kernel gains deliver up to a $2.15\times$ end-to-end generation speedup with negligible indexing overhead during prefill (\autoref{appx:systems_accounting}).

\begin{table}[!htbp]
\centering
\caption{\textbf{Single decode-step attention latency, ETA versus FlashAttention-2.} NVIDIA H100 SXM $80$\,GB, \texttt{float16}, $H_Q=32$, $H_{KV}=4$ ($G=8$), $d_h=64$, $b=64$. Speedup ($\times$) is $t_{\text{FA2}} / t_{\text{ETA}}$.}
\label{tab:standalone_attention}
\renewcommand{\arraystretch}{0.56}
\setlength{\tabcolsep}{2.5pt}
\resizebox{0.82\linewidth}{!}{%
\begin{tabular}{ccc ccc ccc ccc}
\toprule
& & & \multicolumn{3}{c}{\textbf{$51.2\%$ union ($38\%$ head)}} & \multicolumn{3}{c}{\textbf{$35.5\%$ union ($25\%$ head)}} & \multicolumn{3}{c}{\textbf{$16.0\%$ union ($10\%$ head)}} \\
\cmidrule(lr){4-6} \cmidrule(lr){7-9} \cmidrule(lr){10-12}
$B$ & $L$ & KV & FA2 & ETA & $\times$ & FA2 & ETA & $\times$ & FA2 & ETA & $\times$ \\
\midrule
$32$ & $64$K & $2$\,GB & $0.710$ & $0.983$ & $0.72$ & $0.710$ & $0.767$ & $0.93$ & $0.711$ & $0.526$ & $\mathbf{1.35}$ \\
$32$ & $128$K & $4$\,GB & $1.398$ & $1.830$ & $0.76$ & $1.395$ & $1.426$ & $0.98$ & $1.396$ & $0.936$ & $\mathbf{1.49}$ \\
$32$ & $256$K & $8$\,GB & $2.772$ & $3.510$ & $0.79$ & $2.775$ & $2.732$ & $1.02$ & $2.774$ & $1.747$ & $\mathbf{1.59}$ \\
$32$ & $512$K & $16$\,GB & $5.529$ & $6.787$ & $0.81$ & $5.525$ & $5.275$ & $1.05$ & $5.533$ & $3.364$ & $\mathbf{1.64}$ \\
\midrule
$64$ & $64$K & $4$\,GB & $1.757$ & $1.833$ & $0.96$ & $1.885$ & $1.435$ & $\mathbf{1.31}$ & $1.897$ & $0.937$ & $\mathbf{2.02}$ \\
$64$ & $128$K & $8$\,GB & $3.872$ & $3.504$ & $1.11$ & $3.653$ & $2.746$ & $\mathbf{1.33}$ & $3.519$ & $1.750$ & $\mathbf{2.01}$ \\
$64$ & $256$K & $16$\,GB & $7.275$ & $6.813$ & $1.07$ & $7.790$ & $5.280$ & $\mathbf{1.48}$ & $7.722$ & $3.373$ & $\mathbf{2.29}$ \\
$64$ & $512$K & $32$\,GB & $15.085$ & $13.432$ & $\mathbf{1.12}$ & $15.009$ & $10.430$ & $\mathbf{1.44}$ & $14.790$ & $6.520$ & $\mathbf{2.27}$ \\
\midrule
$128$ & $64$K & $8$\,GB & $3.807$ & $3.516$ & $1.08$ & $3.518$ & $2.737$ & $\mathbf{1.29}$ & $3.939$ & $1.746$ & $\mathbf{2.26}$ \\
$128$ & $128$K & $16$\,GB & $7.624$ & $6.841$ & $1.11$ & $7.430$ & $5.368$ & $\mathbf{1.38}$ & $7.923$ & $3.351$ & $\mathbf{2.36}$ \\
$128$ & $256$K & $32$\,GB & $14.836$ & $13.481$ & $1.10$ & $15.302$ & $10.461$ & $\mathbf{1.46}$ & $15.039$ & $6.622$ & $\mathbf{2.27}$ \\
\bottomrule
\end{tabular}%
}
\end{table}

\section{Limitations and Conclusion}
\label{sec:conclusion}

We introduced Elastic Threshold Attention (ETA), an end-to-end trainable sparse attention architecture that contracts unselected logits toward $0$ via a SiLU gate rather than deleting them to $-\infty$. This creates a uniform attention floor that eliminates localized attention sinks and makes the model naturally invariant to inference-time distribution shifts. Paired with our fused Triton decode kernel, ETA matches or exceeds dense quality across language modeling, reasoning, and RULER retrieval benchmarks while delivering up to $2.36\times$ kernel decode speedups over FlashAttention-2.

Natural extensions of this work span both systems and modeling. On the systems side, integrating Hopper asynchronous warp-specialization (\texttt{WGMMA}/\texttt{TMA} \citep{dao2023flashattention}) and \texttt{FP8} metadata quantization can further reduce screening overhead at moderate sparsity levels ($\sim 50\%$ union density), while incorporating continual long-context pretraining ($32\text{K}\text{--}128\text{K}$) and exploring alternative zero-contracting gates offer promising directions for large-scale inference serving.

\clearpage

\subsection*{AI use statement}
In this work, we used generative AI tools for editing and tightening prose and as a coding aid when implementing and debugging the Triton kernels and the evaluation harness. We also used AI tools for assistance in deriving the attention gradients used in our backward kernel and in running the ANOVA decomposition test in \autoref{appx:contextuality_variance}.  Additionally, we used generative AI tools for creating and refining Tikz figures, as well as for suggesting experimental parameters.

We have not used generative AI tools for research ideation, retrieval and discovery, or generating synthetic datasets. We did not use AI tools in writing the core logic of the ETA model. We have reviewed all AI-assisted work. We manually inspected all lines of code written by generative AI and ran numerous experiments for correctness verification. All mathematical and conceptual claims, including potentially erroneous ones, are our own. We take responsibility for the final content of this work,
including text, claims or artifacts produced with the aid of generative AI.

\bibliography{main}

@inproceedings{Vaswani+2017,
 author = {Vaswani, Ashish and Shazeer, Noam and Parmar, Niki and Uszkoreit, Jakob and Jones, Llion and Gomez, Aidan N and Kaiser, \L ukasz and Polosukhin, Illia},
 booktitle = {Advances in Neural Information Processing Systems},
 pages = {},
 publisher = {Curran Associates, Inc.},
 title = {Attention is All you Need},
 url = {https://proceedings.neurips.cc/paper_files/paper/2017/file/3f5ee243547dee91fbd053c1c4a845aa-Paper.pdf},
 volume = {30},
 year = {2017}
}

@article{llama3modelcard,
title={Llama 3 Model Card},
author={AI@Meta},
year={2024},
url = {https://github.com/meta-llama/llama3/blob/main/MODEL_CARD.md}
}

@article{yerram2024hire,
  title={HiRE: High Recall Approximate Top-$ k $ Estimation for Efficient LLM Inference},
  author={Yerram, Varun and You, Chong and Bhojanapalli, Srinadh and Kumar, Sanjiv and Jain, Prateek and Netrapalli, Praneeth and others},
  journal={arXiv preprint arXiv:2402.09360},
  year={2024}
}

@inproceedings{tillet2019triton,
  title={Triton: an intermediate language and compiler for tiled neural network computations},
  author={Tillet, Philippe and Kung, Hsiang-Tsung and Cox, David},
  booktitle={Proceedings of the 3rd ACM SIGPLAN International Workshop on Machine Learning and Programming Languages},
  pages={10--19},
  year={2019}
}

@inproceedings{yuan2025native,
  title={Native sparse attention: Hardware-aligned and natively trainable sparse attention},
  author={Yuan, Jingyang and Gao, Huazuo and Dai, Damai and Luo, Junyu and Zhao, Liang and Zhang, Zhengyan and Xie, Zhenda and Wei, Yuxing and Wang, Lean and Xiao, Zhiping and others},
  booktitle={Proceedings of the 63rd Annual Meeting of the Association for Computational Linguistics (Volume 1: Long Papers)},
  pages={23078--23097},
  year={2025}
}

@article{child2019generating,
  title={Generating long sequences with sparse transformers},
  author={Child, Rewon and Gray, Scott and Radford, Alec and Sutskever, Ilya},
  journal={arXiv preprint arXiv:1904.10509},
  year={2019}
}

@article{zaheer2020big,
  title={Big bird: Transformers for longer sequences},
  author={Zaheer, Manzil and Guruganesh, Guru and Dubey, Kumar Avinava and Ainslie, Joshua and Alberti, Chris and Ontanon, Santiago and Pham, Philip and Ravula, Anirudh and Wang, Qifan and Yang, Li and others},
  journal={Advances in neural information processing systems},
  volume={33},
  pages={17283--17297},
  year={2020}
}

@article{choromanski2020rethinking,
  title={Rethinking attention with performers},
  author={Choromanski, Krzysztof and Likhosherstov, Valerii and Dohan, David and Song, Xingyou and Gane, Andreea and Sarlos, Tamas and Hawkins, Peter and Davis, Jared and Mohiuddin, Afroz and Kaiser, Lukasz and others},
  journal={arXiv preprint arXiv:2009.14794},
  year={2020}
}

@inproceedings{shen2021efficient,
  title={Efficient attention: Attention with linear complexities},
  author={Shen, Zhuoran and Zhang, Mingyuan and Zhao, Haiyu and Yi, Shuai and Li, Hongsheng},
  booktitle={Proceedings of the IEEE/CVF winter conference on applications of computer vision},
  pages={3531--3539},
  year={2021}
}

@article{dao2022flashattention,
  title={Flashattention: Fast and memory-efficient exact attention with io-awareness},
  author={Dao, Tri and Fu, Dan and Ermon, Stefano and Rudra, Atri and R{\'e}, Christopher},
  journal={Advances in neural information processing systems},
  volume={35},
  pages={16344--16359},
  year={2022}
}

@article{dao2023flashattention,
  title={Flashattention-2: Faster attention with better parallelism and work partitioning},
  author={Dao, Tri},
  journal={arXiv preprint arXiv:2307.08691},
  year={2023}
}

@article{zhang2023h2o,
  title={H2o: Heavy-hitter oracle for efficient generative inference of large language models},
  author={Zhang, Zhenyu and Sheng, Ying and Zhou, Tianyi and Chen, Tianlong and Zheng, Lianmin and Cai, Ruisi and Song, Zhao and Tian, Yuandong and R{\'e}, Christopher and Barrett, Clark and others},
  journal={Advances in Neural Information Processing Systems},
  volume={36},
  pages={34661--34710},
  year={2023}
}

@inproceedings{chen2025magicpig,
  title={Magicpig: Lsh sampling for efficient llm generation},
  author={Chen, Zhuoming and Sadhukhan, Ranajoy and Ye, Zihao and Zhou, Yang and Zhang, Jianyu and Nolte, Niklas and Tian, Yuandong and Douze, Matthijs and Bottou, Leon and Jia, Zhihao and others},
  booktitle={International Conference on Learning Representations},
  volume={2025},
  pages={44169--44190},
  year={2025}
}

@article{tang2024quest,
  title={Quest: Query-aware sparsity for efficient long-context llm inference},
  author={Tang, Jiaming and Zhao, Yilong and Zhu, Kan and Xiao, Guangxuan and Kasikci, Baris and Han, Song},
  journal={arXiv preprint arXiv:2406.10774},
  year={2024}
}

@inproceedings{xiao2024efficient,
  title={Efficient streaming language models with attention sinks},
  author={Xiao, Guangxuan and Tian, Yuandong and Chen, Beidi and Han, Song and Lewis, Mike},
  booktitle={International Conference on Learning Representations},
  volume={2024},
  pages={21875--21895},
  year={2024}
}

@article{roy2021efficient,
  title={Efficient content-based sparse attention with routing transformers},
  author={Roy, Aurko and Saffar, Mohammad and Vaswani, Ashish and Grangier, David},
  journal={Transactions of the Association for Computational Linguistics},
  volume={9},
  pages={53--68},
  year={2021}
}

@inproceedings{hu2022adaptive,
  title={Adaptive threshold selective self-attention for Chinese NER},
  author={Hu, Biao and Huang, Zhen and Hu, Minghao and Zhang, Ziwen and Dou, Yong},
  booktitle={Proceedings of the 29th International Conference on Computational Linguistics},
  pages={1823--1833},
  year={2022}
}

@misc{merity2016pointer,
      title={Pointer Sentinel Mixture Models},
      author={Stephen Merity and Caiming Xiong and James Bradbury and Richard Socher},
      year={2016},
      eprint={1609.07843},
      archivePrefix={arXiv},
      primaryClass={cs.CL}
}

@inproceedings{zellers2019hellaswag,
  title={Hellaswag: Can a machine really finish your sentence?},
  author={Zellers, Rowan and Holtzman, Ari and Bisk, Yonatan and Farhadi, Ali and Choi, Yejin},
  booktitle={Proceedings of the 57th annual meeting of the association for computational linguistics},
  pages={4791--4800},
  year={2019}
}

@article{penedo2024fineweb,
  title={The fineweb datasets: Decanting the web for the finest text data at scale},
  author={Penedo, Guilherme and Kydl{\'\i}{\v{c}}ek, Hynek and Lozhkov, Anton and Mitchell, Margaret and Raffel, Colin and Von Werra, Leandro and Wolf, Thomas and others},
  journal={Advances in Neural Information Processing Systems},
  volume={37},
  pages={30811--30849},
  year={2024}
}

@article{clark2018think,
  title={Think you have solved question answering? try arc, the ai2 reasoning challenge},
  author={Clark, Peter and Cowhey, Isaac and Etzioni, Oren and Khot, Tushar and Sabharwal, Ashish and Schoenick, Carissa and Tafjord, Oyvind},
  journal={arXiv preprint arXiv:1803.05457},
  year={2018}
}

@article{JMLR:v21:20-074,
  author  = {Colin Raffel and Noam Shazeer and Adam Roberts and Katherine Lee and Sharan Narang and Michael Matena and Yanqi Zhou and Wei Li and Peter J. Liu},
  title   = {Exploring the Limits of Transfer Learning with a Unified Text-to-Text Transformer},
  journal = {Journal of Machine Learning Research},
  year    = {2020},
  volume  = {21},
  number  = {140},
  pages   = {1--67},
  url     = {http://jmlr.org/papers/v21/20-074.html}
}

@misc{bloc972023ntk,
  author = {bloc97},
  title = {NTK-Aware Scaled RoPE allows LLaMA models to have extended (8k+) context size...},
  year = {2023},
  url = {https://www.reddit.com/r/LocalLLaMA/comments/14lz7j5/ntkaware_scaled_rope_allows_llama_models_to_have/}
}

@inproceedings{haris2025knn,
  title={knn attention demystified: A theoretical exploration for scalable transformers},
  author={Haris, Themistoklis},
  booktitle={International Conference on Learning Representations},
  volume={2025},
  pages={47576--47603},
  year={2025}
}

@article{desai2024hashattention,
  title={Hashattention: Semantic sparsity for faster inference},
  author={Desai, Aditya and Yang, Shuo and Cuadron, Alejandro and Zaharia, Matei and Gonzalez, Joseph E and Stoica, Ion},
  journal={arXiv preprint arXiv:2412.14468},
  year={2024}
}

@article{jiang2024minference,
  title={Minference 1.0: Accelerating pre-filling for long-context llms via dynamic sparse attention},
  author={Jiang, Huiqiang and Li, Yucheng and Zhang, Chengruidong and Wu, Qianhui and Luo, Xufang and Ahn, Surin and Han, Zhenhua and Abdi, Amir H and Li, Dongsheng and Lin, Chin-Yew and others},
  journal={Advances in Neural Information Processing Systems},
  volume={37},
  pages={52481--52515},
  year={2024}
}

@article{shi2025trainable,
  title={Trainable dynamic mask sparse attention},
  author={Shi, Jingze and Wu, Yifan and Peng, Yiran and Wu, Bingheng and Wang, Liangdong and Liu, Guang and Luo, Yuyu},
  journal={arXiv preprint arXiv:2508.02124},
  year={2025}
}

@inproceedings{gao2024seerattention,
  title={SeerAttention: Learning Intrinsic Sparse Attention in Your {LLM}s},
  author={Gao, Yizhao and Zeng, Zhichen and Du, Dayou and Cao, Shijie and Zhou, Peiyuan and Qi, Jiaxing and Lai, Junjie and So, Hayden Kwok-Hay and Cao, Ting and Yang, Fan and Yang, Mao},
  booktitle={Advances in Neural Information Processing Systems},
  year={2025}
}

@inproceedings{zhang2025spargeattention,
  title={SpargeAttention: Accurate and Training-free Sparse Attention Accelerating Any Model Inference},
  author={Zhang, Jintao and Xiang, Chendong and Huang, Haofeng and Wei, Jia and Xi, Haocheng and Zhu, Jun and Chen, Jianfei},
  booktitle={Proceedings of the 42nd International Conference on Machine Learning},
  series={Proceedings of Machine Learning Research},
  year={2025}
}

@article{hsieh2024ruler,
  title={RULER: What's the real context size of your long-context language models?},
  author={Hsieh, Cheng-Ping and Sun, Simeng and Kriman, Samuel and Acharya, Shantanu and Rekesh, Dima and Jia, Fei and Zhang, Yang and Ginsburg, Boris},
  journal={arXiv preprint arXiv:2404.06654},
  year={2024}
}

\appendix

\section{Transformer Preliminaries}
\label{app:transformer_math}

We focus on the Decoder-only Transformer architecture, viewed as a mapping from sequences $(x_1,...,x_n) \in \mathbb{R}^{n\times d}$ of token embeddings to sequences of latent embeddings $(h_1^{L},...,h_n^L) \in \mathbb{R}^{n\times d}$, with $L$ being the total number of layers and $d$ being the embedding dimension. Let $h_t^0 := x_t \in \mathbb{R}^d$. 

The processing flow of a single token $t$ through a specific layer $\ell$ maps $h_t^{\ell-1}$ to $h_t^{\ell}$. Let $H$ be the number of attention heads and $d_h = \frac{d}{H}$ be the head dimension. We compute the query $q_t$, key $k_t$, and value $v_t$ embeddings via linear projections $W_Q, W_K, W_V \in \mathbb{R}^{d_h \times d}$ on the normalized input $\hat{h} = \text{RMSNorm}(h_t^{\ell-1})$. Rotary positional embeddings (RoPE) are applied as follows:
$$q_t = \mathcal{R}_t (W_Q \cdot \hat{h}), \quad k_t = \mathcal{R}_t (W_K \cdot \hat{h}), \quad v_t = W_V \cdot \hat{h}$$
where $\mathcal{R}_t \in \mathbb{R}^{d_h \times d_h}$ is a fixed unitary rotation matrix depending on position $t$.

The model computes the attention output by having query $q_t$ attend to keys $\{k_s\}_{s \leq t}$ and values $\{v_s\}_{s \leq t}$. The attention weights and output are computed via scaled dot-product:
$$A_{ts} = \text{softmax}\left(\frac{\langle q_t, k_s \rangle}{\sqrt{d_h}} + M_{ts}\right), \qquad o_t = \sum_{s \leq t} A_{ts} v_s$$
where $M_{ts}$ is a causal mask ($0$ if $s \leq t$, $-\infty$ otherwise). 

The layer output $h_t^{\ell}$ is formed via residual connections and a Multi-Layer Perceptron (MLP), with output projection $W_O \in \mathbb{R}^{d \times d_h}$:
$$z_t = h_t^{\ell-1} + W_O \cdot o_t$$
$$h_t^{\ell} = z_t + \text{MLP}(\text{RMSNorm}(z_t))$$

During inference, Transformer-based LLMs operate in two distinct stages:
\begin{enumerate}
    \item \textbf{Pre-fill}: A prompt of size $L$ is processed in a single forward pass. This is \textit{compute-heavy}, requiring $\Theta(L^2 d)$ FLOPS. 
    \item \textbf{Decoding}: Subsequent tokens are autoregressively predicted. The keys and values of previous tokens are stored in a \textbf{KV Cache} to avoid recomputation, making decoding a strictly \textit{memory-bound} operation.
\end{enumerate}

\section{Training Details and Hyperparameters}
\label{appx:training}
We train our 1.45B-parameter ETA model on a slice of the FineWeb dataset (\texttt{CC-MAIN-2018-26}) \citep{penedo2024fineweb} at sequence length $L = 2048$ for $10{,}000$ steps across $8$ NVIDIA H100 GPUs (effective global batch size $2048$ sequences, ${\approx}42$B tokens total, exceeding Chinchilla-optimal compute by $1.4\times$). The backbone comprises $22$ Transformer decoder layers with $d_{\text{model}} = 2048$, intermediate SwiGLU dimension $8192$, $H_Q = 32$ query heads ($d_k = 64$), and $H_{KV} = 4$ GQA key-value heads. Optimization uses AdamW ($\beta_1 = 0.9, \beta_2 = 0.95, \epsilon = 10^{-8}$, weight decay $0.1$, gradient clipping $1.0$) with cosine learning rate decay from $4 \times 10^{-4}$ to $4 \times 10^{-5}$ after $300$ warmup steps.

Per-layer linear threshold predictors ($W_\tau^\ell \in \mathbb{R}^{2048 \times 32}$, zero-initialized weights with start-dense bias $b_\tau^\ell = -8.0$) act on concatenated post-RoPE query heads. The inverse temperature $\beta$ is linearly annealed from $1.0$ to $5.0$ over steps $0\text{--}7{,}000$ and held constant thereafter. To prevent premature pruning before stable representations form, sparsity regularization is zero for the first $750$ steps, ramps linearly over $1{,}690$ steps to a peak of $\lambda_{\text{sparse}} = 5 \times 10^{-4}$, and is then reduced to $2.5 \times 10^{-5}$ ($5\%$ of peak) for the remainder of training.

\subsection{Optimizer Dynamics and Escape from Start-Dense Initialization}
\label{appx:training_escape}
With start-dense initialization ($b_\tau^\ell = -8.0, W_\tau^\ell = 0$) and initial attention logits centered near zero, terminal sharpness ($\beta = 5$) would evaluate to $\sigma(40)$, stalling gradient descent ($\sigma'(40) \approx 4.3 \times 10^{-18}$). Two mechanisms can explain the observed threshold migration without gradient starvation: (1)~\emph{temperature annealing} starts at $\beta = 1.0$, keeping the initial sigmoid derivative fourteen orders of magnitude higher ($\sigma'(8) = 3.35 \times 10^{-4}$) until thresholds enter the score bulk; (2)~\emph{Adam's scale normalization} ($\hat{m} / (\sqrt{\hat{v}} + \epsilon) \approx \pm 1$) converts the monotone directional pressure of $\mathcal{L}_{\text{sparsity}}$ into steady step-size drift independent of small gradient magnitude.

\section{Additive vs Multiplicative Masking}
\label{appx:additive_masking}
In this section we examine additive masking as a substitute for multiplicative masking in ETA. We first compare the pretraining trajectories of ETA under additive (\textit{Strict Mode}, $\tilde{S}_{ts}^h = S_{ts}^h + 100(m_{ts}^h - 1)$) and multiplicative (\textit{Normal Mode}, $\tilde{S}_{ts}^h = S_{ts}^h \cdot m_{ts}^h$) formulations under identical regularization schedules ($\lambda=0.05$, ramping over steps $750$--$2050$) on 126M-parameter models.

As illustrated in \autoref{fig:strict_vs_normal_dynamics}(a), both formulations exhibit nearly identical cross-entropy loss trajectories on FineWeb, with multiplicative masking gaining a slight edge. However, \autoref{fig:strict_vs_normal_dynamics}(b) reveals distinct sparsity steady states. Under multiplicative masking, sub-threshold logits scale towards $0$, leaving a uniform baseline softmax floor that allows the threshold predictor to push sparsity aggressively to $83.14\%$ ($16.86\%$ active density) without language modeling loss penalties. Conversely, Strict Mode imposes hard exponential suppression ($e^{S_{ts}^h-100}\approx 0$), eliminating masked key-value states entirely from the attention context. 

\begin{figure}[h]
\centering
\includegraphics[width=\linewidth]{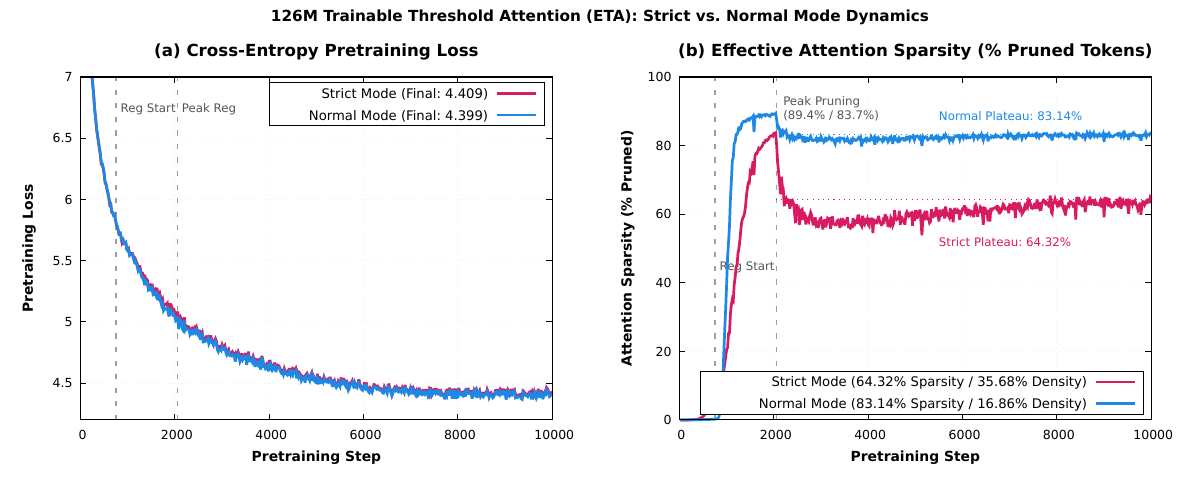}
\caption{\textbf{Pretraining dynamics and sparsity evolution across ETA masking formulations (126M parameters, 10k steps).} }
\label{fig:strict_vs_normal_dynamics}
\end{figure}

\paragraph{Inference Performance and Masking Mode Comparison.}
As reported in \autoref{tab:strict_vs_normal_results}, additive masking proves inferior to multiplicative masking across all downstream evaluation benchmarks when deployed with block-sparse inference kernels. Specifically, additive masking incurs higher perplexity across FineWeb ($116.38$ vs.\ $99.80$), C4 ($117.93$ vs.\ $100.73$), and WikiText-2 ($177.16$ vs.\ $158.64$) while retaining significantly lower compute sparsity ($\sim 35\text{--}38\%$ vs.\ $65.91\%$ pruned).

\begin{table}[h]
\centering
\small
\caption{\textbf{Downstream perplexity and sparsity comparison between Normal and Strict ETA (126M parameters, 10k pretraining steps).} Evaluated using 1\% prefill and 99\% autoregressive decode under the block-sparse inference kernel ($b=4$, threshold adjustment $a=0.4$, $z=2.0$).}
\label{tab:strict_vs_normal_results}
\begin{tabular}{lcc}
\toprule
\textbf{Metric / Benchmark} & \textbf{Normal ETA} & \textbf{Strict ETA} \\
\midrule
Pretraining Loss (Step 10k) & \textbf{4.3986} & 4.4085 \\
\midrule
\multicolumn{3}{l}{\textit{FineWeb (200 batches, 1.62M tokens)}} \\
Perplexity (PPL) & \textbf{99.80 $\pm$ 22.59} & 116.38 $\pm$ 24.38 \\
Compute Sparsity (FLOPs Pruned) & \textbf{65.91\%} (34.09\% density) & 37.26\% (62.74\% density) \\
GQA Union Sparsity (I/O Pruned) & \textbf{49.23\%} (50.77\% density) & 23.59\% (76.41\% density) \\
\midrule
\multicolumn{3}{l}{\textit{C4 (200 batches, 1.62M tokens)}} \\
Perplexity (PPL) & \textbf{100.73 $\pm$ 17.56} & 117.93 $\pm$ 21.20 \\
Compute Sparsity (FLOPs Pruned) & \textbf{65.91\%} (34.09\% density) & 38.18\% (61.82\% density) \\
GQA Union Sparsity (I/O Pruned) & \textbf{49.23\%} (50.77\% density) & 24.54\% (75.46\% density) \\
\midrule
\multicolumn{3}{l}{\textit{Wikitext-2 (Full test set, 2.94M tokens)}} \\
Perplexity (PPL) & \textbf{158.64 $\pm$ 22.15} & 177.16 $\pm$ 23.98 \\
Compute Sparsity (FLOPs Pruned) & \textbf{65.91\%} (34.09\% density) & 34.92\% (65.08\% density) \\
GQA Union Sparsity (I/O Pruned) & \textbf{49.23\%} (50.77\% density) & 21.23\% (78.77\% density) \\
\bottomrule
\end{tabular}
\end{table}

\begin{table}[h]
\centering
\caption{\textbf{Mechanistic and downstream comparison of masking formulations (126M parameters, FineWeb).} Tail statistics are measured over the sub-threshold key set ($n=24$ sequences, $L=1024$). Downstream metrics reflect block-sparse decode ($b=4, z=2.0, a=0.4$). Multiplicative gating flattens sub-threshold keys into a uniform background, achieving superior perplexity and density over additive deletion.}
\label{tab:masking_formulation_comparison}
\resizebox{\linewidth}{!}{%
\begin{tabular}{lccc}
\toprule
\textbf{Sub-threshold tail metric} & \textbf{Pre-gate (raw)} & \textbf{Normal mode (multiplicative)} & \textbf{Strict mode (additive)} \\
\midrule
Sub-threshold probability mass & $42.4\%$ & $\mathbf{43.8\%}$ (preserved) & $\mathbf{1.5 \times 10^{-25}}$ (annihilated) \\
Normalized tail entropy ($H / \log n$) & $0.790$ & $\mathbf{0.916}$ (near-uniform) & $\mathbf{0.081}$ (collapsed) \\
Tail max-to-mean ratio & $15.6$ & $\mathbf{2.03}$ (flattened) & $\mathbf{249.7}$ (extreme peakiness) \\
\midrule
\multicolumn{4}{l}{\textit{Block-sparse kernel decode, FineWeb}} \\
Perplexity ($\downarrow$) & --- & $\mathbf{99.80}$ & $116.38$ \quad ($+16.6\%$) \\
Compute density ($\downarrow$) & --- & $\mathbf{34.09\%}$ & $62.74\%$ \quad ($1.84\times$) \\
\bottomrule
\end{tabular}%
}
\end{table}

\section{Analysis of Runtime Contextuality and Offline Threshold Calibration}
\label{appx:calibration}

This appendix provides the complete empirical and mathematical analysis supporting \autoref{subsec:contextuality}: (1) testing whether query conditioning is necessary during pretraining and probing what the learned thresholds encode (\autoref{appx:contextuality_variance}), (2) deriving the $1/t$-weighted score histogram formulation for exact density-matched calibration (\autoref{alg:calibration}), and (3) evaluating distilled static thresholds in-distribution and under length extrapolation (\autoref{subsection:frozen_benchmarks}).

\subsection{Empirical Analysis of Runtime Contextuality}
\label{appx:contextuality_variance}

\paragraph{Necessity of Query Conditioning During Pretraining.}
To test whether input-dependent thresholds ($\tau_{t,h}^\ell = w_{\tau,h}^{\ell \top} q_t + b_{\tau,h}^\ell$) are required to learn sparse attention representations from scratch, we trained an ablated baseline where each head's threshold is parameterized as a single learnable, input-independent scalar $\tau_h^\ell$. Under identical sparsity regularization ($\lambda = 0.05$), the input-independent baseline fails to sparsify effectively, converging to $2.13\times$ higher active density than contextual ETA. Because a static threshold applies uniformly across all sequence positions, any threshold high enough to prune routine tokens inevitably starves complex synthesis tokens of required context; consequently, gradient descent pushes the shared scalar downward to protect high-loss predictions.

\paragraph{Probing What $\tau_{t,h}^\ell$ Encodes.}
Given that query conditioning is essential, we next ask whether the learned linear predictor $W_\tau^\ell$ merely computes a proxy for simple surface statistics (such as query norm or sequence position), or whether it captures higher-order semantic context. We evaluate our 126M ETA checkpoint on held-out C4 sequences and regress the realized per-token thresholds $\tau_{t,h}^\ell$ (separately for each layer, head, and position) against eight candidate summary statistics: the query Euclidean norm $\|q_t\|$, the log-sum-exp of the pre-softmax attention row ($\text{lse}$), the row score maximum ($s_{\max}$), standard deviation ($s_{\text{std}}$), mean ($s_{\text{mean}}$), and Shannon entropy ($H$), the logarithmic sequence position ($\log t$), and the input hidden-state norm ($\|x_t\|$).

\begin{table}[h]
\centering
\caption{\textbf{Proportion of variance ($R^2$) in learned thresholds $\tau_{t,h}^\ell$ explained by surface query and score-distribution statistics} on held-out C4 across all layers and heads (126M model). Even when combined in a joint multiple regression, all eight statistics account for less than one-third of threshold variance.}
\label{tab:tau_r2}
\resizebox{0.85\linewidth}{!}{%
\begin{tabular}{lcccccccc|c}
\toprule
\textbf{Statistic} & $\|q_t\|$ & $\text{lse}$ & $s_{\max}$ & $s_{\text{std}}$ & $H$ & $\log t$ & $s_{\text{mean}}$ & $\|x_t\|$ & \textbf{Joint (All 8)} \\
\midrule
$R^2$ & $0.187$ & $0.176$ & $0.168$ & $0.156$ & $0.091$ & $0.071$ & $0.059$ & $0.021$ & $\mathbf{0.327}$ \\
\bottomrule
\end{tabular}%
}
\end{table}

As shown in \autoref{tab:tau_r2}, the strongest individual feature ($\|q_t\|$) accounts for only $18.7\%$ of the variance in $\tau_{t,h}^\ell$, and a joint linear regression combining all eight statistics explains only $32.7\%$. More than two-thirds ($67.3\%$) of the variance in the learned threshold is irreducible to surface summary statistics, confirming that $W_\tau^\ell$ exploits fine-grained directional information in $q_t$.

\paragraph{Within- vs.\ Between-Document ANOVA Decomposition.}
To test whether ETA's threshold variations merely reflect global document-level domain shifts (e.g., allocating higher density to technical passages and lower density to conversational text) rather than local token-by-token difficulty, we decompose the variance of realized active densities $d_{i,t} = \frac{1}{t}\sum_{s=1}^t \mathbb{I}(S_{i,t,s} > \tau_{i,t})$ across $N$ held-out C4 documents of length $T$. Because active density naturally decays with sequence position $t$ (as a fixed number of relevant tokens represents a smaller fraction of a growing causal window), we first remove this mechanical positional trend by subtracting the cross-document positional mean $\bar{d}_{\cdot, t} = \frac{1}{N}\sum_{i=1}^N d_{i,t}$, yielding position-detrended residual densities $\tilde{d}_{i,t} = d_{i,t} - \bar{d}_{\cdot, t}$. Applying a one-way analysis of variance (ANOVA) partitions the total sum of squares across all $N \times T$ tokens into orthogonal between-document and within-document components:
$$\underbrace{\sum_{i=1}^N \sum_{t=1}^T (\tilde{d}_{i,t} - \bar{\tilde{d}})^2}_{\text{SS}_{\text{total}}} \;=\; \underbrace{T \sum_{i=1}^N (\bar{\tilde{d}}_{i,\cdot} - \bar{\tilde{d}})^2}_{\text{SS}_{\text{between}} \; (\approx 30\%)} \;+\; \underbrace{\sum_{i=1}^N \sum_{t=1}^T (\tilde{d}_{i,t} - \bar{\tilde{d}}_{i,\cdot})^2}_{\text{SS}_{\text{within}} \; (\approx 70\%)},$$
where $\bar{\tilde{d}}_{i,\cdot} = \frac{1}{T}\sum_{t=1}^T \tilde{d}_{i,t}$ is the mean detrended density of document $i$. If ETA functioned merely as a coarse document-level register switch, document means $\bar{\tilde{d}}_{i,\cdot}$ would dominate and the between-document share would approach $100\%$. Instead, $\approx 70\%$ of the variance occurs \emph{within} individual documents ($\text{SS}_{\text{within}}/\text{SS}_{\text{total}} \approx 0.70$), demonstrating that ETA continuously expands and contracts its attention window token by token—pruning aggressively on predictable syntax tokens while widening context on semantically demanding retrieval steps.

\subsection{Offline Threshold Calibration via \texorpdfstring{$1/t$}{1/t}-Weighted Histograms}
\label{appx:calibration_derivation}
At decoding step $t$, an attention head's active density is the fraction of its $t$ available causal keys retained by the soft gate, $\frac{1}{t}\sum_{s=1}^t m_{ts}^h$. Averaged over a small, disjoint validation/calibration split $\mathcal{D}$ ($|\mathcal{D}|=16$ sequences of length $T$), the dynamic predictor $\tau_{t,h}^\ell$ realizes a target head density 
$$
\bar{d}^{\,\ell h} = \frac{1}{|\mathcal{D}| T}\sum_{x \in \mathcal{D}} \sum_{t=1}^T \frac{1}{t}\sum_{s=1}^t m_{ts}^h(x).
$$
For stationary deployments, our goal is to replace the dynamic predictor with a single static constant $c^{\ell h}$ per head that yields this exact same empirical density on $\mathcal{D}$: $g(c^{\ell h}) = \bar{d}^{\,\ell h}$, where 
$$g(c)=\frac{1}{|\mathcal{D}| T}\sum_{x \in \mathcal{D}} \sum_{t=1}^T \frac{1}{t}\sum_{s=1}^t \sigma(\beta\cdot(S_{ts}^{\ell h}(x)-c))$$ 
denotes the head's expected density on $\mathcal{D}$ under threshold $c$. Because $g(c)$ decreases strictly and monotonically from $1$ to $0$ as $c$ increases, we can find the unique root $c^{\ell h}$ via continuous binary search.

\paragraph{Score Histograms and $1/t$ Row-Weighting.}
Evaluating $g(c)$ directly over millions of raw attention scores $S_{ts}^h$ across $\mathcal{D}$ at every binary search iteration would require repeatedly scanning the calibration split. Instead, during a single pass over $\mathcal{D}$, we bin raw scores $S_{ts}^{\ell h}$ into a 1D histogram $H^{\ell h}$ ($B$ bins with centers $s_b$). When building this histogram, we increment bin counts by $+\frac{1}{t}$ rather than $+1$ to account for the causal weighting. Once $H^{\ell h}$ is built, evaluating any candidate threshold during binary search reduces to a fast $B$-element dot product:
$$g(c) = \sum_{b=1}^{B} H^{\ell h}[b] \,\sigma\!\big(\beta(s_b - c)\big).$$
Assuming $B$ is large enough, binary search approximates the optimal value of $c$ within machine precision (residual $<2\times 10^{-16}$). See \autoref{alg:calibration} for more details.

\begin{algorithm}[h]
\caption{Density-Matched Offline Threshold Calibration}
\label{alg:calibration}
\begin{algorithmic}[1]
\Require Trained ETA model $M$ with dynamic predictor $\tau(\cdot)$; calibration split $\mathcal{D}$; bin count $B$; inverse temperature $\beta$
\State $H^{\ell h} \gets \mathbf{0} \in \mathbb{R}^{B}$, \quad $\bar{d}^{\,\ell h} \gets 0$ \quad for all $(\ell, h)$
\For{each calibration sequence $x \in \mathcal{D}$}
    \State Forward pass $M(x)$ with dynamic thresholds, recording raw scores $S^{\ell h}$ and gates $m^{\ell h}$
    \For{each layer $\ell$ and head $h$}
        \State $\bar{d}^{\,\ell h} \mathrel{+}= \tfrac{1}{|\mathcal{D}|} \cdot \operatorname{rowmean}\big(m^{\ell h}\big)$ \Comment{Accumulate target per-head density}
        \For{each causal pair $(t,s)$ with $s \le t$}
            \State $H^{\ell h}\big[\operatorname{bin}(S^{\ell h}_{ts})\big] \mathrel{+}= \frac{1}{t}$ \Comment{Accumulate $1/t$ causal row weight}
        \EndFor
    \EndFor
\EndFor
\For{each layer $\ell$ and head $h$}
    \State $H^{\ell h} \gets H^{\ell h} / \textstyle\sum_{b=1}^{B} H^{\ell h}[b]$ \Comment{Normalize weighted score histogram}
    \State Define $g(c) = \sum_{b=1}^{B} H^{\ell h}[b] \,\sigma\!\big(\beta(s_b - c)\big)$ \Comment{Monotonically decreasing in $c$}
    \State $c^{\ell h} \gets \textsc{Bisect}\big(g(c) = \bar{d}^{\,\ell h}\big)$ \Comment{Exact root-finding via bisection}
\EndFor
\State \Return Static threshold table $\{c^{\ell h}\}$
\end{algorithmic}
\end{algorithm}

\subsection{Distillation Evaluation and Length-Extrapolation Stability}
\label{subsection:frozen_benchmarks}

We evaluate calibrated static thresholds $\{c^{\ell h}\}$ against dynamic per-token thresholds $\tau_{t,h}^\ell$ on held-out C4 sequences executed directly through our fused block-sparse Triton kernel ($b=4$) on the 126M model ($L_{\text{train}}=1024$).

\begin{table}[h]
\centering
\caption{\textbf{Dynamic per-token thresholds ($\tau_{t,h}^\ell$) versus calibrated static per-(layer, head) thresholds ($c^{\ell h}$)}, evaluated on held-out C4 via the block-sparse decode kernel ($b=4$, 126M model with $L_{\text{train}}=1024$). Constants are calibrated on a disjoint $16$-sequence split. In-distribution evaluations are paired at the batch level ($\rho = 0.9996$, $t = 8.64$).}
\label{tab:frozen_tau}
\resizebox{0.82\linewidth}{!}{%
\begin{tabular}{llcc}
\toprule
\textbf{Evaluation Regime} & \textbf{Metric} & \textbf{Dynamic $\tau_{t,h}^\ell$} & \textbf{Calibrated Static $c^{\ell h}$} \\
\midrule
In-Distribution ($L=1024$) & Perplexity ($\downarrow$) & $91.96$ & $\mathbf{91.60}$ \\
 & Active Head Density ($\downarrow$) & $24.14\%$ & $\mathbf{22.40\%}$ \\
\midrule
Length Extrapolation ($L=2048$, $z=3.0$) & Perplexity ($\downarrow$) & $100.52$ & $\mathbf{100.48}$ \\
 & Active Head Density ($\downarrow$) & $44.77\%$ & $\mathbf{32.83\%}$ \\
\bottomrule
\end{tabular}%
}
\end{table}

Both in-distribution and length-extrapolated experiments show that our calibrated static thresholds achieve slightly better perplexity and higher sparsity than the dynamic thresholds when evaluated on the target domain used for calibration. 

\section{Hardware-Aligned Training Kernel Implementation}
\label{app:triton_kernel}

To train our model efficiently without materializing $T \times T$ matrices, we implement a custom fused training attention kernel in Triton. Our implementation extends the block-based online softmax algorithm of FlashAttention-2 by tightly integrating the differentiable thresholding mechanism and the gradient tracking for the sparsity regularizer.

\subsection{Forward Pass Formulation}
In the forward pass, the kernel divides the query sequence into blocks of size $B_M$ and the key/value sequences into blocks of size $B_N$. The outer loop iterates over query blocks, loading queries $Q_{block}$ and their corresponding predicted thresholds $\tau_{block}$ from HBM to SRAM. The inner loop iterates causally over key/value blocks ($K_{block}, V_{block}$). 

For a given query $q_i$ and key $k_j$ within these blocks, the scaled attention score is computed as $S_{ij} = \text{scale} \cdot \langle q_i, k_j \rangle$. Concurrently, the kernel computes the soft inclusion gate $m_{ij}$ (denoted $\sigma_{ij}$ below) via the sigmoid relaxation:
$$\sigma_{ij} = \sigma(\beta(S_{ij} - \tau_i))$$
In particular, the sum of these probabilities is accumulated directly in SRAM to compute the sparsity regularizer $\mathcal{L}_{\text{sparsity}}$ without allocating an explicit full sequence tensor:
$$R_i = \sum_{j \leq i} \sigma_{ij}$$

The kernel then applies the selected dynamic mask operation (additive or multiplicative) to yield the masked scores $\tilde{S}_{ij}$. To prevent numerical overflow, we track the running maximum $m_i$ and the running exponential sum $l_i$ via a standard online softmax update\footnote{A familiar formulation from the original FlashAttention \citep{dao2022flashattention}.}:
\begin{align*}
    m_i^{\text{new}} &= \max(m_i, \max_j \tilde{S}_{ij}) \\
    l_i^{\text{new}} &= l_i e^{m_i - m_i^{\text{new}}} + \sum_j e^{\tilde{S}_{ij} - m_i^{\text{new}}}
\end{align*}
The attention output is accumulated as a weighted sum of $V_{block}$ and normalized by $l_i^{\text{new}}$ before being written back to HBM. We additionally store the final log-sum-exp vectors $L_i = m_i + \ln(l_i)$ and the row-wise sums $R_i$ to facilitate the backward pass. \autoref{alg:fwd_kernel} gives the full forward pass.

\begin{algorithm}[h]
\caption{Hardware-Aligned Forward Kernel}
\label{alg:fwd_kernel}
\begin{algorithmic}[1]
\Require Queries $Q$, Keys $K$, Values $V \in \mathbb{R}^{T \times d}$
\Require Thresholds $\tau \in \mathbb{R}^T$, Scale $s$, Temperature $\beta$
\Ensure Output $O \in \mathbb{R}^{T \times d}$, Regularizer Sums $\Sigma_{\text{sums}} \in \mathbb{R}^T$, Log-Sum-Exp $L \in \mathbb{R}^T$

\For{each query block $i$ from $1$ to $T / B_M$}
    \State Load $Q_i, \tau_i$ from HBM to SRAM
    \State Initialize $m_i \gets -\infty$, $l_i \gets 0$, $acc_i \gets 0$, $\Sigma_{\text{sums}, i} \gets 0$
    
    \For{each key/value block $j$ from $1$ to $i$} \Comment{Iterate up to causal limit}
        \State Load $K_j, V_j$ from HBM to SRAM
        \State $S_{ij} \gets s \cdot Q_i K_j^T$ \Comment{Compute pre-softmax scores}
        \State $\sigma_{ij} \gets \text{sigmoid}(\beta \cdot (S_{ij} - \tau_i))$ \Comment{Compute inclusion gates}
        \State $\Sigma_{\text{sums}, i} \gets \Sigma_{\text{sums}, i} + \text{row\_sum}(\sigma_{ij})$ \Comment{Accumulate regularizer}
        
        \If{additive}
            \State $\tilde{S}_{ij} \gets S_{ij} + C\cdot \sigma_{ij} -C$ \Comment{Additive masking}
        \Else
            \State $\tilde{S}_{ij} \gets S_{ij} \odot \sigma_{ij}$ \Comment{Multiplicative masking}
        \EndIf
        
        \State Apply causal mask to $\tilde{S}_{ij}$; for diagonal blocks ($j = i$), restore token diagonal $\mathrm{diag}(\tilde{S}_{ii}) \gets \mathrm{diag}(S_{ii})$ \Comment{Bypass threshold on diagonal}
        
        \State $m_{\text{new}} \gets \max(m_i, \max(\tilde{S}_{ij}))$ \Comment{Online softmax update}
        \State $\alpha \gets \exp(m_i - m_{\text{new}})$
        \State $P_{ij} \gets \exp(\tilde{S}_{ij} - m_{\text{new}})$
        
        \State $l_{\text{new}} \gets l_i \cdot \alpha + \text{row\_sum}(P_{ij})$
        \State $acc_i \gets acc_i \cdot \alpha + P_{ij} V_j$
        \State $m_i \gets m_{\text{new}}$, \quad $l_i \gets l_{\text{new}}$
    \EndFor
    
    \State $O_i \gets acc_i / l_i$ \Comment{Normalize output}
    \State $L_i \gets m_i + \ln(l_i)$ \Comment{Store log-sum-exp for backward pass}
    \State Store $O_i, \Sigma_{\text{sums}, i}, L_i$ from SRAM to HBM
\EndFor
\end{algorithmic}
\end{algorithm}

\subsection{Backward Pass and Custom Gradients}
The backward pass is significantly more complex than standard attention due to the chain rule dependencies flowing through the threshold $\tau_i$ and the regularization signal $\nabla_{\sigma} \mathcal{L}_{\text{sparsity}}$. 

The backward kernel loops over the $K$ and $V$ blocks in the outer loop, and iterates over the $Q$ blocks in the inner loop. For each block pair, the kernel recomputes the forward logits $S_{ij}$, the inclusion probabilities $\sigma_{ij}$, and the softmax probabilities $P_{ij} = \exp(\tilde{S}_{ij} - L_i)$. 

Let $dO_i \in \mathbb{R}^d$ be the upstream gradient of the attention output, and $d\Sigma_i$ be the upstream gradient from the sparsity regularization loss. The standard derivatives are computed as follows (for completeness we include the derivations in \autoref{app:attention_gradients}):
\begin{align*}
    dV_j &= \sum_i P_{ij} dO_i \\
    dP_{ij} &= dO_i^\top \cdot V_j \\
    dS_{ij}^{\text{masked}} &= P_{ij} \left( dP_{ij} - \sum_k P_{ik} dP_{ik} \right)
\end{align*}

Because our soft inclusion mask relies on $\sigma_{ij}$, gradients must be routed back to the threshold predictor. The local gradient with respect to the threshold $\tau_i$ incorporates both the task loss (routed through the mask) and the sparsity regularizer (routed through $d\Sigma_i$). Using the derivative of the sigmoid function, $\sigma'_{ij} = \sigma_{ij}(1 - \sigma_{ij})$, the gradient step for the threshold is the following (see \autoref{app:attention_gradients}):
$$d\tau_i = \sum_j \left( dS_{ij}^{\text{masked}} \frac{\partial \tilde{S}_{ij}}{\partial \sigma_{ij}} + d\Sigma_i \right) \big(-\beta \cdot \sigma'_{ij}\big)$$

Similarly, the gradient with respect to the raw pre-softmax attention score $S_{ij}$ receives an additional term from the thresholding mechanism, yielding the final $dS_{ij}^{\text{raw}}$ used to compute $dK$ and $dQ$. 

Because the backward pass's inner loop iterates over queries while holding keys constant, gradients for $Q$ and $\tau$ are accumulated atomically across multiple inner iterations (\texttt{tl.atomic\_add}), ensuring thread-safe gradient reductions across the sequence length without necessitating massive workspace buffers. \autoref{alg:bwd_kernel} gives the full backward pass.

\begin{algorithm}[h]
\caption{Hardware-Aligned Backward Kernel}
\label{alg:bwd_kernel}
\begin{algorithmic}[1]
\Require Forward tensors $Q, K, V, \tau, O, L$, upstream gradients $dO, d\Sigma$
\Require Scale $s$, Temperature $\beta$
\Ensure Gradients $dQ, dK, dV \in \mathbb{R}^{T \times d}$, $d\tau \in \mathbb{R}^T$

\State Initialize $dQ, d\tau \gets 0$ in HBM
\For{each key block $j$ from $1$ to $T / B_N$}
    \State Load $K_j, V_j$ from HBM to SRAM
    \State Initialize $dK_j \gets 0$, $dV_j \gets 0$ in SRAM
    
    \For{each query block $i$ from $j$ to $T / B_M$} \Comment{Respect causal bounds}
        \State Load $Q_i, \tau_i, O_i, L_i, dO_i, d\Sigma_i$ from HBM to SRAM
        
        \State Recompute $S_{ij} \gets s \cdot Q_i K_j^\top$
        \State Recompute $\sigma_{ij} \gets \text{sigmoid}(\beta \cdot (S_{ij} - \tau_i))$
        \State Recompute masked scores $\tilde{S}_{ij}$ \Comment{Using strict or soft mask}
        \State Apply causal/diagonal mask to $\tilde{S}_{ij}$
        
        \State $P_{ij} \gets \exp(\tilde{S}_{ij} - L_i)$ \Comment{Recompute attention probabilities}
        
        \State $dV_j \gets dV_j + P_{ij}^\top dO_i$ \Comment{Accumulate value gradients}
        \State $dP_{ij} \gets dO_i V_j^\top$
        \State $D_i \gets \text{row\_sum}(dO_i \odot O_i)$ \Comment{Softmax scaling factor}
        
        \State $dS_{ij}^{\text{masked}} \gets P_{ij} \odot (dP_{ij} - D_i)$ \Comment{Gradient through softmax}
        
        \State Compute threshold gradient $d\tau_{\text{local}}$ using $dS_{ij}^{\text{masked}}, d\Sigma_i$, and $\sigma'_{ij}$
        \State Compute raw pre-softmax gradient $dS_{ij}^{\text{raw}}$
        
        \State $dK_j \gets dK_j + s(dS_{ij}^{\text{raw}})^\top Q_i$ \Comment{Accumulate key gradients}
        
        \State \texttt{atomic\_add}($dQ_i$, $dS_{ij}^{\text{raw}} K_j$) \Comment{Accumulate to HBM safely}
        \State \texttt{atomic\_add}($d\tau_i$, $\text{row\_sum}(d\tau_{\text{local}})$)
    \EndFor
    \State Store $dK_j, dV_j$ from SRAM to HBM
\EndFor
\end{algorithmic}
\end{algorithm}

\subsection{Derivation of the Attention Gradients}
\label{app:attention_gradients}

To formulate the backward pass, we differentiate the scalar loss $\mathcal{L}$ with respect to the intermediate variables in reverse. We denote the upstream gradient of the loss with respect to any variable $X$ as $dX = \frac{\partial \mathcal{L}}{\partial X}$. Recall the forward attention operations for a query $i$ and key $j$:
\begin{equation}
    S_{ij} = s \langle Q_i, K_j \rangle, \quad P_{ij} = \frac{\exp(\tilde{S}_{ij})}{\sum_k \exp(\tilde{S}_{ik})}, \quad O_i = \sum_k P_{ik} V_k
\end{equation}
where $s$ is the scaling factor and $\tilde{S}_{ij}$ is the dynamically masked pre-softmax logit.

We first compute the gradients with respect to the value vectors $V_j$ and the scalar attention probabilities $P_{ij}$. Applying the chain rule to the forward output sum yields:
\begin{equation}
    dV_j = \sum_i P_{ij} dO_i, \quad \text{and} \quad dP_{ij} = dO_i \cdot V_j
\end{equation}

Next, we route the gradient through the softmax normalization to find the derivative with respect to the masked logits, $\tilde{S}_{ij}$. Because $\tilde{S}_{ij}$ affects all probabilities $P_{ik}$ for query $i$, we sum over $k$ using the standard softmax Jacobian, $\frac{\partial P_{ik}}{\partial \tilde{S}_{ij}} = P_{ij}(\delta_{jk} - P_{ik})$, where $\delta_{jk}$ is the Kronecker delta:
\begin{equation}
    dS_{ij}^{\text{masked}} = \sum_k dP_{ik} \frac{\partial P_{ik}}{\partial \tilde{S}_{ij}} = P_{ij} \left( dP_{ij} - \sum_k P_{ik} dP_{ik} \right)
\end{equation}

The threshold $\tau_i$ and the raw pre-softmax logit $S_{ij}$ interact through the inclusion probability $\sigma_{ij} = \sigma(\beta(S_{ij} - \tau_i))$. This probability branches into two computational pathways: the task loss (acting as a dynamic mask to produce $\tilde{S}_{ij}$) and the sparsity regularization loss (accumulating into the row-wise sum $R_i = \sum_j \sigma_{ij}$). 

Both gradient derivations rely on the shared upstream gradient flowing through $\sigma_{ij}$. Using $dS_{ij}^{\text{masked}} = \frac{\partial \mathcal{L}}{\partial \tilde{S}_{ij}}$ and $d\Sigma_i = \frac{\partial \mathcal{L}}{\partial R_i}$, we apply the chain rule to find:
\begin{equation*}
    \frac{\partial \mathcal{L}}{\partial \sigma_{ij}} = dS_{ij}^{\text{masked}} \frac{\partial \tilde{S}_{ij}}{\partial \sigma_{ij}} + d\Sigma_i
\end{equation*}

\paragraph{Threshold Gradient ($d\tau_i$):} 
The threshold $\tau_i$ only affects the loss via $\sigma_{ij}$. Because it influences all keys $j$ up to the causal limit, we sum the chain rule across the sequence. Evaluating the inner derivative $\frac{\partial \sigma_{ij}}{\partial \tau_i} = -\beta \sigma'_{ij}$ (where $\sigma'_{ij} = \sigma_{ij}(1 - \sigma_{ij})$), we obtain:
\begin{equation*}
    d\tau_i = \sum_j \frac{\partial \mathcal{L}}{\partial \sigma_{ij}} \frac{\partial \sigma_{ij}}{\partial \tau_i} = \sum_j \left( dS_{ij}^{\text{masked}} \frac{\partial \tilde{S}_{ij}}{\partial \sigma_{ij}} + d\Sigma_i \right) \big(-\beta \cdot \sigma'_{ij}\big)
\end{equation*}

\paragraph{Raw Score Gradient ($dS_{ij}^{\text{raw}}$):} 
Unlike the threshold, the raw score $S_{ij}$ affects the loss through both the indirect sigmoid path ($\sigma_{ij}$) and the direct dynamic mask path ($\tilde{S}_{ij}$). Applying the multivariate chain rule:
\begin{equation*}
    dS_{ij}^{\text{raw}} = \frac{\partial \mathcal{L}}{\partial \tilde{S}_{ij}} \frac{\partial \tilde{S}_{ij}}{\partial S_{ij}} + \frac{\partial \mathcal{L}}{\partial \sigma_{ij}} \frac{\partial \sigma_{ij}}{\partial S_{ij}}
\end{equation*}
Notice the mathematical symmetry in the sigmoid's argument: the inner derivative with respect to the score is $\frac{\partial \sigma_{ij}}{\partial S_{ij}} = +\beta \sigma'_{ij}$. Substituting this and the shared intermediate term yields the final local gradient:
\begin{equation*}
    dS_{ij}^{\text{raw}} = dS_{ij}^{\text{masked}} \frac{\partial \tilde{S}_{ij}}{\partial S_{ij}} + \left( dS_{ij}^{\text{masked}} \frac{\partial \tilde{S}_{ij}}{\partial \sigma_{ij}} + d\Sigma_i \right) \big(\beta \cdot \sigma'_{ij}\big)
\end{equation*}

\paragraph{Masking Derivatives:}
To implement these equations natively, the exact values of the local mask derivatives are determined by the chosen dynamic masking strategy:
\begin{itemize}
    \item \textbf{Additive Masking} ($\tilde{S}_{ij} = S_{ij} + C \cdot \sigma_{ij} - C$): The direct derivative is $\frac{\partial \tilde{S}_{ij}}{\partial S_{ij}} = 1$, and the sigmoid interaction is $\frac{\partial \tilde{S}_{ij}}{\partial \sigma_{ij}} = C$.
    \item \textbf{Multiplicative Masking} ($\tilde{S}_{ij} = S_{ij} \odot \sigma_{ij}$): The direct derivative is $\frac{\partial \tilde{S}_{ij}}{\partial S_{ij}} = \sigma_{ij}$, and the sigmoid interaction is $\frac{\partial \tilde{S}_{ij}}{\partial \sigma_{ij}} = S_{ij}$.
\end{itemize}

\paragraph{Query and Key Gradients} As we saw, $dS_{ij}^{\text{masked}}$ is combined with the sparsity regularization gradient to form the total upstream gradient for the raw score, denoted as $dS_{ij}^{\text{raw}}$. We distribute this raw gradient back to the query and key vectors. Differentiating the scaled dot-product $S_{ij} = s Q_i K_j^T$ yields symmetric accumulations:
\begin{equation}
    dQ_i = \sum_j s \cdot dS_{ij}^{\text{raw}} K_j, \quad \text{and} \quad dK_j = \sum_i s \cdot dS_{ij}^{\text{raw}} Q_i
\end{equation}

From a hardware implementation perspective, the scalar $s$ is absorbed into $dS_{ij}^{\text{raw}}$ immediately prior to these matrix multiplications. Because our backward kernel holds keys in the outer loop and iterates over queries in the inner loop, $dK_j$ is accumulated safely in on-chip SRAM, whereas $dQ_i$ requires atomic additions to prevent race conditions in global High Bandwidth Memory (HBM).

\subsection{Complexity Analysis}
\label{appx:complexities}

To formally evaluate the hardware efficiency of our Threshold Attention mechanism, we analyze its space, time (FLOPs), and memory access (I/O) complexities. We follow the analytical framework of FlashAttention \citep{dao2022flashattention}. 

Let $T$ be the sequence length and $d$ be the head dimension. Let $M$ denote the capacity of the fast on-chip SRAM. To maximize SRAM utilization without overflowing, our block sizes are constrained such that $B_M, B_N = \Theta(M / d)$. For this analysis, we assume a causal masking paradigm, meaning the inner loops traverse approximately $T^2 / 2$ total element interactions.

\paragraph{Forward Pass Complexity}
In standard attention, materializing the dense $T \times T$ attention matrix requires $\mathcal{O}(T^2)$ space and $\mathcal{O}(T^2)$ High Bandwidth Memory (HBM) accesses. Our hardware-aligned forward pass avoids this.
\begin{itemize}
    \item \textbf{Space:} We only store the inputs $Q, K, V \in \mathbb{R}^{T \times d}$, the threshold parameters $\tau \in \mathbb{R}^T$, the final output $O \in \mathbb{R}^{T \times d}$, and the vectors required for the backward pass ($L, \Sigma \in \mathbb{R}^T$). The dynamic mask $\sigma_{ij}$ is computed locally in SRAM and never explicitly materialized. Thus, the total space complexity is $\mathcal{O}(Td)$, identical to standard FlashAttention.
    \item \textbf{Time:} The inner loop computes the dot product, the sigmoid inclusion probability, and the softmax accumulator updates. While the sigmoid and masking operations introduce a constant factor overhead compared to raw FlashAttention, they are element-wise $\mathcal{O}(1)$ operations. Across all block pairs, the dominant operation remains the matrix multiplications ($Q_i K_j^T$ and $P_{ij} V_j$), yielding a time complexity of $\mathcal{O}(T^2 d)$.
    \item \textbf{I/O (HBM Accesses):} The outer loop iterates over $T/B_M$ query blocks. For each query block, it loads $Q_i$ and $\tau_i$ from HBM ($\Theta(B_M d)$ elements). The inner loop iterates over $T/B_N$ key/value blocks, loading $K_j$ and $V_j$ ($\Theta(B_N d)$ elements). The total number of HBM accesses is given by:
    \begin{equation}
        \text{HBM Accesses} = \sum_{i=1}^{T/B_M} \left[ \Theta(B_M d) + \sum_{j=1}^{i} \Theta(B_N d) \right] = \mathcal{O} \left( \frac{T^2 d}{B_M} \right)
    \end{equation}
    Substituting $B_M = \Theta(M / d)$, the overall I/O complexity is $\mathcal{O}(T^2 d^2 M^{-1})$.
\end{itemize}

\paragraph{Backward Pass Complexity}
The backward pass reconstructs the forward pass intermediate states to avoid saving the massive $T \times T$ matrices, trading additional FLOPs for substantially reduced memory reads.
\begin{itemize}
    \item \textbf{Space:} The backward pass requires storing the gradients $dQ, dK, dV \in \mathbb{R}^{T \times d}$ and $d\tau \in \mathbb{R}^T$, requiring $\mathcal{O}(Td)$ space.
    \item \textbf{Time:} Recomputing $S_{ij}$ and $\sigma_{ij}$ requires the same $\mathcal{O}(T^2 d)$ FLOPs as the forward pass. Computing the standard attention gradients ($dS_{ij}, dP_{ij}, dV_j, dK_j, dQ_i$) and the threshold gradient $d\tau_i$ involves purely linear matrix multiplications and element-wise derivatives. Thus, the total time complexity remains $\mathcal{O}(T^2 d)$.
    \item \textbf{I/O (HBM Accesses):} The outer loop iterates over $T/B_N$ key blocks, holding $K_j$ and $V_j$ in SRAM while accumulating $dK_j$ and $dV_j$. The inner loop streams the query blocks $Q_i, dO_i$ and vectors $L, \tau, D, d\Sigma$ from HBM ($\Theta(B_M d)$ elements). The gradients $dQ_i$ and $d\tau_i$ are accumulated to HBM using atomic additions. The total HBM access cost is:
    \begin{equation}
        \text{HBM Accesses} = \sum_{j=1}^{T/B_N} \left[ \Theta(B_N d) + \sum_{i=j}^{T/B_M} \Theta(B_M d) \right] = \mathcal{O} \left( \frac{T^2 d}{B_N} \right)
    \end{equation}
    Substituting $B_N = \Theta(M / d)$, the backward I/O complexity is $\mathcal{O}(T^2 d^2 M^{-1})$.
\end{itemize}

\section{Inference-Time Block-Sparse Fused Kernel}
\label{appx:fused-kernel}
Here we detail the algorithmic design and systems optimizations of our custom Triton \citep{tillet2019triton} fused kernel for block-sparse attention decoding. The design can be visualized in \autoref{fig:inference_block_selective_load}.

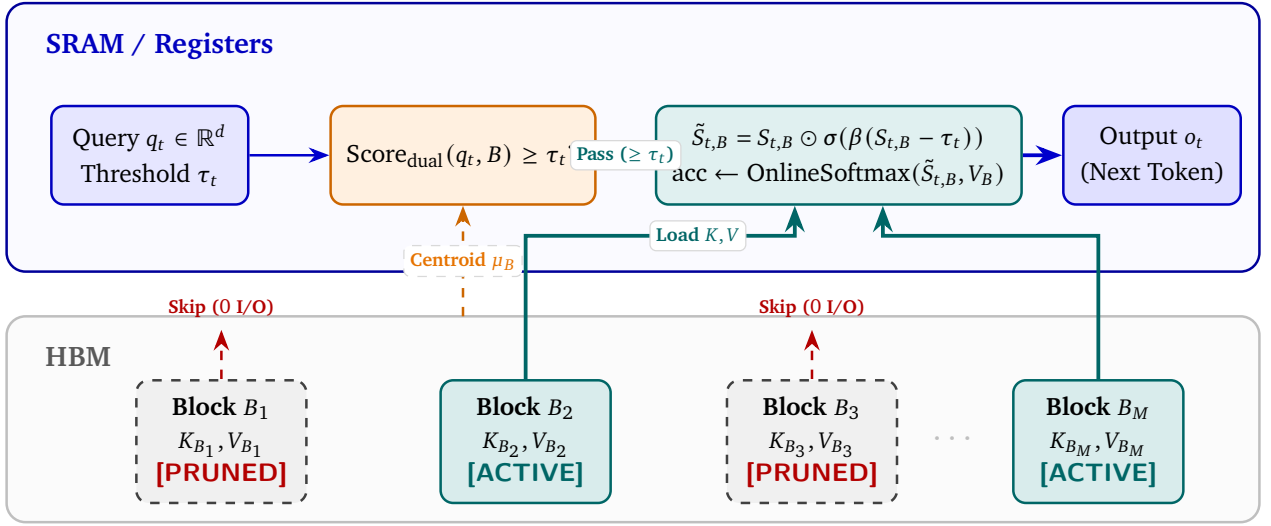
\begin{figure}[t]
\centering
\resizebox{\linewidth}{!}{%
\begin{tikzpicture}[
    >=Stealth,
    font=\sffamily,
    sram_area/.style={draw=blue!60!black, fill=blue!2, rounded corners=6pt, line width=0.8pt},
    hbm_area/.style={draw=gray!50, fill=gray!3, rounded corners=6pt, line width=0.8pt},
    node_base/.style={rounded corners=4pt, align=center, inner sep=5pt, font=\scriptsize},
    input_node/.style={node_base, draw=blue!75!black, fill=blue!10, thick, minimum width=2.2cm, minimum height=1.1cm},
    screen_node/.style={node_base, draw=orange!80!black, fill=orange!12, thick, minimum width=2.8cm, minimum height=1.1cm},
    compute_node/.style={node_base, draw=teal!80!black, fill=teal!12, thick, minimum width=3.0cm, minimum height=1.1cm},
    output_node/.style={node_base, draw=blue!75!black, fill=blue!12, thick, minimum width=1.6cm, minimum height=1.1cm},
    block_active/.style={node_base, draw=teal!80!black, fill=teal!15, thick, minimum width=1.9cm, minimum height=1.0cm},
    block_pruned/.style={node_base, draw=gray!50!black, fill=gray!12, thick, dashed, minimum width=1.9cm, minimum height=1.0cm},
    badge/.style={font=\tiny\bfseries, fill=white, inner sep=2pt, rounded corners=2pt, draw=gray!30, line width=0.4pt},
    header_text/.style={font=\footnotesize\bfseries, align=left}
]
    \draw[sram_area] (0.0, 2.8) rectangle (14.0, 5.8);
    \node[header_text, anchor=north west, color=blue!75!black] at (0.3, 5.6) {SRAM / Registers};

    \draw[hbm_area] (0.0, 0.0) rectangle (14.0, 2.3);
    \node[header_text, anchor=north west, color=gray!70!black] at (0.3, 2.1) {HBM};

    \node[input_node] (inputs) at (1.6, 4.1) {Query $q_t \in \mathbb{R}^d$\\[2pt]Threshold $\tau_t$};
    \node[screen_node] (screen) at (5.1, 4.1) {$\text{Score}_{\text{dual}}(q_t, B) \ge \tau_t$?};
    \node[compute_node] (softmax) at (9.3, 4.1) {$\tilde{S}_{t, B} = S_{t, B} \odot \sigma(\beta(S_{t, B} - \tau_t))$\\[2pt]$\text{acc} \leftarrow \text{OnlineSoftmax}(\tilde{S}_{t, B}, V_B)$};
    \node[output_node] (output) at (12.8, 4.1) {Output $o_t$\\[2pt]\scriptsize (Next Token)};

    \draw[->, thick, draw=blue!80!black] (inputs.east) -- (screen.west);
    \draw[->, thick, draw=black!75] (screen.east) -- (softmax.west) node[midway, badge, text=teal!80!black] {Pass ($\ge \tau_t$)};
    \draw[->, very thick, draw=blue!80!black] (softmax.east) -- (output.west);

    \node[block_pruned] (b1) at (2.4, 0.9) {\textbf{Block $B_1$}\\[1pt]$K_{B_1}, V_{B_1}$\\[1pt]\color{red!70!black}\textbf{\textsf{[PRUNED]}}};
    \node[block_active] (b2) at (5.8, 0.9) {\textbf{Block $B_2$}\\[1pt]$K_{B_2}, V_{B_2}$\\[1pt]\color{teal!80!black}\textbf{\textsf{[ACTIVE]}}};
    \node[block_pruned] (b3) at (9.0, 0.9) {\textbf{Block $B_3$}\\[1pt]$K_{B_3}, V_{B_3}$\\[1pt]\color{red!70!black}\textbf{\textsf{[PRUNED]}}};
    \node[font=\bfseries\color{gray!50}] at (10.6, 0.9) {$\cdots$};
    \node[block_active] (bm) at (12.2, 0.9) {\textbf{Block $B_M$}\\[1pt]$K_{B_M}, V_{B_M}$\\[1pt]\color{teal!80!black}\textbf{\textsf{[ACTIVE]}}};

    \draw[->, thick, dashed, draw=orange!80!black] (5.1, 2.3) -- (screen.south) node[midway, badge, text=orange!90!black] {Centroid $\mu_B$};
    \draw[->, thick, dashed, draw=red!70!black] (b1.north) -- ++(0, 0.55) node[above, font=\tiny\bfseries\color{red!70!black}] {Skip ($0$ I/O)};
    \draw[->, thick, dashed, draw=red!70!black] (b3.north) -- ++(0, 0.55) node[above, font=\tiny\bfseries\color{red!70!black}] {Skip ($0$ I/O)};
    \draw[->, very thick, draw=teal!80!black] (b2.north) -- (5.8, 3.2) -| ([xshift=-14pt]softmax.south) node[pos=0.32, badge, text=teal!80!black] {Load $K, V$};
    \draw[->, very thick, draw=teal!80!black] (bm.north) -- (12.2, 3.2) -| ([xshift=14pt]softmax.south);
\end{tikzpicture}%
}
\caption{\textbf{Inference-time block-sparse selective loading mechanism.} Unimportant blocks are pruned in SRAM via lightweight metadata screening, while active blocks are loaded from HBM for fused online softmax accumulation.}
\label{fig:inference_block_selective_load}
\end{figure}

\subsection{GQA-Grouped Execution and Compact Metadata Layout}
In Grouped-Query Attention (GQA), $G = H_Q / H_{KV}$ query heads share a single key-value head ($G = 8$ in our 1.45B model). If a decoding kernel launches a separate thread block for each query head, any KV block needed by multiple query heads in the same group is fetched from HBM up to $G$ times, wasting memory bandwidth. To eliminate redundant reads, our kernel instead assigns each thread block to an entire GQA group of $G$ query heads sharing a single KV head:
\begin{enumerate}[leftmargin=*,itemsep=2pt,topsep=2pt]
    \item \textbf{Batched Metadata Screening:} For each chunk of candidate blocks, the thread block loads the per-block metadata ($\boldsymbol{\mu}_B, \boldsymbol{\sigma}_B, M_B$) once into SRAM and evaluates $\operatorname{Score}_{\text{dual}}(q_t^h, B) \ge \tilde{\tau}_{t,h}$ across all $G$ query heads simultaneously via a single Tensor Core matrix multiplication (\texttt{tl.dot}).
    \item \textbf{Union KV Loading with Per-Head Masking:} A full KV block $B$ is fetched from HBM into SRAM \emph{at most once}—whenever at least one of the $G$ query heads selects it (forming the GQA union). Once in SRAM, each query head updates its own online-softmax accumulator using only the blocks that passed its individual threshold $\tilde{\tau}_{t,h}$.
\end{enumerate}
As a result, actual HBM traffic scales with the \emph{GQA union density} $D_{\text{union}}$ across the $H_{KV}$ groups rather than the sum of individual query-head densities.

To minimize metadata I/O, each per-block record stores only the 64-dimensional centroid $\boldsymbol{\mu}_B$ (\texttt{float16}), 64-dimensional coordinate standard deviation $\boldsymbol{\sigma}_B$ (\texttt{float16}), and scalar maximum norm $M_B$ (\texttt{float32}), totaling $258$ bytes per block per KV head ($1.6\%$ overhead at $b=64$). Caching $\boldsymbol{\sigma}_B$ rather than variance $\boldsymbol{\sigma}_B^2$ prevents half-precision underflow below $6 \times 10^{-5}$ in low-variance blocks.

\subsection{Mixed Precision and Numerical Stability}
While Tensor Core matrix multiplications ($q K_B^\top$ and $P V_B$) execute in \texttt{float16}, squaring coordinates in half precision during proxy bound evaluation can induce overflow or cancellation. Consequently, all coordinate-wise square accumulations, dot-product bounding terms ($\langle q, \boldsymbol{\mu}_B \rangle$), scaling factors, and online softmax accumulators ($m_c, l_c$) are strictly promoted to \texttt{float32}.

\subsection{Chunked Decoding and Two-Stage Reduction}
To maximize Streaming Multiprocessor (SM) occupancy during long-context decoding, the block sequence is partitioned into independent \emph{chunks} of $c$ contiguous blocks. The primary scan kernel evaluates block thresholds and updates local online softmax accumulators ($m_{\text{mid}}^{(k)}$, $l_{\text{mid}}^{(k)}$, $O_{\text{mid}}^{(k)} \in \mathbb{R}^d$) independently for each chunk $k$. When $N_{\text{chunks}} > 1$, a secondary reduction kernel aggregates chunk statistics via associative online softmax reduction:
$$
o = \frac{\sum_{k=1}^{N_{\text{chunks}}} O_{\text{mid}}^{(k)} \exp\left(m_{\text{mid}}^{(k)} - m_{\text{global}}\right)}{\sum_{k=1}^{N_{\text{chunks}}} l_{\text{mid}}^{(k)} \exp\left(m_{\text{mid}}^{(k)} - m_{\text{global}}\right)}, \quad \text{where } m_{\text{global}} = \max_k \, m_{\text{mid}}^{(k)}.
$$
When $N_{\text{chunks}} = 1$, a fast path normalizes and writes $o = \text{acc}_1 / l_1$ directly to HBM without staging intermediate buffers. Larger chunk sizes $c$ reduce intermediate buffer traffic and reduction trip counts at the cost of grid parallelism; as shown in \autoref{tab:standalone_arms}, $b=64, c=128$ achieves optimal throughput at $51.2\%$ union density, whereas coarser blocks ($b=128, c=64$) reach $2.50\times$ speedup in sparse regimes ($16.0\%$ union).

\begin{table}[h]
\centering
\caption{\textbf{Effect of block size ($b$) and chunk size ($c$)} at batch size $B=64$, $L=512$K ($32.8$\,GB KV cache), NVIDIA H100 SXM $80$\,GB, \texttt{float16}. $c$ denotes blocks per chunk. Coarser blocks reduce metadata overhead in sparse regimes ($16\%$), while $b=64$ achieves the best balance at $51.2\%$ union density.}
\label{tab:standalone_arms}
\begin{tabular}{lcccccc}
\toprule
& \multicolumn{2}{c}{$51.2\%$ union} & \multicolumn{2}{c}{$35.5\%$ union} & \multicolumn{2}{c}{$16.0\%$ union} \\
\cmidrule(lr){2-3} \cmidrule(lr){4-5} \cmidrule(lr){6-7}
\textbf{Arm} & ms & $\times$ & ms & $\times$ & ms & $\times$ \\
\midrule
FlashAttention-2 & $15.085$ & --- & $15.009$ & --- & $14.790$ & --- \\
ETA $b=32$, $c=64$ & $16.617$ & $0.91$ & $13.498$ & $1.11$ & $9.207$ & $1.61$ \\
ETA $b=64$, $c=64$ & $13.797$ & $1.09$ & $10.566$ & $1.42$ & $6.429$ & $2.30$ \\
ETA $b=64$, $c=128$ & $\mathbf{13.432}$ & $\mathbf{1.12}$ & $\mathbf{10.430}$ & $\mathbf{1.44}$ & $6.520$ & $2.27$ \\
ETA $b=128$, $c=64$ & $16.733$ & $0.90$ & $11.472$ & $1.31$ & $\mathbf{5.921}$ & $\mathbf{2.50}$ \\
\bottomrule
\end{tabular}
\end{table}

\subsection{Top-1 Rescue, Tail Handling, and Kernel Validation}
If an aggressive threshold prunes all candidate blocks in the first chunk ($k=1$), leaving an empty softmax denominator ($l_k = 0$), the kernel executes a localized \emph{Top-1 Rescue Phase}: it tracks the maximum proxy bound index during scanning and unconditionally loads $K_{\text{max\_idx}}, V_{\text{max\_idx}}$ into SRAM. Furthermore, whenever sequence length $T$ is not an exact multiple of block size $b$, the final chunk ($k = N_{\text{chunks}}$) executes a predicated tail phase (\texttt{IS\_TAIL=True}) to process remaining tokens with out-of-bounds positions masked to $-\infty$.

We validate the grouped decode kernel across $31$ test configurations against an independent PyTorch reference implementation modeling greedy block selection and top-1 rescue. Integer selection counts (per-head density, attended token fraction, and GQA union blocks) match with zero discrepancy across all cases, while attention outputs match within $10^{-6}$ in \texttt{float32} and $10^{-3}$ in \texttt{float16}. The complete decode procedure is summarized in \autoref{alg:inference_kernel}.

\begin{breakablealgorithm}
\caption{Hardware-Aligned Block-Sparse Decoding (Inference)}
\label{alg:inference_kernel}
\begin{algorithmic}[1]
\Require Query $q \in \mathbb{R}^d$, KV Cache $K, V \in \mathbb{R}^{T \times d}$
\Require Block Size $b$, Blocks per Chunk $c$, Total Full Blocks $N_{\text{full}} \gets \lfloor T/b \rfloor$
\Require Metadata: Centroids $\mu_j \in \mathbb{R}^d$, Variances $\sigma_j^2 \in \mathbb{R}^d$, Max Norms $M_j \in \mathbb{R}$ for $j \in [1, N_{\text{full}}]$
\Require Dynamic Threshold $\tau$, Softmax Scale $s$, Operating Offset $a$, Z-score $z$
\Ensure Output $o \in \mathbb{R}^d$

\State Calculate $N_{\text{chunks}} \gets \lceil N_{\text{full}} / c \rceil$
\If{$N_{\text{chunks}} > 1$}
    \State Allocate intermediate buffers $M_{\text{mid}}, L_{\text{mid}} \in \mathbb{R}^{N_{\text{chunks}}}$, $O_{\text{mid}} \in \mathbb{R}^{N_{\text{chunks}} \times d}$ in HBM
\EndIf
\State $\tau_{\text{eff}} \gets \tau - a$ \Comment{Effective threshold with operating offset}

\For{each chunk $k$ from $1$ to $N_{\text{chunks}}$} \Comment{Processed in parallel across GPU SMs}
    \State $m_k \gets -\infty$, \quad $l_k \gets 0$, \quad $acc_k \gets \mathbf{0} \in \mathbb{R}^d$, \quad $\text{tokens\_attended} \gets 0$
    \State $\text{max\_bound} \gets -\infty$, \quad $\text{max\_idx} \gets -1$
    
    \For{each full block $j \in [(k-1)c + 1, \, \min(k c, N_{\text{full}})]$}
        \State Load $\mu_j, \sigma_j^2, M_j$ from HBM into SRAM
        \State $\text{bound}_{\text{gauss}} \gets s \cdot \left(\langle q, \mu_j \rangle + z \cdot \sqrt{\sum_i q_i^2 (\sigma_j^2)_i}\right)$
        \State $\text{bound}_{\text{spher}} \gets s \cdot \|q\| \cdot M_j$
        \State $\text{bound} \gets \min(\text{bound}_{\text{gauss}}, \text{bound}_{\text{spher}})$
        
        \If{$\text{bound} \geq \tau_{\text{eff}}$}
            \State Load $K_j, V_j$ from HBM into SRAM \Comment{Significant block}
            \State Update $m_k, l_k, acc_k$ via gated online softmax ($\tilde{S}_j = S_j \odot \sigma(\beta(S_j - \tau_{\text{eff}}))$)
            \State $\text{tokens\_attended} \gets \text{tokens\_attended} + b$
        \EndIf
        
        \If{$k == 1$ \textbf{and} $\text{bound} > \text{max\_bound}$} \Comment{Track max only for Chunk 1 rescue}
            \State $\text{max\_bound} \gets \text{bound}$, \quad $\text{max\_idx} \gets j$
        \EndIf
    \EndFor
    
    \If{$k == 1$ \textbf{and} $\text{tokens\_attended} == 0$} \Comment{Top-1 Rescue Phase (Chunk 1 only)}
        \State Load $K_{\text{max\_idx}}, V_{\text{max\_idx}}$ from HBM into SRAM
        \State Update $m_k, l_k, acc_k$ via gated online softmax using $q, K_{\text{max\_idx}}, V_{\text{max\_idx}}$
    \EndIf
    
    \If{$k == N_{\text{chunks}}$ \textbf{and} $T > N_{\text{full}} \cdot b$} \Comment{Tail-Block Phase (partial remaining tokens)}
        \State Load remaining tail tokens $K_{\text{tail}}, V_{\text{tail}}$ from HBM into SRAM
        \State Update $m_k, l_k, acc_k$ via gated online softmax with partial masking
    \EndIf
    
    \If{$N_{\text{chunks}} == 1$}
        \State \Return $acc_k / l_k$ \Comment{Single-chunk fast path: return normalized output directly}
    \Else
        \State Store unnormalized accumulators: $m_k \to M_{\text{mid}}[k]$, \quad $l_k \to L_{\text{mid}}[k]$, \quad $acc_k \to O_{\text{mid}}[k]$
    \EndIf
\EndFor
\vspace{2mm}
\State \textbf{Global Reduction Kernel:} \Comment{Executed when $N_{\text{chunks}} > 1$}
\State $m_{\text{global}} \gets -\infty$, \quad $l_{\text{global}} \gets 0$, \quad $o \gets \mathbf{0} \in \mathbb{R}^d$
\For{each chunk $k$ from $1$ to $N_{\text{chunks}}$}
    \State $m_{\text{new}} \gets \max(m_{\text{global}}, M_{\text{mid}}[k])$
    \State $\alpha_{\text{global}} \gets \exp(m_{\text{global}} - m_{\text{new}})$, \quad $\alpha_{\text{chunk}} \gets \exp(M_{\text{mid}}[k] - m_{\text{new}})$
    
    \State $l_{\text{global}} \gets l_{\text{global}} \cdot \alpha_{\text{global}} + L_{\text{mid}}[k] \cdot \alpha_{\text{chunk}}$
    \State $o \gets o \cdot \alpha_{\text{global}} + O_{\text{mid}}[k] \cdot \alpha_{\text{chunk}}$
    \State $m_{\text{global}} \gets m_{\text{new}}$
\EndFor
\State \Return $o / l_{\text{global}}$ \Comment{Final normalization across all chunks}
\end{algorithmic}
\end{breakablealgorithm}

\section{Train-to-Inference Representational Consistency}
\label{appx:train_inference_consistency}

\subsection{Theoretical Analysis of Block Pruning and Over-Inclusion}
\label{appx:log_odds_invariance}

To build intuition for why SiLU tail tempering supports both block pruning and block over-inclusion, we analyze a simplified regime at decoding step $t$. Let $S_s = \frac{q_t k_s^\top}{\sqrt{d_k}}$ denote the raw pre-softmax logit for key $s \le t$, $\tilde{\tau}_t = \tau_t - a$ the effective head threshold, and $\tilde{S}_s = S_s \sigma(\beta(S_s - \tilde{\tau}_t))$ the SiLU-tempered logit with weight $w_s := \exp(\tilde{S}_s)$. 

At block-sparse inference, the causal context $\mathcal{S}_t = \{1, \dots, t\}$ partitions into three disjoint sets:
\begin{enumerate}[itemsep=1pt,topsep=2pt]
\item $\mathcal{K}_t = \{s : S_s \ge \tilde{\tau}_t\}$: the salient active keys,
\item $\mathcal{C}_t \subset \mathcal{P}_t$: the co-admitted sub-threshold keys inside loaded blocks, and
\item $\mathcal{D}_t = \mathcal{P}_t \setminus \mathcal{C}_t$: the dropped sub-threshold keys inside pruned blocks,
\end{enumerate}
where $\mathcal{P}_t = \mathcal{C}_t \sqcup \mathcal{D}_t$ is the collective sub-threshold tail.

\begin{definition}
For any subset $\mathcal{X} \subseteq \mathcal{S}_t$, let $Z_{\mathcal{X}} := \sum_{s \in \mathcal{X}} w_s$ be its total weight and $\tilde{v}_{\mathcal{X}} := \frac{1}{Z_{\mathcal{X}}} \sum_{s \in \mathcal{X}} w_s v_s$ be its locally normalized value average. The token-sparse output $o_t^{\mathrm{token}}$ (retaining $\mathcal{K}_t$), block-sparse inference output $o_t^{\mathrm{block}}$ (retaining $\mathcal{K}_t \sqcup \mathcal{C}_t$), and training output $o_t^{\mathrm{train}}$ (retaining $\mathcal{S}_t = \mathcal{K}_t \sqcup \mathcal{C}_t \sqcup \mathcal{D}_t$) are defined as:
\begin{align}
o_t^{\mathrm{token}} &:= \tilde{v}_{\mathcal{K}_t},\nonumber\\
o_t^{\mathrm{block}} &:= \tilde{v}_{\mathcal{K}_t \cup \mathcal{C}_t} = \frac{Z_{\mathcal{K}} o_t^{\mathrm{token}} + Z_{\mathcal{C}} \tilde{v}_{\mathcal{C}}}{Z_{\mathcal{K}} + Z_{\mathcal{C}}},\nonumber\\
o_t^{\mathrm{train}} &:= \tilde{v}_{\mathcal{S}_t} = \frac{Z_{\mathcal{K}} o_t^{\mathrm{token}} + Z_{\mathcal{C}} \tilde{v}_{\mathcal{C}} + Z_{\mathcal{D}} \tilde{v}_{\mathcal{D}}}{Z_{\mathcal{K}} + Z_{\mathcal{C}} + Z_{\mathcal{D}}}.
\end{align}
\end{definition}

We make two simplifying assumptions on the sub-threshold tail $\mathcal{P}_t$:
\begin{enumerate}[leftmargin=*,itemsep=1pt,topsep=2pt,label=\textbf{(A\arabic*)}]
    \item \textbf{Bounded Tail Weight Spread:} On any sub-threshold set $\mathcal{X} \subseteq \mathcal{P}_t$, the max-to-mean weight ratio is bounded:
    $$
    \kappa = \frac{\max_{s \in \mathcal{X}} w_s}{\frac{1}{|\mathcal{X}|}\sum_{s \in \mathcal{X}} w_s} \ge 1.
    $$ 
    We verify this empirically: $\kappa = 2.03$ for ETA vs.\ $15.6$ for dense (\autoref{tab:masking_formulation_comparison}).
    \item \textbf{Uncorrelated Tail Values:} Sub-threshold value vectors $\{v_s\}_{s \in \mathcal{P}_t}$ have mean $\boldsymbol{\mu}_v = \mathbb{E}[v_s] = \mathbf{0}$, variance $\operatorname{tr}(\Sigma_v) = \mathbb{E}[\|v_s\|_2^2]$, and zero pairwise correlation $\mathbb{E}[\langle v_s, v_r \rangle \mid q_t] = 0$ for $s \neq r$ (verified empirically on C4 in \autoref{tab:prop1_verification_145b}).
\end{enumerate}

We show that non-salient tokens can be safely ignored during inference:
\begin{proposition}
\label{prop:uniform_floor_cancellation}
Under \textbf{(A1)}--\textbf{(A2)}, for any sub-threshold set $\mathcal{X} \in \{\mathcal{C}_t, \mathcal{D}_t, \mathcal{P}_t\}$ of size $|\mathcal{X}|$, the expected squared norm of $\tilde{v}_{\mathcal{X}}$ satisfies:
\begin{equation}
\frac{\operatorname{tr}(\Sigma_v)}{|\mathcal{X}|} \;\le\; \mathbb{E}\!\left[\|\tilde{v}_{\mathcal{X}}\|_2^2\right] \;\le\; \frac{\kappa \operatorname{tr}(\Sigma_v)}{|\mathcal{X}|},
\label{eq:tail_variance_neff}
\end{equation}
Furthermore, the minimum $\operatorname{tr}(\Sigma_v)/|\mathcal{X}|$ is uniquely attained by uniform weights ($w_s \equiv c, \kappa = 1$). In that case, we can substitute $\tilde{v}_{\mathcal{C}} \approx \mathbf{0}$ and $\tilde{v}_{\mathcal{D}} \approx \mathbf{0}$, so all three outputs are positive scalar multiples of $o_t^{\mathrm{token}}$ and coincide under RMSNorm:
\begin{equation}
\operatorname{RMSNorm}\!\big(o_t^{\mathrm{train}}\big) \approx
\operatorname{RMSNorm}\!\big(o_t^{\mathrm{block}}\big) \approx \operatorname{RMSNorm}\!\big(o_t^{\mathrm{token}}\big)
\label{eq:rmsnorm_invariance}
\end{equation}
\end{proposition}

\begin{proof}
For any non-empty sub-threshold subset $\mathcal{X} \subseteq \mathcal{P}_t$, define the normalized tail weights $p_s := w_s / Z_{\mathcal{X}}$ for $s \in \mathcal{X}$, so that $p_s > 0$ and $\sum_{s \in \mathcal{X}} p_s = 1$.

We first decompose the variance conditioned on $q_t$:
\begin{align}
\mathbb{E}\!\left[\|\tilde{v}_{\mathcal{X}}\|_2^2\right]
&= \mathbb{E}\!\left[\left\langle \sum_{s \in \mathcal{X}} p_s v_s,\, \sum_{r \in \mathcal{X}} p_r v_r \right\rangle\right] \nonumber \\
&= \sum_{s \in \mathcal{X}} p_s^2 \, \mathbb{E}\!\left[\|v_s\|_2^2\right] + \sum_{\substack{s, r \in \mathcal{X} \\ s \neq r}} p_s p_r \, \mathbb{E}\!\left[\langle v_s, v_r \rangle \mid q_t\right] \nonumber \\
&\overset{\textbf{(A2)}}{=} \operatorname{tr}(\Sigma_v) \sum_{s \in \mathcal{X}} p_s^2
= \operatorname{tr}(\Sigma_v) \, \frac{\sum_{s \in \mathcal{X}} w_s^2}{Z_{\mathcal{X}}^2}.
\label{eq:proof_var_decomp}
\end{align}

The lower bound comes via the Cauchy--Schwarz inequality, which gives:
\begin{align*}
Z_{\mathcal{X}}^2
= \left(\sum_{s \in \mathcal{X}} w_s \cdot 1\right)^{\!2}
&\le \left(\sum_{s \in \mathcal{X}} w_s^2\right)\!\left(\sum_{s \in \mathcal{X}} 1^2\right)
= |\mathcal{X}| \sum_{s \in \mathcal{X}} w_s^2 \\
\implies \quad \frac{\sum_{s \in \mathcal{X}} w_s^2}{Z_{\mathcal{X}}^2} &\ge \frac{1}{|\mathcal{X}|},
\end{align*}
Here equality holds if and only if $(w_s)_{s \in \mathcal{X}}$ is collinear with $\mathbf{1}$, i.e., $w_s \equiv c$ for all $s \in \mathcal{X}$ ($\kappa = 1$). Substituting this into \autoref{eq:proof_var_decomp} gives the lower bound $\mathbb{E}[\|\tilde{v}_{\mathcal{X}}\|_2^2] \ge \operatorname{tr}(\Sigma_v) / |\mathcal{X}|$.

For the upper bound we rely on \textbf{(A1)}, to obtain that
\begin{align*}
\sum_{s \in \mathcal{X}} w_s^2
&\le w_{\max} \sum_{s \in \mathcal{X}} w_s
= \left(\frac{\kappa}{|\mathcal{X}|} Z_{\mathcal{X}}\right) Z_{\mathcal{X}}
= \frac{\kappa}{|\mathcal{X}|} Z_{\mathcal{X}}^2 \\
\implies \quad \frac{\sum_{s \in \mathcal{X}} w_s^2}{Z_{\mathcal{X}}^2} &\le \frac{\kappa}{|\mathcal{X}|}.
\end{align*}
Combining the above establishes \autoref{eq:tail_variance_neff}.

Finally, under SiLU tail tempering, both $\tilde{v}_{\mathcal{C}}$ and $\tilde{v}_{\mathcal{D}}$ concentrate around $\mathbf{0}$ at rate $\mathcal{O}(\sqrt{\kappa/|\mathcal{X}|})$. Substituting $\tilde{v}_{\mathcal{C}} \approx \mathbf{0}$ and $\tilde{v}_{\mathcal{D}} \approx \mathbf{0}$ into \autoref{eq:convex_output_decomp} leaves $o_t^{\mathrm{train}}$ and $o_t^{\mathrm{block}}$ collinear with $o_t^{\mathrm{token}}$:
\begin{align*}
o_t^{\mathrm{train}} &\approx \alpha_{\mathrm{train}} \, o_t^{\mathrm{token}}, \qquad \alpha_{\mathrm{train}} := \frac{Z_{\mathcal{K}}}{Z_{\mathcal{K}} + Z_{\mathcal{C}} + Z_{\mathcal{D}}} > 0, \\
o_t^{\mathrm{block}} &\approx \alpha_{\mathrm{block}} \, o_t^{\mathrm{token}}, \qquad \alpha_{\mathrm{block}} := \frac{Z_{\mathcal{K}}}{Z_{\mathcal{K}} + Z_{\mathcal{C}}} > 0.
\end{align*}
Because $\operatorname{RMSNorm}(x) = \frac{x}{\|x\|_{\mathrm{rms}}} \odot \gamma$ with $\|x\|_{\mathrm{rms}} = \sqrt{\frac{1}{d}\sum_{i=1}^d x_i^2}$ is invariant to scaling by constants we have
\begin{align*}
\operatorname{RMSNorm}\!\big(\alpha \, o_t^{\mathrm{token}}\big)
= \frac{\alpha \, o_t^{\mathrm{token}}}{\alpha \, \|o_t^{\mathrm{token}}\|_{\mathrm{rms}}} \odot \gamma
= \operatorname{RMSNorm}\!\big(o_t^{\mathrm{token}}\big),
\end{align*}
which yields \autoref{eq:rmsnorm_invariance}.
\end{proof}

Next, we show that co-admitting sub-threshold keys $\mathcal{C}_t$ inside loaded GPU blocks moves the pre-norm attention output closer to the training output under ETA tempering, whereas it introduces output error under additive hard masking:
\begin{proposition}
\label{prop:overinclusion_interpolation}
Suppose tail value averages cancel ($\tilde{v}_{\mathcal{C}} = \tilde{v}_{\mathcal{D}} = \mathbf{0}$) and $\|o_t^{\mathrm{token}}\|_2 > 0$. When loaded GPU blocks co-admit sub-threshold keys $\mathcal{C}_t$ ($Z_{\mathcal{C}} > 0$) alongside pruned blocks $\mathcal{D}_t$ ($Z_{\mathcal{D}} > 0$):
\begin{enumerate}[leftmargin=*,itemsep=2pt,topsep=2pt]
    \item \textbf{ETA Tempering ($o_t^{\mathrm{train}} = \tilde{v}_{\mathcal{S}_t}$):} Block-sparse inference $o_t^{\mathrm{block}}$ lies strictly closer to the training output $o_t^{\mathrm{train}}$ than exact token pruning $o_t^{\mathrm{token}}$:
    \begin{equation}
    \big\|o_t^{\mathrm{block}} - o_t^{\mathrm{train}}\big\|_2 \;<\; \big\|o_t^{\mathrm{token}} - o_t^{\mathrm{train}}\big\|_2.
    \label{eq:eta_overinclusion_bound}
    \end{equation}
    \item \textbf{Additive Hard Masking ($o_t^{\mathrm{train,add}} = o_t^{\mathrm{token}}$):} Because sub-threshold keys receive zero weight during training ($Z_{\mathcal{C}}^{\mathrm{train}} = Z_{\mathcal{D}}^{\mathrm{train}} = 0$), co-admitting $\mathcal{C}_t$ at inference strictly \emph{increases} output error:
    \begin{equation}
    \big\|o_t^{\mathrm{block}} - o_t^{\mathrm{train,add}}\big\|_2 \;>\; \big\|o_t^{\mathrm{token}} - o_t^{\mathrm{train,add}}\big\|_2 = 0.
    \label{eq:additive_overinclusion_bound}
    \end{equation}
\end{enumerate}
\end{proposition}

\begin{proof}
Under tail value cancellation ($\tilde{v}_{\mathcal{C}} = \tilde{v}_{\mathcal{D}} = \mathbf{0}$), \autoref{eq:convex_output_decomp} simplifies the token-sparse, block-sparse, and ETA training outputs to positive scalar multiples of $o_t^{\mathrm{token}}$:
\begin{align}
o_t^{\mathrm{token}} = 1 \cdot o_t^{\mathrm{token}}, \qquad
o_t^{\mathrm{block}} = \frac{Z_{\mathcal{K}}}{Z_{\mathcal{K}} + Z_{\mathcal{C}}}\, o_t^{\mathrm{token}}, \qquad
o_t^{\mathrm{train}} = \frac{Z_{\mathcal{K}}}{Z_{\mathcal{K}} + Z_{\mathcal{C}} + Z_{\mathcal{D}}}\, o_t^{\mathrm{token}}.
\label{eq:collinear_outputs_proof}
\end{align}

Since $Z_{\mathcal{K}}, Z_{\mathcal{C}}, Z_{\mathcal{D}} > 0$, the partition denominators satisfy $Z_{\mathcal{K}} < Z_{\mathcal{K}} + Z_{\mathcal{C}} < Z_{\mathcal{K}} + Z_{\mathcal{C}} + Z_{\mathcal{D}}$. We thus have:
\begin{equation*}
0 \;<\; \frac{Z_{\mathcal{K}}}{Z_{\mathcal{K}} + Z_{\mathcal{C}} + Z_{\mathcal{D}}} \;<\; \frac{Z_{\mathcal{K}}}{Z_{\mathcal{K}} + Z_{\mathcal{C}}} \;<\; 1.
\end{equation*}
Thus $o_t^{\mathrm{block}}$ interpolates strictly between $o_t^{\mathrm{train}}$ and $o_t^{\mathrm{token}}$ along $o_t^{\mathrm{token}}$, so:
\begin{align*}
\big\|o_t^{\mathrm{block}} - o_t^{\mathrm{train}}\big\|_2
&= \left|\frac{Z_{\mathcal{K}}}{Z_{\mathcal{K}} + Z_{\mathcal{C}}} - \frac{Z_{\mathcal{K}}}{Z_{\mathcal{K}} + Z_{\mathcal{C}} + Z_{\mathcal{D}}}\right| \|o_t^{\mathrm{token}}\|_2 \\
&= \left(\frac{Z_{\mathcal{K}}}{Z_{\mathcal{K}} + Z_{\mathcal{C}}} - \frac{Z_{\mathcal{K}}}{Z_{\mathcal{K}} + Z_{\mathcal{C}} + Z_{\mathcal{D}}}\right) \|o_t^{\mathrm{token}}\|_2 \\
&< \left(1 - \frac{Z_{\mathcal{K}}}{Z_{\mathcal{K}} + Z_{\mathcal{C}} + Z_{\mathcal{D}}}\right) \|o_t^{\mathrm{token}}\|_2
= \big\|o_t^{\mathrm{token}} - o_t^{\mathrm{train}}\big\|_2,
\end{align*}
which establishes \autoref{eq:eta_overinclusion_bound}.

Under additive hard masking, exact token pruning matches the training output ($\|o_t^{\mathrm{token}} - o_t^{\mathrm{train,add}}\|_2 = 0$), whereas co-admitting $\mathcal{C}_t$ inside loaded blocks ($Z_{\mathcal{C}} > 0$) shifts $o_t^{\mathrm{block}}$ away from $o_t^{\mathrm{train,add}}$:
\begin{align*}
\big\|o_t^{\mathrm{block}} - o_t^{\mathrm{train,add}}\big\|_2
&= \left\|\frac{Z_{\mathcal{K}}}{Z_{\mathcal{K}} + Z_{\mathcal{C}}}\, o_t^{\mathrm{token}} - o_t^{\mathrm{token}}\right\|_2 \\
&= \left(1 - \frac{Z_{\mathcal{K}}}{Z_{\mathcal{K}} + Z_{\mathcal{C}}}\right) \|o_t^{\mathrm{token}}\|_2\\
&= \frac{Z_{\mathcal{C}}}{Z_{\mathcal{K}} + Z_{\mathcal{C}}}\, \|o_t^{\mathrm{token}}\|_2 \;>\; 0
= \big\|o_t^{\mathrm{token}} - o_t^{\mathrm{train,add}}\big\|_2,
\end{align*}
which establishes \autoref{eq:additive_overinclusion_bound}.
\end{proof}

\subsection{Empirical Verification of Residual Stream Alignment and \autoref{prop:uniform_floor_cancellation}}
\label{appx:empirical_alignment}
We empirically test the predictions of \autoref{prop:uniform_floor_cancellation} across model scales on held-out C4 sequences. Specifically, we examine two theoretical claims:
\begin{enumerate}
    \item \textbf{Tail Decay Rate:} The centered pruned-tail value average $\|\bar{v}_{\mathcal{P}_t} - \boldsymbol{\mu}_v\|_2 / \sqrt{\operatorname{Tr}(\Sigma_v)}$ decays at a rate of $1/\sqrt{m}$ as the pruned set size $m = |\mathcal{P}_t|$ increases.
    \item \textbf{Radial Scaling and Direction Preservation:} Omitting pruned blocks $\mathcal{D}_t$ (while retaining SiLU-tempered weights on loaded blocks $\mathcal{K}_t \sqcup \mathcal{C}_t$) scales the pre-norm attention output radially by $(\alpha_{\mathrm{block}} / \alpha_{\mathrm{train}})$ while preserving the residual direction ($\mathcal{S}_{\cos} \approx 1$).
\end{enumerate}
To validate these behaviors, we evaluate both the 12-layer 126M and 22-layer 1.45B ETA models ($b=64, \mathtt{lwb}=2$, sub-block envelope).

\paragraph{Direct Empirical Verification of \autoref{prop:uniform_floor_cancellation} on the 1.45B Model.}
For each head and layer of the 1.45B model during C4 decoding, we extract the pruned value vectors $\{v_j\}_{j \in \mathcal{P}_t}$, center them by the sequence value mean $\boldsymbol{\mu}_v$, and measure the normalized Euclidean norm $\|\bar{v}_{\mathcal{P}_t} - \boldsymbol{\mu}_v\|_2 / \sqrt{\operatorname{Tr}(\Sigma_v)}$ as a function of the pruned set size $m = |\mathcal{P}_t| \in \{16, 32, 64, 128, 256, 384\}$. 

As reported in \autoref{tab:prop1_verification_145b}, the empirical norm tracks the theoretical $1/\sqrt{m}$ bound closely at small $m$ and decays even faster than $1/\sqrt{m}$ at larger $m$. Simultaneously, the active partition fraction satisfies $\alpha_{\mathrm{train}} = \mathbf{0.9973 \pm 0.0153}$ ($\text{median} = 1.0000$, $25\text{th percentile} = 0.9998$), confirming that $(1-\alpha_{\mathrm{train}})\|\bar{v}_{\mathcal{P}_t}\|_2 \ll 10^{-3}$.

\begin{table}[!h]
\centering
\small
\caption{\textbf{Empirical verification of \autoref{prop:uniform_floor_cancellation} ($1/\sqrt{m}$ background value cancellation) on the 22-layer 1.45B ETA model.} Normalized centered norm $\|\bar{v}_{\mathcal{P}_t} - \boldsymbol{\mu}_v\|_2 / \sqrt{\operatorname{Tr}(\Sigma_v)}$ measured across pruned set sizes $m = |\mathcal{P}_t|$ on C4, compared against the theoretical $1/\sqrt{m}$ rate.}
\label{tab:prop1_verification_145b}
\begin{tabular}{lcccccc}
\toprule
\textbf{Pruned Set Size $m = |\mathcal{P}_t|$} & $m=16$ & $m=32$ & $m=64$ & $m=128$ & $m=256$ & $m=384$ \\
\midrule
\textbf{Theoretical Rate ($1/\sqrt{m}$)} & $0.2500$ & $0.1768$ & $0.1250$ & $0.0884$ & $0.0625$ & $0.0510$ \\
\textbf{Empirical 1.45B Norm} & $\mathbf{0.2504}$ & $\mathbf{0.1750}$ & $\mathbf{0.1202}$ & $\mathbf{0.0799}$ & $\mathbf{0.0448}$ & $\mathbf{0.0274}$ \\
\bottomrule
\end{tabular}
\end{table}

\paragraph{Layer-by-Layer Residual Stream Alignment (126M 12-Layer and 1.45B 22-Layer Models).}
Let $\mathbf{x}_{\text{train}}^{(l)}$ and $\mathbf{x}_{\text{infer}}^{(l)}$ denote the residual stream hidden states at layer $l$ produced by full-sequence SiLU-gated attention ($\mathcal{S}_t$) and block-sparse SiLU-gated decoding ($\mathcal{K}_t \sqcup \mathcal{C}_t$), respectively. At each decoding step $t$, we compute the directional cosine similarity and the relative Euclidean norm shift:
\begin{equation}
\mathcal{S}_{\cos}(l) = \frac{\langle \mathbf{x}_{\text{train}}^{(l)}, \mathbf{x}_{\text{infer}}^{(l)} \rangle}{\|\mathbf{x}_{\text{train}}^{(l)}\|_2 \, \|\mathbf{x}_{\text{infer}}^{(l)}\|_2},
\qquad 
\mathcal{E}_{\ell_2}(l) = \frac{\|\mathbf{x}_{\text{train}}^{(l)} - \mathbf{x}_{\text{infer}}^{(l)}\|_2}{\|\mathbf{x}_{\text{train}}^{(l)}\|_2}.
\end{equation}

\begin{figure}[h]
\centering
\includegraphics[width=\linewidth]{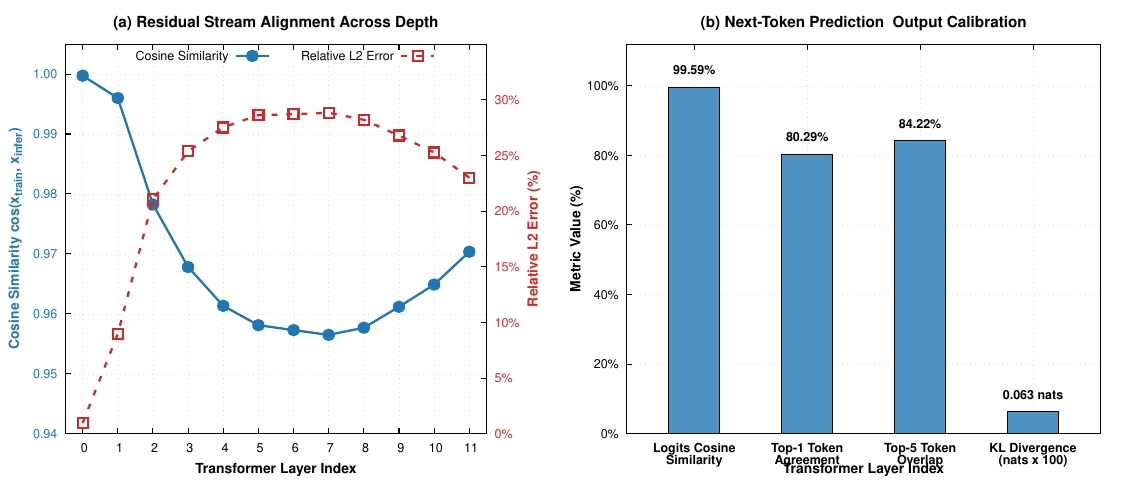}
\caption{\textbf{Internal representation and output distribution alignment between full-sequence SiLU-gated attention ($\mathcal{S}_t$) and block-sparse decoding ($\mathcal{K}_t \sqcup \mathcal{C}_t$).} (a) Layer-by-layer residual stream trajectory alignment across network depth. (b) Final language modeling head calibration metrics.}
\label{fig:train_vs_infer_alignment_main}
\end{figure}

As shown in \autoref{fig:train_vs_infer_alignment_main} and \autoref{tab:representation_alignment}, the empirical measurements confirm our theoretical decomposition at both the 126M (12-layer) and 1.45B (22-layer) scales:
\begin{itemize}[leftmargin=*,itemsep=2pt,topsep=2pt]
    \item \textbf{22-Layer 1.45B Alignment:} Across all $22$ layers ($L0\text{--}L21$), the 1.45B residual stream maintains an overall mean directional cosine similarity of $\mathcal{S}_{\cos} = \mathbf{0.9917}$ ($1.0000$ at $L0\text{--}L2$, $0.9910$ at $L9$, and $0.9888$ at $L21$), while the final language modeling head achieves a logit cosine similarity of $\mathbf{0.9944}$ and a KL divergence of only $\mathbf{0.0379\text{ nats}}$.
    \item \textbf{12-Layer 126M Alignment:} Under aggressive sparsity ($\lambda = 0.05$), the 126M model maintains a mean directional cosine similarity of $\mathcal{S}_{\cos} = \mathbf{0.9691}$ ($0.9997$ at Layer~0) with a radial norm shift of $\mathcal{E}_{\ell_2} = \mathbf{22.77\%}$ absorbed by $\operatorname{RMSNorm}$, achieving a logit cosine similarity of $\mathbf{0.9959}$ and a KL divergence of $\mathbf{0.0631\text{ nats}}$.
\end{itemize}

\begin{table}[!h]
\centering
\small
\caption{\textbf{Layer-wise residual stream alignment ($\mathcal{S}_{\cos}$) and relative $\ell_2$ shift ($\mathcal{E}_{\ell_2}$) across both the 12-layer 126M and 22-layer 1.45B ETA models} on held-out C4 sequences during autoregressive block-sparse decoding.}
\label{tab:representation_alignment}
\setlength{\tabcolsep}{4.5pt}
\resizebox{\linewidth}{!}{%
\begin{tabular}{lcc|lcc|lcc}
\toprule
\multicolumn{3}{c|}{\textbf{126M Model (Layers 0--11)}} & \multicolumn{6}{c}{\textbf{1.45B Model (Layers 0--21, $b=64$ Sub-Block Envelope + $\mathtt{lwb}=2$)}} \\
\midrule
\textbf{Layer} & $\mathcal{S}_{\cos}$ & $\mathcal{E}_{\ell_2}$ & \textbf{Layer} & $\mathcal{S}_{\cos}$ & $\mathcal{E}_{\ell_2}$ & \textbf{Layer} & $\mathcal{S}_{\cos}$ & $\mathcal{E}_{\ell_2}$ \\
\midrule
Layer 0  & 0.9997 & \phantom{0}0.96\% & Layer 0  & 1.0000 & \phantom{0}0.01\% & Layer 11 & 0.9881 & 13.56\% \\
Layer 1  & 0.9960 & \phantom{0}8.98\% & Layer 1  & 1.0000 & \phantom{0}0.34\% & Layer 12 & 0.9867 & 14.19\% \\
Layer 2  & 0.9782 & 21.06\% & Layer 2  & 1.0000 & \phantom{0}0.41\% & Layer 13 & 0.9863 & 14.38\% \\
Layer 3  & 0.9678 & 25.37\% & Layer 3  & 0.9999 & \phantom{0}1.26\% & Layer 14 & 0.9860 & 14.53\% \\
Layer 4  & 0.9613 & 27.52\% & Layer 4  & 0.9997 & \phantom{0}2.19\% & Layer 15 & 0.9863 & 14.19\% \\
Layer 5  & 0.9581 & 28.62\% & Layer 5  & 0.9997 & \phantom{0}2.35\% & Layer 16 & 0.9856 & 14.47\% \\
Layer 6  & 0.9573 & 28.70\% & Layer 6  & 0.9975 & \phantom{0}6.62\% & Layer 17 & 0.9854 & 14.46\% \\
Layer 7  & 0.9565 & 28.83\% & Layer 7  & 0.9932 & 10.22\% & Layer 18 & 0.9858 & 14.22\% \\
Layer 8  & 0.9577 & 28.20\% & Layer 8  & 0.9925 & 10.69\% & Layer 19 & 0.9865 & 13.96\% \\
Layer 9  & 0.9612 & 26.78\% & Layer 9  & 0.9910 & 11.80\% & Layer 20 & 0.9882 & 13.20\% \\
Layer 10 & 0.9649 & 25.27\% & Layer 10 & 0.9891 & 12.82\% & Layer 21 & 0.9888 & 12.83\% \\
Layer 11 & 0.9704 & 23.00\% & \multicolumn{3}{c|}{---} & \textbf{1.45B Mean} & \textbf{0.9917} & \textbf{9.94\%} \\
\midrule
\textbf{126M Mean} & \textbf{0.9691} & \textbf{22.77\%} & \multicolumn{6}{l}{\textbf{1.45B Logit Cos:} $\mathbf{0.9944}$ \quad $|$ \quad \textbf{KL Div:} $\mathbf{0.0379\text{ nats}}$ \quad $|$ \quad \textbf{Top-1 Match:} $\mathbf{100.0\%}$} \\
\bottomrule
\end{tabular}%
}
\end{table}

\section{Mechanistic Analysis: Elimination of Attention Sinks and Layer Specialization}
\label{appx:sink_layer_analysis}
To analyze ETA's behavior across network depth and sequence position, we record active token densities and initial-token softmax mass ($t \in [0, 3]$) across all $22$ layers and $32$ query heads ($704$ heads total) of our 1.45B model on held-out FineWeb sequences ($L = 2048$).

\paragraph{Mathematical Mechanism of Attention Sink Elimination.}
In standard dense or additive hard-masked attention ($S_{ts} \to -\infty$ for $s \notin \mathcal{I}_t$), attention probabilities must sum to unity over the active set ($\sum_{s \in \mathcal{I}_t} P_{ts} = 1$). When a query $q_t$ has no relevant context matches, softmax still forces $100\%$ of the probability mass onto active keys. To prevent noisy context mixtures during such ``no-op'' steps, standard Transformers route unused mass onto initial tokens ($t \in [0, 3]$), forming persistent \emph{attention sinks} \citep{xiao2024efficient} that absorb over $50\%$ of total attention mass (\autoref{fig:mechanisms}, Left; \autoref{subsec:mechanistic_interpretation}). Consequently, post-hoc eviction methods (e.g., StreamingLLM, $\text{H}_2\text{O}$) collapse unless initial sink tokens are heuristically pinned.

Under ETA's multiplicative gating ($\tilde{S}_{ts} = S_{ts} \cdot \sigma(\beta(S_{ts} - \tau_t))$), sub-threshold logits contract toward $0$ rather than $-\infty$. Because $\exp(0) = 1$, the training-time partition function $Z_t$ decomposes as:
$$Z_t = \sum_{s \in \mathcal{I}_t} \exp(\tilde{S}_{ts}) + \sum_{s \notin \mathcal{I}_t} \exp(0) = \sum_{s \in \mathcal{I}_t} \exp(\tilde{S}_{ts}) + (t - |\mathcal{I}_t|).$$
The $(t - |\mathcal{I}_t|)$ sub-threshold keys act as a built-in, uniform probability reservoir ($\tilde{P}_{ts} \approx 1/t$) whenever $\tau_t$ rises on uninformative queries, removing any optimization pressure to create localized sinks on initial tokens. As shown in \autoref{fig:mechanisms} (Left), average probability mass on tokens $t \in [0, 3]$ drops from $>50\%$ in dense attention to $<2\%$ under ETA across all sequence lengths, without heuristic sink pinning.

\paragraph{Autonomous Layer Specialization Across Depth.}
The $22 \times 32$ layer-head density profile (\autoref{fig:mechanisms}, Right; \autoref{subsec:mechanistic_interpretation}) shows that ETA organizes its attention budget into three stages without layer-specific tuning:
\begin{enumerate}[leftmargin=*,itemsep=1pt,topsep=2pt]
    \item \textbf{Broad Surface Aggregation (Layer 0):} Retains $\approx 96\%$ active density across heads to mix local and global surface features.
    \item \textbf{Mid-Network Compression (Layers 3--14):} Compresses context sharply (reaching $<8\%$ mean density at Layer~7), where heads specialize into targeted retrieval and syntactic routing.
    \item \textbf{Deep Predictive Synthesis (Layers 15--21):} Restores density to $35\%\text{--}48\%$ to aggregate distributed representations for next-token prediction.
\end{enumerate}

\section{NSA vs.\ ETA: Experimental Details}
\label{appx:nsa_eta}
Both 126M models were pretrained from scratch on FineWeb for $10{,}000$ steps at sequence length $L=1024$ on a single NVIDIA H100 80GB GPU (effective batch size $16$ sequences / $16{,}384$ tokens per step, $163.8\text{M}$ tokens total). We use AdamW ($\beta_1=0.9, \beta_2=0.95, \epsilon=10^{-8}$, weight decay $0.1$) with cosine learning rate decay (peak $\text{lr}=3\times 10^{-4}$, $500$ warmup steps). ETA includes our sparsity regularizer ($\lambda=0.05$), while NSA \citep{yuan2025native} trains with standard cross-entropy. To isolate the attention mechanism, both models share an identical 12-layer LLaMA-style backbone (\autoref{tab:arch_params}).

\begin{table}[h]
\centering
\scriptsize
\caption{\textbf{Backbone Architectural Hyperparameters.} All backbone dimensions and positional encodings are identical between the two models. ETA adds $12 \times (768 \times 12 + 12) = 0.11$M threshold-predictor parameters ($<0.1\%$ of total parameters).}
\label{tab:arch_params}
\begin{tabular}{lcc}
\toprule
\textbf{Hyperparameter} & \textbf{NSA 126M} & \textbf{ETA 126M} \\
\midrule
Layers ($L$) / Hidden Size ($d_{\text{model}}$) / FFN ($d_{\text{ffn}}$) & $12$ / $768$ / $2048$ & $12$ / $768$ / $2048$ \\
Query Heads ($H_Q$) / KV Heads ($H_{KV}$) / Head Dim ($d_k$) & $12$ / $4$ / $64$ & $12$ / $4$ / $64$ \\
Total Parameters (excl.\ output embedding) & $125.85\text{M}$ ($100.66\text{M}$) & $125.85\text{M}$ ($100.66\text{M}$) \\
Vocabulary Size / RoPE Base ($\theta$) & $32{,}768$ / $5\times 10^5$ (LLaMA-3 $8\times$) & $32{,}768$ / $5\times 10^5$ (LLaMA-3 $8\times$) \\
\bottomrule
\end{tabular}
\end{table}

\subsection{Training Convergence and Wall-Clock Dynamics}
As shown in \autoref{fig:loss_convergence}, ETA converges faster and to a lower cross-entropy loss than NSA ($\mathbf{4.370\text{--}4.399}$, $\text{PPL} = 81.37$ vs.\ $\mathbf{4.508}$, $\text{PPL} = 90.78$). On identical H100 hardware, ETA's fused Triton forward/backward kernel matches NSA's training step time within $0.08\%$ ($\mathbf{0.5315\text{ s/step}}$ vs.\ $\mathbf{0.5311\text{ s/step}}$), while avoiding NSA's three separate attention branches (compression, selection, and sliding window) and auxiliary gating MLPs.

\begin{figure}[h]
\centering
\begin{tikzpicture}
\begin{axis}[
    width=0.72\columnwidth,
    height=4.8cm,
    xlabel={\small Training Steps},
    ylabel={\small Pretraining Cross-Entropy Loss},
    xmin=0, xmax=10000,
    ymin=4.2, ymax=7.8,
    grid=both,
    grid style={line width=.1pt, draw=gray!20},
    major grid style={line width=.2pt,draw=gray!40},
    legend pos=north east,
    legend cell align={left},
    legend style={font=\footnotesize, fill=white, fill opacity=0.85},
    tick label style={font=\footnotesize},
    label style={font=\small},
    every axis plot/.append style={thick}
]
\addplot[color=nsaorange, mark=none, line width=1.15pt] coordinates {
    (180, 7.574) (200, 7.436) (300, 7.014) (400, 6.615) (500, 6.384)
    (600, 6.205) (700, 6.059) (1000, 5.642) (1500, 5.385) (2000, 5.145)
    (2100, 5.141) (2200, 5.079) (2300, 5.052) (2380, 5.051) (3000, 4.895)
    (3800, 4.742) (4080, 4.709) (5000, 4.640) (6000, 4.585) (7000, 4.548)
    (8000, 4.525) (9000, 4.512) (9800, 4.502) (9900, 4.492) (10000, 4.508)
};
\addlegendentry{NSA 126M ($0.5311$\,s/step; Final: 4.508)}
\addplot[color=etablue, mark=none, line width=1.15pt] coordinates {
    (180, 7.445) (200, 7.326) (300, 6.736) (400, 6.425) (500, 6.219)
    (600, 6.019) (700, 5.872) (900, 5.668) (1200, 5.460) (1500, 5.264)
    (1800, 5.110) (2100, 4.990) (2400, 4.923) (2700, 4.866) (3000, 4.795)
    (3300, 4.727) (3600, 4.697) (3900, 4.636) (4200, 4.629) (4500, 4.570)
    (4800, 4.556) (5100, 4.521) (5400, 4.480) (5700, 4.460) (6000, 4.456)
    (6300, 4.462) (6600, 4.396) (6900, 4.434) (7200, 4.406) (7500, 4.437)
    (7800, 4.399) (8100, 4.399) (8400, 4.380) (8700, 4.407) (9000, 4.424)
    (9300, 4.393) (9600, 4.396) (9800, 4.370) (10000, 4.399)
};
\addlegendentry{ETA 126M ($0.5315$\,s/step; Final: 4.370--4.399)}
\end{axis}
\end{tikzpicture}
\caption{\textbf{Pretraining Loss Convergence and Matched-H100 Throughput.} Training loss trajectories across $10{,}000$ steps ($164$M tokens, $L=1024$) on identical H100 GPUs. ETA achieves $\Delta \text{Loss} = -0.11\text{ to }-0.14$ ($\Delta \text{PPL} = -9.4$) below NSA at matched step latency ($0.5315$\,s/step vs.\ $0.5311$\,s/step).}
\label{fig:loss_convergence}
\end{figure}
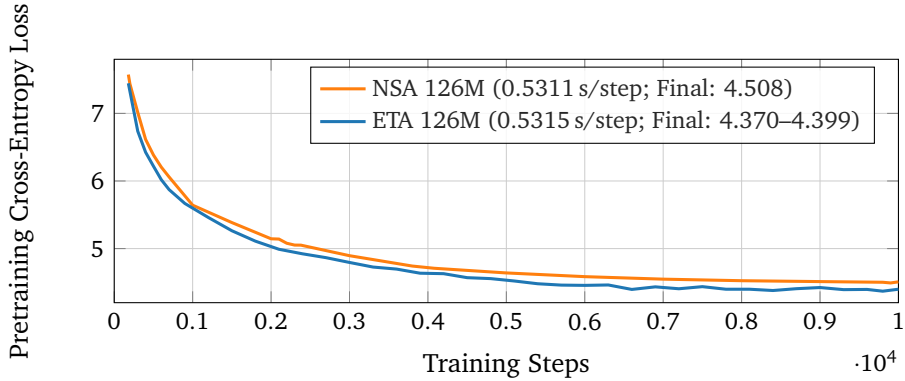

\subsection{Autoregressive Downstream Perplexity}
We evaluate both checkpoints at $L=2048$ ($2\times L_{\text{train}}$) under $1\%$ prompt prefill and $99\%$ autoregressive KV-cached decoding across FineWeb, C4, and WikiText-2 ($a=0.4, z=2.0$; \autoref{tab:ppl_results}). ETA achieves lower perplexity on FineWeb ($99.80$ vs.\ $103.55$) and WikiText-2 ($158.64$ vs.\ $161.65$) while pruning ${\approx}53\%$ of attention compute, and matches NSA on C4 ($100.73$ vs.\ $100.18$).

\begin{table}[h]
\centering
\small
\caption{\textbf{Perplexity (PPL) and Sparsity Comparison.} $126$M models, evaluated under 1\% prompt prefill + 99\% autoregressive decode (batch size $4$, $L=2048$).}
\label{tab:ppl_results}
\begin{tabular}{lcccc}
\toprule
\textbf{Dataset} & \textbf{NSA PPL} & \textbf{ETA PPL} & \textbf{$\Delta$ PPL} & \textbf{ETA Sparsity} \\
\midrule
\textbf{FineWeb (10BT)} & $103.55 \pm 26.51$ \small{[$\pm 1.87$]} & $\mathbf{99.80 \pm 21.71}$ \small{[$\pm 1.54$]} & $\mathbf{-3.75}$ & \textbf{53.60\%} \\
\textbf{C4} & $\mathbf{100.18 \pm 16.08}$ \small{[$\pm 1.14$]} & $100.73 \pm 17.96$ \small{[$\pm 1.27$]} & $+0.55$ & \textbf{53.89\%}  \\
\textbf{Wikitext-2} & $161.65 \pm 22.50$ \small{[$\pm 1.24$]} & $\mathbf{158.64 \pm 22.30}$ \small{[$\pm 1.17$]} & $\mathbf{-3.01}$ & \textbf{51.17\%} \\
\bottomrule
\end{tabular}
\end{table}

\subsection{Autoregressive Decoding Latency and Throughput}
We benchmark decoding latency and throughput across batch sizes $B \in \{16, 32\}$ and context lengths up to $L = 32\text{K}$ ($10\%$ prefill / $90\%$ decode; \autoref{tab:decoding_efficiency}). By replacing NSA's multi-branch kernel launches with a single fused block-sparse pass, ETA reduces per-step decode latency across all configurations—reaching a $2.72\times$ speedup at $B = 32, L = 8\text{K}$ ($25.35$\,ms vs.\ $69.07$\,ms; $1{,}262$ vs.\ $463$\,tok/s) and a $1.63\times$ speedup at $B = 16, L = 32\text{K}$ ($37.36$\,ms vs.\ $60.93$\,ms).

\begin{table}[h]
\centering
\small
\caption{\textbf{Decoding Efficiency \& Throughput Sweep.} Evaluated under 10\% prompt prefill + 90\% autoregressive decode across sequence lengths up to 32,768 tokens.}
\label{tab:decoding_efficiency}
\resizebox{\linewidth}{!}{%
\begin{tabular}{cclccc}
\toprule
\textbf{Batch ($B$)} & \textbf{Context ($L$)} & \textbf{Model} & \textbf{Prefill (TTFT)} & \textbf{Step Latency (ms)} & \textbf{Throughput (tok/s)} \\
\midrule
16 & 16,384 (16K) & NSA & 0.992 s & $61.319 \pm 2.248$ & 260.93 \\
16 & 16,384 (16K) & \textbf{ETA} & \textbf{0.319 s} & $\mathbf{24.254 \pm 6.086}$ & \textbf{659.68} \\
\midrule
16 & 32,768 (32K) & NSA & 1.724 s & $60.932 \pm 2.927$ & 262.59 \\
16 & 32,768 (32K) & \textbf{ETA} & \textbf{0.649 s} & $\mathbf{37.356 \pm 12.475}$ & \textbf{428.31} \\
\midrule
32 & 8,192 (8K) & NSA & 0.174 s & $69.069 \pm 1.573$ & 463.31 \\
32 & 8,192 (8K) & \textbf{ETA} & 0.294 s & $\mathbf{25.353 \pm 5.808}$ & \textbf{1,262.16} \\
\midrule
32 & 16,384 (16K) & NSA & 0.407 s & $69.102 \pm 1.359$ & 463.09 \\
32 & 16,384 (16K) & \textbf{ETA} & 0.502 s & $\mathbf{37.737 \pm 12.720}$ & \textbf{847.98} \\
\midrule
32 & 32,768 (32K) & NSA & 0.465 s & $66.193 \pm 4.728$ & 483.44 \\
32 & 32,768 (32K) & \textbf{ETA} & 1.149 s & $\mathbf{61.291 \pm 23.937}$ & \textbf{522.10} \\
\bottomrule
\end{tabular}%
}
\end{table}

\section{Extended Block Granularity, Pareto, and RULER Benchmarks}
\label{appx:block_mechanics}

This appendix provides supplementary figures and full sweep tables for Section~\ref{sec:blocksize} and Section~\ref{subsec:needle_retrieval}.

\subsection{GQA Union Loading and Variance Dilution at Coarse Block Sizes}
\label{appx:gqa_union_loading}
\label{appx:screening_index_quality}
\label{appx:block_size_sweeps}
In Grouped-Query Attention (GQA), each KV head is shared by $G = H_Q / H_{KV}$ query heads, so the GPU loads the union of active blocks $\mathcal{U}_g = \bigcup_{q \in \mathcal{G}_g} \mathcal{S}_q$ from HBM. On the 126M model ($G=3$), \autoref{fig:gqa_union_bars} shows that the active group union expands by only $1.49\times$ over per-head compute density ($50.8\%$ union vs.\ $34.1\%$ head density), as query heads sharing a KV head attend to overlapping context spans.

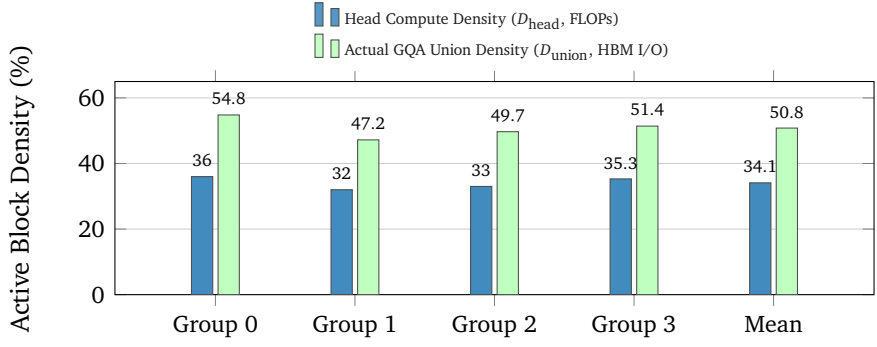
\begin{figure}[h]
\centering
\begin{tikzpicture}
\begin{axis}[
    ybar,
    bar width=8pt,
    width=0.70\columnwidth,
    height=4.4cm,
    ylabel={\small Active Block Density (\%)},
    symbolic x coords={Group 0, Group 1, Group 2, Group 3, Mean},
    xtick=data,
    ymin=0, ymax=65,
    legend style={
        at={(0.5, 1.05)},
        anchor=south,
        legend columns=1,
        font=\tiny,
        fill=white,
        fill opacity=0.9,
        draw=none
    },
    legend cell align={left},
    tick label style={font=\footnotesize},
    label style={font=\small},
    ymajorgrids=true,
    grid style={line width=.1pt, draw=gray!20},
    major grid style={line width=.2pt,draw=gray!40},
    enlarge x limits=0.18,
    nodes near coords,
    nodes near coords style={font=\tiny, /pgf/number format/fixed, /pgf/number format/precision=1}
]
\addplot[fill=etablue!85, draw=black!70] coordinates {
    (Group 0, 36.0) (Group 1, 32.0) (Group 2, 33.0) (Group 3, 35.3) (Mean, 34.1)
};
\addlegendentry{Head Compute Density ($D_{\text{head}}$, FLOPs)}

\addplot[fill=green!25, draw=black!70] coordinates {
    (Group 0, 54.8) (Group 1, 47.2) (Group 2, 49.7) (Group 3, 51.4) (Mean, 50.8)
};
\addlegendentry{Actual GQA Union Density ($D_{\text{union}}$, HBM I/O)}
\end{axis}
\end{tikzpicture}
\caption{\textbf{GQA Union Loading vs.\ Compute Sparsity across KV Groups} on C4 decoding (126M model).}
\label{fig:gqa_union_bars}
\end{figure}

While larger blocks ($b=64$) reduce GQA union overhead and align with Tensor Core tiles (\autoref{tab:blocksize_pareto_1}), computing a single flat variance $\sigma_{B,i}^2 = \frac{1}{b}\sum_{k \in B} (k_i - \mu_{B,i})^2$ across all $64$ tokens dilutes localized $4$-token spikes by up to $16\times$. As shown in \autoref{tab:index_quality}, scaling $b$ from $4$ to $64$ under flat variance increases block false negatives from $3.4\%$ to $27.5\%$ and missed softmax mass from $10.2\%$ to $36.3\%$.

\begin{table}[h]
\centering
\caption{\textbf{Effect of variance dilution under flat block variance ($z=2.0, a=0.4$)} before applying the sub-block envelope ($1.45$B model on FineWeb, $L=2048$). Missed mass is the fraction of softmax probability mass in rejected blocks.}
\label{tab:index_quality}
\begin{tabular}{rrrrrr}
\toprule
$b$ & False neg. & False pos. & Token recall & Missed mass & PPL \\
\midrule
$4$  & $\mathbf{3.4\%}$  & $30.8\%$ & $\mathbf{98.3\%}$ & $\mathbf{10.2\%}$ & $\mathbf{17.38}$ \\
$16$ & $11.7\%$ & $16.8\%$ & $96.9\%$ & $21.5\%$ & $18.47$ \\
$64$ & $27.5\%$ & $\mathbf{5.6\%}$  & $95.3\%$ & $36.3\%$ & $19.15$ \\
\bottomrule
\end{tabular}
\end{table}

\paragraph{Sub-Block Variance Envelope.}
To eliminate the $b=64$ variance dilution shown in \autoref{tab:index_quality} without increasing the $(1, d_h)$ metadata footprint per block, we partition $B$ into $16$ contiguous $4$-token sub-blocks $B_1, \dots, B_{16}$ ($b_{\text{sub}}=4$), store the coordinate-wise \emph{sub-block standard deviation}
\begin{equation}
\tilde{\sigma}_{B,i} = \max_{1 \le m \le b/b_{\text{sub}}} \sqrt{\frac{1}{b_{\text{sub}}} \sum_{k \in B_m} (k_i - \mu_{B,i})^2},
\label{eq:subblock_envelope_appx}
\end{equation}
and pin the two most recent blocks ($\mathtt{lwb}=2$, $128$ tokens). Replacing flat $\boldsymbol{\sigma}_B$ with $\tilde{\boldsymbol{\sigma}}_B$ at $b=64$ recovers FineWeb perplexity from $19.15$ (\autoref{tab:index_quality}) down to $\mathbf{17.11}$ (\autoref{tab:inference_experiments_stats}), and on C4 ($a=-0.4$) boosts token recall from $76.8\%$ (flat variance, $21.02$ PPL) to $\mathbf{97.3\%}$ ($\mathbf{18.60}$ PPL; \autoref{tab:pareto_c4_needle}).

\subsection{Extended C4 Perplexity and RULER Retrieval Tables}
\label{app:pareto_details}
\label{appx:ruler_suite}
\autoref{tab:full_c4_pareto_appendix} provides the full offset sweep supplementing \autoref{tab:pareto_c4_needle} in \autoref{sec:blocksize}. For long-context retrieval (\autoref{subsec:needle_retrieval}), \autoref{tab:ruler_suite} reports the complete $16$-arm RULER suite across all four tasks at $L=2048$ and $L=4096$ ($N=432$ trials/arm), and \autoref{tab:single_needle_extrap_appx} reports single-needle passkey extrapolation from $L=1024$ to $L=8192$ ($4\times L_{\text{train}}$).

\begin{table}[h]
\centering
\caption{\textbf{Extended C4 Perplexity ($L=2048$) and Needle-in-a-Haystack ($L=4096$) Pareto sweep} on the 1.45B model across block size $b$ and threshold offset $a$, supplementing \autoref{tab:pareto_c4_needle}.}
\label{tab:full_c4_pareto_appendix}
\label{tab:bound_tightness_quest}
\renewcommand{\arraystretch}{0.84}
\setlength{\tabcolsep}{3.8pt}
\resizebox{0.92\linewidth}{!}{%
\begin{tabular}{l l c c c c c}
\toprule
\textbf{Method Variant} & \textbf{Offset / Budget} & \textbf{Head ($D_{\text{head}}$)} & \textbf{Union ($D_{\text{union}}$)} & \textbf{Recall / FNR} & \textbf{C4 PPL ($\downarrow$)} & \textbf{Needle $4\text{K}$ ($\uparrow$)} \\
\midrule
Dense Baseline & \texttt{FlashAttention-2} & $100.0\%$ & $100.0\%$ & $100.0\%$ / $0.0\%$ & $17.49 \pm 0.33$ & $66.7\%$ ($6/9$) \\
\cmidrule(lr){1-7}
\multirow{6}{*}{ETA ($b=4, z=2.0$)} 
 & $a = +0.4$ & $36.3\%$ & $66.6\%$ & $98.3\%$ / $7.1\%$ & $\mathbf{17.56 \pm 0.33}$ & $\mathbf{66.7\%}$ ($6/9$) \\
 & $a = 0.0$  & $32.0\%$ & $61.2\%$ & $97.8\%$ / $8.9\%$ & $17.61 \pm 0.33$ & $\mathbf{66.7\%}$ ($6/9$) \\
 & $a = -0.8$ & $24.4\%$ & $50.6\%$ & $97.1\%$ / $12.1\%$ & $17.82 \pm 0.33$ & $\mathbf{66.7\%}$ ($6/9$) \\
 & $a = -1.5$ & $18.9\%$ & $41.7\%$ & $96.0\%$ / $16.1\%$ & $18.22 \pm 0.33$ & $\mathbf{66.7\%}$ ($6/9$) \\
 & $a = -2.0$ & $15.7\%$ & $35.6\%$ & $95.4\%$ / $18.8\%$ & $18.81 \pm 0.34$ & $\mathbf{66.7\%}$ ($6/9$) \\
 & $a = -2.5$ & $\mathbf{13.1\%}$ & $\mathbf{30.1\%}$ & $94.8\%$ / $21.2\%$ & $19.81 \pm 0.34$ & $\mathbf{66.7\%}$ ($6/9$) \\
\cmidrule(lr){1-7}
\multirow{4}{*}{ETA ($b=64$, Sub-Blk + $\mathtt{lwb}=2$)} 
 & $a = +0.2$ & $55.3\%$ & $79.0\%$ & $\mathbf{98.4\%}$ / $12.3\%$ & $18.30 \pm 0.29$ & $\mathbf{66.7\%}$ ($6/9$) \\
 & $a = -0.4$ & $48.3\%$ & $69.7\%$ & $97.3\%$ / $16.6\%$ & $18.60 \pm 0.29$ & $\mathbf{66.7\%}$ ($6/9$) \\
 & $a = -1.0$ & $42.7\%$ & $60.4\%$ & $95.6\%$ / $21.3\%$ & $18.87 \pm 0.29$ & $\mathbf{66.7\%}$ ($6/9$) \\
 & $a = -1.8$ & $37.6\%$ & $49.6\%$ & $91.6\%$ / $28.3\%$ & $19.21 \pm 0.30$ & $\mathbf{66.7\%}$ ($6/9$) \\
\cmidrule(lr){1-7}
\multirow{3}{*}{$\text{H}_2\text{O}$ Eviction} 
 & $W=256$ (${\approx}38\%$ target) & $48.0\%$ & $68.0\%$ & --- & $17.59 \pm 0.33$ & $22.2\%$ ($2/9$) \\
 & $W=128$ (${\approx}25\%$ target) & $33.2\%$ & $54.3\%$ & --- & $17.71 \pm 0.33$ & $0.0\%$ ($0/9$) \\
 & $W=64$ (${\approx}10\%$ target)  & $17.5\%$ & $28.8\%$ & --- & $17.91 \pm 0.33$ & $0.0\%$ ($0/9$) \\
\bottomrule
\end{tabular}%
}
\end{table}

\begin{table}[h]
\centering
\caption{\textbf{RULER Long-Context Suite on the 1.45B model ($L=2048$ Native \& $L=4096$ with $2\times$ NTK RoPE Scaling; $N=432$ trials/arm).} Reports per-task and average accuracy ($\pm 95\%$ CI), plus routine haystack vs.\ query-step density ($\text{Routine} \to \text{Query}$).}
\label{tab:ruler_suite}
\renewcommand{\arraystretch}{0.86}
\setlength{\tabcolsep}{3.2pt}
\resizebox{\linewidth}{!}{%
\begin{tabular}{llccccccc}
\toprule
\textbf{Context} & \textbf{Method} & \textbf{Head ($\text{Routine} \to \text{Query}$)} & \textbf{Union ($\text{Routine} \to \text{Query}$)} & \textbf{Single} & \textbf{MultiKey} & \textbf{MultiVal} & \textbf{VT (2-Hop)} & \textbf{RULER Avg} \\
\midrule
\multirow{8}{*}{$L = 2048$} 
 & Dense Baseline & $100.0\%$ & $100.0\%$ & $87.0\%$ & $76.9\%$ & $88.9\%$ & $87.0\%$ & $84.95\% \pm 3.19\%$ \\
 \cmidrule(lr){2-9}
 & ETA ($b=4, z=1.0, a=0.4$) & $38.0\% \to 41.0\%$ & $66.3\% \to 69.8\%$ & $\mathbf{88.0\%}$ & $87.0\%$ & $\mathbf{88.9\%}$ & $\mathbf{88.0\%}$ & $\mathbf{87.96\% \pm 2.86\%}$ \\
 & ETA ($b=64, z=1.5, a=0.4$) & $40.6\% \to 43.5\%$ & $54.6\% \to 57.9\%$ & $87.0\%$ & $\mathbf{88.0\%}$ & $\mathbf{88.9\%}$ & $77.8\%$ & $85.42\% \pm 3.14\%$ \\
 & ETA ($b=64, z=1.5, a=0.0$) & $25.1\% \to 27.8\%$ & $\mathbf{37.4\% \to 39.8\%}$ & $87.0\%$ & $76.9\%$ & $78.7\%$ & $67.6\%$ & $77.55\% \pm 3.78\%$ \\
 \cmidrule(lr){2-9}
 & Quest ($b=64$, top-$38\%$) & $39.1\%$ & $51.2\%$ & $55.6\%$ & $23.1\%$ & $12.0\%$ & $13.0\%$ & $25.93\% \pm 4.14\%$ \\
 & Quest ($b=64$, top-$25\%$) & $\mathbf{23.1\%}$ & $32.7\%$ & $34.3\%$ & $12.0\%$ & $0.9\%$ & $0.9\%$ & $12.04\% \pm 3.07\%$ \\
 & $\text{H}_2\text{O}$ ($W=256$, ${\sim}38\%$) & $38.3\%$ & $72.3\%$ & $25.0\%$ & $12.0\%$ & $0.9\%$ & $0.9\%$ & $9.72\% \pm 2.80\%$ \\
 & $\text{H}_2\text{O}$ ($W=256$, ${\sim}25\%$) & $25.3\%$ & $50.8\%$ & $13.9\%$ & $11.1\%$ & $0.9\%$ & $0.0\%$ & $6.48\% \pm 2.33\%$ \\
\midrule
\multirow{8}{*}{$L = 4096$} 
 & Dense Baseline & $100.0\%$ & $100.0\%$ & $66.7\%$ & $65.7\%$ & $56.5\%$ & $33.3\%$ & $55.56\% \pm 4.14\%$ \\
 \cmidrule(lr){2-9}
 & ETA ($b=4, z=1.0, a=0.4$) & $38.0\% \to 44.5\%$ & $66.3\% \to 73.9\%$ & $\mathbf{66.7\%}$ & $\mathbf{65.7\%}$ & $\mathbf{56.5\%}$ & $\mathbf{33.3\%}$ & $\mathbf{55.56\% \pm 4.14\%}$ \\
 & ETA ($b=64, z=1.5, a=0.4$) & $38.2\% \to 44.6\%$ & $51.5\% \to 58.7\%$ & $\mathbf{66.7\%}$ & $\mathbf{65.7\%}$ & $55.6\%$ & $24.1\%$ & $53.01\% \pm 4.16\%$ \\
 & ETA ($b=64, z=1.5, a=0.0$) & $24.8\% \to 31.1\%$ & $\mathbf{36.9\% \to 41.2\%}$ & $56.5\%$ & $55.6\%$ & $50.9\%$ & $23.1\%$ & $46.53\% \pm 4.16\%$ \\
 \cmidrule(lr){2-9}
 & Quest ($b=64$, top-$38\%$) & $38.3\%$ & $50.9\%$ & $24.1\%$ & $13.0\%$ & $1.9\%$ & $0.9\%$ & $9.95\% \pm 2.82\%$ \\
 & Quest ($b=64$, top-$25\%$) & $\mathbf{24.1\%}$ & $34.5\%$ & $13.0\%$ & $1.9\%$ & $0.0\%$ & $0.0\%$ & $3.70\% \pm 1.78\%$ \\
 & $\text{H}_2\text{O}$ ($W=256$, ${\sim}38\%$) & $38.1\%$ & $74.6\%$ & $3.7\%$ & $1.9\%$ & $0.9\%$ & $0.0\%$ & $1.62\% \pm 1.20\%$ \\
 & $\text{H}_2\text{O}$ ($W=256$, ${\sim}25\%$) & $25.1\%$ & $56.7\%$ & $1.9\%$ & $0.9\%$ & $0.0\%$ & $0.0\%$ & $0.69\% \pm 0.78\%$ \\
\bottomrule
\end{tabular}%
}
\end{table}

\begin{table}[h]
\centering
\caption{\textbf{Single-Needle passkey retrieval accuracy across context lengths up to $L=8192$ ($4\times L_{\text{train}}$).} Accuracy averaged over $9$ needle depths ($0\%\text{--}100\%$). Density reports position-averaged head density at $L=1024 \to 8192$.}
\label{tab:single_needle_extrap_appx}
\renewcommand{\arraystretch}{0.86}
\setlength{\tabcolsep}{5pt}
\resizebox{0.72\linewidth}{!}{%
\begin{tabular}{l l c c c c c}
\toprule
\textbf{Variant} & \textbf{Config} & \textbf{Density (\%)} & \textbf{1024 Acc} & \textbf{2048 Acc} & \textbf{4096 Acc} & \textbf{8192 Acc} \\
\midrule
\textbf{SDPA} & \texttt{Dense Baseline} & 100.0 & \hlbase{\textbf{100.0}} & \hlbase{\textbf{88.9}} & \hlbase{\textbf{66.7}} & \textit{0.0} \\
\cmidrule(lr){1-7} \addlinespace[1pt] 
\textbf{H2O} & \texttt{W=256, HH=10\%} & 29.0$\rightarrow$11.5 & 22.2 & 11.1 & \textit{0.0} & \textit{0.0} \\
\textbf{BigBird} & \texttt{W=256, Str=16} & 30.0$\rightarrow$9.2 & 22.2 & 11.1 & \textit{0.0} & \textit{0.0} \\
\textbf{StreamLLM} & \texttt{W=256, S=4} & 25.4$\rightarrow$3.2 & 22.2 & 11.1 & \textit{0.0} & \textit{0.0} \\
\textbf{SWA} & \texttt{W=256} & 25.0$\rightarrow$3.1 & 22.2 & 11.1 & \textit{0.0} & \textit{0.0} \\
\textbf{ETA} & \texttt{b=4, a=0.4} & 30.1$\rightarrow$69.3 & \hlprob{\textbf{100.0}} & \hlprob{\textbf{88.9}} & \hlprob{\textbf{66.7}} & \hlprob{\textbf{22.2}} \\
\bottomrule
\end{tabular}%
}
\end{table}

\section{Dual Probabilistic-Geometric Screening Index: Mathematical Analysis}
\label{appx:dual_index_math}

This section derives the finite-population concentration and geometric bounds governing the Dual Screening Index in \autoref{sec:inference} (\autoref{eq:dual_bound_eq}).

\subsection{Cantelli Bound and Certified Finite-Population Maximums}
Consider a key block $B = \{k_1, \dots, k_b\} \subset \mathbb{R}^{d_h}$ with empirical mean $\boldsymbol{\mu}_B = \frac{1}{b} \sum_{j=1}^b k_j$ and covariance $\Sigma_B = \frac{1}{b} \sum_{j=1}^b (k_j - \boldsymbol{\mu}_B)(k_j - \boldsymbol{\mu}_B)^\top$. Drawing $k \sim \operatorname{Unif}(B)$ uniformly at random yields a scalar inner product $X = \langle q, k \rangle$ with mean $\mathbb{E}[X] = \langle q, \boldsymbol{\mu}_B \rangle$ and variance $\operatorname{Var}(X) = q^\top \Sigma_B \, q$. By Cantelli's one-sided inequality, for any $z > 0$:
\begin{equation}
\Pr_{k \sim \operatorname{Unif}(B)}\!\Big[ \langle q, k \rangle \ge \langle q, \boldsymbol{\mu}_B \rangle + z \sqrt{q^\top \Sigma_B \, q} \Big] \le \frac{1}{1 + z^2}.
\end{equation}
Since $B$ has finite cardinality $b$, multiplying by $b$ and taking the floor bounds the exact number of keys exceeding the threshold:
\begin{equation}
\Big|\Big\{ k \in B : \langle q, k \rangle \ge \langle q, \boldsymbol{\mu}_B \rangle + z \sqrt{q^\top \Sigma_B \, q} \Big\}\Big| \le \left\lfloor \frac{b}{1 + z^2} \right\rfloor.
\end{equation}
Whenever $z > \sqrt{b - 1}$, the right-hand side is strictly less than $1$, so the count must be $0$. Thus for $b = 4$ (or a $4$-token sub-block $b_{\text{sub}}=4$) and $z \ge 2$, Cantelli's inequality guarantees a deterministic upper bound on $\max_{k \in B} \langle q, k \rangle$.

\subsection{Diagonal Variance Surrogate and Spherical Norm Ceiling}
To avoid $\mathcal{O}(d_h^2)$ storage for $\Sigma_B$, we cache only the coordinate-wise variance $\boldsymbol{\sigma}_B^2 = \operatorname{diag}(\Sigma_B) \in \mathbb{R}^{d_h}$, yielding the directional variance estimate $\sigma_S(q, B) = \sqrt{\sum_{i=1}^{d_h} q_i^2 \sigma_{B,i}^2}$. Intersecting this moment bound with the exact Cauchy--Schwarz spherical ceiling $\langle q, k \rangle \le \|q\|_2 \max_{k \in B} \|k\|_2 = \|q\|_2 M_B$ yields the $\mathcal{O}(1)$ Dual Screening Score of \autoref{eq:dual_bound_eq}. Unlike coordinate-wise bounding boxes \citep{tang2024quest}, which evaluate inner products against a synthetic corner vertex whose Euclidean norm inflates across $b=64$ tokens when isolated coordinates spike, clipping by $\|q\|_2 M_B$ prevents coordinate outliers from triggering false admissions while projecting all $G$ query heads in a GQA group onto a shared centroid $\boldsymbol{\mu}_B$.

\section{End-to-End Systems Accounting}
\label{appx:systems_accounting}
\label{tab:prefill_metadata_overhead}
\label{tab:training_throughput_breakdown}

To complement the standalone kernel benchmarks in \autoref{sec:speed} (\autoref{tab:standalone_attention}), \autoref{tab:e2e_decode_breakdown} reports full-model autoregressive decode latency ($q_{\text{len}}=1, B=64$, \texttt{float16}) across all $22$ layers of our 1.45B model on an H100 SXM 80GB GPU. At each decode step, non-attention layers (QKV/O projections, SwiGLU MLPs, RMSNorms, and the vocabulary head) contribute a constant $11.76\text{ ms/tok}$ independent of $L$, while evaluating the threshold predictor $\tau_t^\ell$ and updating the tail block metadata $(\boldsymbol{\mu}_B, \boldsymbol{\sigma}_B, M_B)$ adds $0.26\text{ ms/tok}$ ($<0.6\%$ of step time). Because KV-cache attention accounts for $78\%\text{--}93\%$ of total latency at $L \ge 64\text{K}$, ETA delivers up to a $\mathbf{1.40\times}$ end-to-end speedup at $35.5\%$ union density and $\mathbf{2.15\times}$ at $16.0\%$ union density ($L=256\text{K}$).

\begin{table}[h]
\centering
\caption{\textbf{Full 22-layer 1.45B model decode latency ($B=64$, \texttt{float16}, H100 80GB).} Total step time sums 22-layer attention, the constant non-attention floor ($11.76\text{ ms}$), and ETA predictor/metadata updates ($0.26\text{ ms}$).}
\label{tab:e2e_decode_breakdown}
\small
\renewcommand{\arraystretch}{0.88}
\setlength{\tabcolsep}{6pt}
\begin{tabular}{c l r r c}
\toprule
\textbf{Context $L$} & \textbf{Method ($B=64$)} & \textbf{22-Layer Attn (ms)} & \textbf{Total Step (ms/tok)} & \textbf{E2E Speedup} \\
\midrule
\multirow{4}{*}{\textbf{64K}} 
 & Dense FlashAttention-2           & $42.14$ & $53.90$ & $1.00\times$ \\
 & ETA $51.2\%$ union ($38\%$ head) & $40.33$ & $52.35$ & $1.03\times$ \\
 & ETA $35.5\%$ union ($25\%$ head) & $31.57$ & $43.59$ & $1.24\times$ \\
 & ETA $16.0\%$ union ($10\%$ head) & $\mathbf{20.61}$ & $\mathbf{32.63}$ & $\mathbf{1.65\times}$ \\
\midrule
\multirow{3}{*}{\textbf{128K}} 
 & Dense FlashAttention-2           & $82.92$ & $94.68$ & $1.00\times$ \\
 & ETA $35.5\%$ union ($25\%$ head) & $70.93$ & $82.95$ & $1.14\times$ \\
 & ETA $16.0\%$ union ($10\%$ head) & $\mathbf{53.54}$ & $\mathbf{65.56}$ & $\mathbf{1.44\times}$ \\
\midrule
\multirow{3}{*}{\textbf{256K}} 
 & Dense FlashAttention-2           & $167.37$ & $179.13$ & $1.00\times$ \\
 & ETA $35.5\%$ union ($25\%$ head) & $116.16$ & $128.18$ & $1.40\times$ \\
 & ETA $16.0\%$ union ($10\%$ head) & $\mathbf{71.22}$ & $\mathbf{83.24}$ & $\mathbf{2.15\times}$ \\
\bottomrule
\end{tabular}
\end{table}

\paragraph{Prefill Indexing and Pretraining Throughput.}
During prompt prefill ($B=1$), constructing the dual block metadata $(\boldsymbol{\mu}_B, \boldsymbol{\sigma}_B, M_B)$ requires a single $\mathcal{O}(L d_h)$ pass over $K$, whereas attention scales as $\Theta(L^2 d_h)$; consequently, relative prefill overhead decays as $\mathcal{O}(1/L)$, dropping from $+2.16\%$ ($+1.33\text{ ms}$) at $L=8\text{K}$ to $+0.51\%$ ($+2.43\text{ ms}$) at $32\text{K}$ and $\mathbf{+0.14\%}$ ($+8.31\text{ ms}$ out of $5{,}911\text{ ms}$) at $128\text{K}$. Similarly, during 22-layer pretraining at $L=2048$ ($B=4$, $8{,}192$ tokens/step), linear projections and SwiGLU MLPs account for $83.2\%$ of step time ($7.21\text{ ms/layer}$ vs.\ $1.45\text{ ms/layer}$ for FlashAttention-2). Thus, although ETA's fused Triton training kernel takes $3.83\text{ ms/layer}$, full-model step latency increases by only $\mathbf{+27.4\%}$ ($242.8\text{ ms}$ vs.\ $190.6\text{ ms}$), sustaining $\mathbf{293.5\text{ TFLOPs/GPU}}$ ($33{,}739\text{ tok/s/GPU}$).

\end{document}